%% file: 0_main.tex
\documentclass[9.5pt,journal,compsoc]{IEEEtran}

\ifCLASSOPTIONcompsoc
  \usepackage[nocompress]{cite}
\else
  \usepackage{cite}
\fi

\usepackage{amsmath}
\usepackage{color}
\usepackage{amsthm}
\usepackage{array}
\usepackage{breqn}
\usepackage{{inputenc}}
\usepackage{balance}
\usepackage{multirow,tabularx}
\usepackage{booktabs,longtable}
\usepackage{hhline}
\usepackage[flushleft]{threeparttable}
\usepackage{algorithm}
\usepackage{algorithmic}
\usepackage{epsfig}
\usepackage{graphicx}
\usepackage{epstopdf}
\usepackage{thmtools}
\usepackage{enumitem}
\usepackage{verbatim}
\usepackage{amsfonts}
\usepackage[colorlinks=false,
            linkcolor=red,
            anchorcolor=blue,
            citecolor=green
            ]{hyperref}
\usepackage{xcolor}
\usepackage{tikz}
\usepackage{colortbl}
\usepackage{subfigure}
\usepackage{thm-restate}

\newcommand\norm[1]{\left\lVert#1\right\rVert}
\newcommand{\Lapl}{\mathbf{\mathop{\mathcal{L}}}}

\newcommand{\Trans}[1]{{#1}^{\top}}

\newcommand{\Mat}[1]{\textbf{#1}}

\newcommand{\Space}[1]{\mathbb{#1}}
\newcommand{\Set}[1]{\mathcal{#1}}

\newcommand{\ie}{\emph{i.e., }}
\newcommand{\eg}{\emph{e.g., }}

\newcommand{\st}{\emph{s.t. }}

\newcommand{\wrt}{\emph{w.r.t. }}
\newcommand{\cf}{\emph{cf. }}
\newcommand{\aka}{\emph{aka. }}

\newcommand{\za}[1]{{\color{black}{#1}}}

\theoremstyle{plain}
\newtheorem{theorem}{Theorem}[section]
\newtheorem{lemma}{Lemma}[section]

\theoremstyle{definition}
\newtheorem{definition}{Definition}[section]

\theoremstyle{remark}

\begin{document}

\title{Differentiable Invariant Causal Discovery}

\author{Yu Wang,
        An Zhang,
        Xiang Wang,
        Yancheng Yuan,
        Xiangnan He,
        Tat-Seng Chua% <-this % stops a space
\IEEEcompsocitemizethanks{
  \IEEEcompsocthanksitem Y. Wang, X. Wang and X. He are with University of Science and Technology of China. Email: yuw164@ucsd.edu, xiangwang1223@gmail.com, xiangnanhe@gmail.com.
  \IEEEcompsocthanksitem A. Zhang, T. Chua are with Sea-NExt Joint Lab from National University of Singapore. Email: an\_zhang@nus.edu.sg, dcscts@nus.edu.sg.
  \IEEEcompsocthanksitem Y. Yuan is with The Hong Kong Polytechnic University. Email: yancheng.yuan@polyu.edu.hk. 
  \IEEEcompsocthanksitem Correspondonce to A. Zhang.
%   \protect\\
%   \IEEEcompsocthanksitem A. Zhang and T. Chua are with National University of Singapore. E-mail: an\_zhang@nus.edu.sg,  dcscts@nus.edu.sg. A. Zhang is the corresponding author.
  }
% \thanks{Manuscript received April 19, 2005; revised August 26, 2015.}
}

\markboth{}%
{Wang \MakeLowercase{\textit{et al.}}}



\IEEEtitleabstractindextext{%
\begin{abstract}
  Learning causal structure from observational data is a fundamental challenge in machine learning. 
However, the majority of commonly used differentiable causal discovery methods are non-identifiable, turning this problem into a continuous optimization task prone to data biases. 
In many real-life situations, data is collected from different environments, in which the functional relations remain consistent across environments, while the distribution of additive noises may vary.
This paper proposes \underline{D}ifferentiable \underline{I}nvariant \underline{C}ausal \underline{D}iscovery (\textbf{DICD}), utilizing the multi-environment information based on a differentiable framework to avoid learning spurious edges and wrong causal directions. 
Specifically, DICD aims to discover the environment-invariant causation while removing the environment-dependent correlation.
We further formulate the constraint that enforces the target structure equation model to maintain optimal across the environments.
Theoretical guarantees for the identifiability of proposed DICD are provided under mild conditions with enough environments.
Extensive experiments on synthetic and real-world datasets verify that DICD outperforms state-of-the-art causal discovery methods up to 36\% in SHD. 
% Our code is open-sourced at \url{https://anonymous.4open.science/r/DICD-9271/}.
Our code will be open-sourced.
\end{abstract}

}

\maketitle

\IEEEdisplaynontitleabstractindextext

\IEEEpeerreviewmaketitle

\input{chapters/1_introduction}

\input{chapters/3_preliminary}
\input{chapters/4_methodology}

\input{chapters/6_theoretical_analysis}
\input{chapters/5_experiments}
\input{chapters/2_related_work}
\input{chapters/7_conclusion}

\bibliographystyle{IEEEtran}
\bibliography{my-tpami2022-rcexplainer}



% \end{thebibliography}
% \newpage
\vspace{-30pt}
\input{8_bio.tex}

\clearpage
\appendix
\input{chapters/8_appendix}

\end{document}

%% file: chapters/1_introduction.tex
% \vspace{-15pt}
\section{Introduction}\label{introduction}
% \vspace{-8pt}
% TODO: 1. causal discovery --> importance, basic task
% TODO: 2. a prevalent line --> score-based --> paradigm, basic framework

% Differentiable; 好在哪里
% Differentiable 不好在哪里, one environment 下，ERM prone to capture data bias；
% 现实中有multi-domain，然后有些方法提到过，但他们有一些问题: ... 不effective; 基本都是linear case; (分类, environment不实际等等); 我们的就是易于实现的且efficient, effective的

Causal discovery (CD) is a fundamental problem in a variety of tasks, such as understanding the generation process of data \cite{zheng2018dags}, and probing explainability of models \cite{CXPlain,PGM-Explainer}. It has tremendous impacts on various domains, like biology \cite{sachs2005causal,opgen2007correlation}, and finance \cite{sanford2012bayesian}.
CD aims to learn the causal structure among a set of variables from the observational data and represent the structure as a directed acyclic graph (DAG).
The acyclicity constraint frames CD as the combinatorial optimization of discrete edges, which however, is NP-hard.
Recently, leading CD solutions \cite{zheng2018dags,yu2019dag, zhu2019causal,LachapelleBDL20} gracefully convert the DAG learning into the continuous optimization task.
Specifically, the idea stemming from NOTEARS \cite{zheng2018dags} is to build a scoring function upon the adjacency matrix over the variables and find an equivalent continuous constraint on acyclicity.

% TODO: 3. deficiency --> tend to learn the data bias, disturb the learning of causation/causal relationships --> data bias usually originates from the environments
NOTEARS~\cite{zheng2018dags} inspired the development of numerous differentiable causal discovery algorithms that use gradient descent to find the optimal causal graph~\cite{zheng2020learning,2021differentiable,yu2020dags}. When compared to traditional constraint-based causal discovery, these methods have demonstrated superior performance and efficiency in uncovering the true graph with a large amount of data.
However, most of them follow the paradigm of empirical risk minimization (ERM) \cite{ERM} --- first imposing the scored DAG on the observations to reconstruct, and then minimizing the empirical risks between the observational and reconstructed data, so as to optimize the DAG.
Despite the promising performance, we argue that ERM is prone to capture data biases or shortcut~\cite{IRM,GroupDRO,Rubi}, thus derailing the structure learning of DAG.
Specifically, the observations of variables are often marred by some spurious correlations, such as the annotator or selection biases in the data acquisition pipeline \cite{IRM}, thus posing undesired entanglements of variables.
ERM easily latches on these correlations \cite{IRM,REx,GroupDRO} to refine the DAG structure, as the following example illustrates.

% Specifically, one inherent assumption is that the observations of variables are collected from one environment.
% However, the observations usually come from different environments, which typically result from various selection pipelines \cite{IRM,REx}, thereby marred by some biases posing spurious relations and entanglements among variables.
% As a result, ERM easily latches on the biases as the shortcut to refine the DAG structure, as the following example illustrates.

% TODO: focusing on one spurious edge between B and Y as an example.
% Consider a classification task's DAG in Figure \ref{fig:real}, which is instantiated by Figure \ref{fig:simulate_cmnist}.
% Let $X$ be the input variable (\eg images of handwritten digits), $Y$ be the label variable (\eg classes of digits).
% $X$ consists of three variables, which have three typical relationships with $Y$: $A$ determines the label (\eg handwritten digits), $B$ is influenced by the label (\eg background of digits), and $C$ is the irrelevant noises.
% Although $A$ is the causal determinant of $Y$, $B$ can be highly predictive of $Y$ (\eg $80\%$ images of digit ``$5$'' are collected with green background).
% Hence, ERM tends to reconstruct $Y$ upon $B$, and learn the invalid causation $B\rightarrow Y$ in DAG.
% This is common in score-based CD solutions optimized by ERM, which is also termed as over-reconstruction \cite{He0SXLJ21}, but remains largely unexplored.

% TODO: 4. running example
Consider the ground-truth DAG in Table \ref{tab:example} as the target being reconstructed by the CD solutions, where each edge denotes a causal relationship between two variable nodes.
$X$ consists of three variables, which have different relationships with $Y$: $A$ determines $Y$, $B$ is influenced by $Y$, and $C$ is irrelevant to $Y$.
While the CD solutions are designed to identify the causation edges, ERM does not need to learn the correct DAG to reach a low reconstruct loss for fitting the observations.
For example, instead of looking at the true causation $B\leftarrow Y$, it is easy to capture the statistical shortcut edge $B\rightarrow Y$, since $B$ is strongly correlated with $Y$ and the reconstructed DAGs achieve even lower losses as compared to the ground-truth DAG.
This common problem is consistent with over-reconstruction \cite{He0SXLJ21} but remains largely unexplored.

% TODO: introducing invariant learning .
% In this study, we aim to design a paradigm that distinguishes the causation edges from spuriously-correlated edges and obtains the faithful DAG.
% Although assessing causality is challenging, we draw inspiration from invariant learning \cite{IRM,REx,GroupDRO} and approximate the task by searching the edges that are invariant.
% Specifically, the observations are divided into different environments, which cause the distribution shifts.
% Across different environments, only factual causation edges remain invariant, while spurious edges hardly remain stable, according to the assumptions in invariant learning \cite{IRM}.
% Considering Figure \ref{fig:real} again, we exhibit the hidden environment variable $E$ that impacts the prior distributions of $B$ and $C$.
% Under these environments, the CD solution are highly likely to learn $B\rightarrow Y$ but with different functions.
% Then the edge $B\rightarrow Y$ should be excluded to obtain a stable DAG.

\begin{table*}[!t]
    \centering
    \caption{Examples that NOTEARS would find the wrong causal graph while multi-environment settings can help identify the true graph. 
    % 要不要去掉NoTears，改成conventional causal discovery methods？
    We present the DAGs derived from minimal reconstruction loss in different environments. The graph on the left denotes the ground truth. The edges in red are wrong or spurious, while the dash one represents the corresponding coefficient is zero. We list the reconstruction loss value following with the triplets in brackets representing coefficients of $Y$ with $A,B$ and $C$. The loss values and parameters of the potential sub-optimal causal structures learned by NOTEARS in a single environment are highlighted in red. Detailed data generating processes are as follows. $X\sim \mathcal{N}(0, 1)$, $A = X + z_A(\sim \mathcal{N}(0, 1))$, $B = X + Y/2 + z_B^e(\sim \mathcal{N}(0, (\sigma_B^e)^2 )) $, $C=X/2 + z_C^e(\sim \mathcal{N}(0, (\sigma_C^e)^2 ))$, $Y=A/4 + \epsilon_Y(\sim N(0, 1))$. $\sigma_B^e$ and $\sigma_C^e$ varies across three environments, $e_1, e_2$ and $e_3$, with $(\sigma_B^e)^2 = (\sigma_C^e)^2 = 1, 2, 4$, respectively.
    }
    \label{tab:example}
    % \resizebox{0.9\textwidth}{!}{%
    \begin{tabular}{r|r|rrr}
        \toprule
   \multirow{3}{*}{\begin{tikzpicture}[scale=1.00, every node/.style={scale=1.00}]
    \node[draw,circle] (b) at (0:0) {\textsf{B}};
    \node[draw,circle] (c) at (0:1.0) {\textsf{C}};
    \node[draw,circle] (a) at (180:1.0) {\textsf{A}};
    % \node[draw=cyan,circle,dashed] (e) at (60:1.0) {{\color{cyan}\textsf{E}}};
    \node[draw,circle] (x) at (90:1.0) {\textsf{X}};
    \node[draw,circle] (y) at (270:1.0) {\textsf{Y}};
    \draw[-latex] (x) -- (a);
    \draw[-latex] (x) -- (b);
    \draw[-latex] (x) -- (c);
    % \draw[-latex,draw=cyan, dashed] (e) -- (b);
    % \draw[-latex,dashed, draw=cyan] (e) -- (a);
    % \draw[-latex,dashed, draw=cyan] (e) -- (c);
    \draw[-latex] (a) -- (y);
    \draw[-latex] (y) -- (b);
    % \draw[-latex] (c) -- (y);
\end{tikzpicture}} &   & \multicolumn{1}{c}{\begin{tikzpicture}[scale=0.7, every node/.style={scale=0.7}]
    \node[draw,circle] (b) at (0:0) {\textsf{B}};
    \node[draw,circle] (c) at (0:1.0) {\textsf{C}};
    \node[draw,circle] (a) at (180:1.0) {\textsf{A}};
    % \node[draw=cyan,circle,dashed] (e) at (60:1.0) {{\color{cyan}\textsf{E}}};
    \node[draw,circle] (x) at (90:1.0) {\textsf{X}};
    \node[draw,circle] (y) at (270:1.0) {\textsf{Y}};
    \draw[-latex] (x) -- (a);
    \draw[-latex] (x) -- (b);
    \draw[-latex] (x) -- (c);
    % \draw[-latex,draw=cyan, dashed] (e) -- (b);
    % \draw[-latex,dashed, draw=cyan] (e) -- (a);
    % \draw[-latex,dashed, draw=cyan] (e) -- (c);
    \draw[-latex, draw=red] (y) -- (a);
    \draw[-latex, draw=red] (b) -- (y);
    \draw[-latex, draw=red] (c) -- (y);
\end{tikzpicture}}
& \multicolumn{1}{c}{\begin{tikzpicture}[scale=0.7, every node/.style={scale=0.7}]
    \node[draw,circle] (b) at (0:0) {\textsf{B}};
    \node[draw,circle] (c) at (0:1.0) {\textsf{C}};
    \node[draw,circle] (a) at (180:1.0) {\textsf{A}};
    % \node[draw=cyan,circle,dashed] (e) at (60:1.0) {{\color{cyan}\textsf{E}}};
    \node[draw,circle] (x) at (90:1.0) {\textsf{X}};
    \node[draw,circle] (y) at (270:1.0) {\textsf{Y}};
    \draw[-latex] (x) -- (a);
    \draw[-latex] (x) -- (b);
    \draw[-latex] (x) -- (c);
    % \draw[-latex,draw=cyan, dashed] (e) -- (b);
    % \draw[-latex,dashed, draw=cyan] (e) -- (a);
    % \draw[-latex,dashed, draw=cyan] (e) -- (c);
    \draw[-latex, draw=red] (b) -- (y);
    \draw[-latex, draw=red] (y) -- (a);
    \draw[-latex, draw=red, dashed] (y) -- (c);
\end{tikzpicture}}
& \multicolumn{1}{c}{\begin{tikzpicture}[scale=0.7, every node/.style={scale=0.7}]
    \node[draw,circle] (b) at (0:0) {\textsf{B}};
    \node[draw,circle] (c) at (0:1.0) {\textsf{C}};
    \node[draw,circle] (a) at (180:1.0) {\textsf{A}};
    % \node[draw=cyan,circle,dashed] (e) at (60:1.0) {{\color{cyan}\textsf{E}}};
    \node[draw,circle] (x) at (90:1.0) {\textsf{X}};
    \node[draw,circle] (y) at (270:1.0) {\textsf{Y}};
    \draw[-latex] (x) -- (a);
    \draw[-latex] (x) -- (b);
    \draw[-latex] (x) -- (c);
    % \draw[-latex,draw=cyan, dashed] (e) -- (b);
    % \draw[-latex,dashed, draw=cyan] (e) -- (a);
    % \draw[-latex,dashed, draw=cyan] (e) -- (c);
    \draw[-latex, draw=red] (y) -- (a);
    \draw[-latex] (y) -- (b);
    \draw[-latex, draw=red, dashed] (y) -- (c);
\end{tikzpicture}} \\
 \cmidrule(lr){2-5}
& $e_1$ &  \cellcolor{pink} 4.57 (0.24, 1.14, -1.32) & 5.07 (0.24, 0.32, 0.00) & 5.07 (0.24, 0.50, 0.00) \\
& $e_2$ & \cellcolor{pink}  6.59 (0.24, 1.09, -1.19) & 7.14 (0.24, 0.23, 0.00) & 7.07 (0.24, 0.50, 0.00) \\
& $e_3$ &  \cellcolor{pink} 10.60 (0.24, 1.07, -1.12) &  11.21 (0.24, 0.15, 0.00) & 11.07 (0.24, 0.50, 0.00) \\
 \cmidrule(lr){2-5}
& & \multicolumn{1}{c}{\begin{tikzpicture}[scale=0.7, every node/.style={scale=0.7}]
    \node[draw,circle] (b) at (0:0) {\textsf{B}};
    \node[draw,circle] (c) at (0:1.0) {\textsf{C}};
    \node[draw,circle] (a) at (180:1.0) {\textsf{A}};
    % \node[draw=cyan,circle,dashed] (e) at (60:1.0) {{\color{cyan}\textsf{E}}};
    \node[draw,circle] (x) at (90:1.0) {\textsf{X}};
    \node[draw,circle] (y) at (270:1.0) {\textsf{Y}};
    \draw[-latex] (x) -- (a);
    \draw[-latex] (x) -- (b);
    \draw[-latex] (x) -- (c);
    % \draw[-latex,draw=cyan, dashed] (e) -- (b);
    % \draw[-latex,dashed, draw=cyan] (e) -- (a);
    % \draw[-latex,dashed, draw=cyan] (e) -- (c);
    \draw[-latex, draw=red] (y) -- (a);
    \draw[-latex] (y) -- (b);
    \draw[-latex, draw=red] (c) -- (y);
\end{tikzpicture}}
& \multicolumn{1}{c}{\begin{tikzpicture}[scale=0.7, every node/.style={scale=0.7}]
    \node[draw,circle] (b) at (0:0) {\textsf{B}};
    \node[draw,circle] (c) at (0:1.0) {\textsf{C}};
    \node[draw,circle] (a) at (180:1.0) {\textsf{A}};
    % \node[draw=cyan,circle,dashed] (e) at (60:1.0) {{\color{cyan}\textsf{E}}};
    \node[draw,circle] (x) at (90:1.0) {\textsf{X}};
    \node[draw,circle] (y) at (270:1.0) {\textsf{Y}};
    \draw[-latex] (x) -- (a);
    \draw[-latex] (x) -- (b);
    \draw[-latex] (x) -- (c);
    % \draw[-latex,draw=cyan, dashed] (e) -- (b);
    % \draw[-latex,dashed, draw=cyan] (e) -- (a);
    % \draw[-latex,dashed, draw=cyan] (e) -- (c);
    \draw[-latex] (a) -- (y);
    \draw[-latex, draw=red] (b) -- (y);
    \draw[-latex, draw=red] (c) -- (y);
\end{tikzpicture}}
& \multicolumn{1}{c}{\begin{tikzpicture}[scale=0.7, every node/.style={scale=0.7}]
    \node[draw,circle] (b) at (0:0) {\textsf{B}};
    \node[draw,circle] (c) at (0:1.0) {\textsf{C}};
    \node[draw,circle] (a) at (180:1.0) {\textsf{A}};
    % \node[draw=cyan,circle,dashed] (e) at (60:1.0) {{\color{cyan}\textsf{E}}};
    \node[draw,circle] (x) at (90:1.0) {\textsf{X}};
    \node[draw,circle] (y) at (270:1.0) {\textsf{Y}};
    \draw[-latex] (x) -- (a);
    \draw[-latex] (x) -- (b);
    \draw[-latex] (x) -- (c);
    % \draw[-latex,draw=cyan, dashed] (e) -- (b);
    % \draw[-latex,dashed, draw=cyan] (e) -- (a);
    % \draw[-latex,dashed, draw=cyan] (e) -- (c);
    \draw[-latex] (a) -- (y);
    \draw[-latex, draw=red] (b) -- (y);
    \draw[-latex, draw=red, dashed] (y) -- (c);
\end{tikzpicture}} \\
\midrule
 5.00 (0.25, 0.50, 0.00) & $e_1$ & 5.05 (0.24, 0.50, 0.10) &  \cellcolor{pink} 4.57 (-0.23, 1.38, -1.54) & 5.12 (0.07, 0.29, 0.00) \\
7.00 (0.25, 0.50, 0.00) & $e_2$ & 7.06 (0.24, 0.50, 0.06) &\cellcolor{pink}6.59 (-0.24, 1.36, -1.44) & 7.17 (0.14, 0.18, 0.00) \\
11.00 (0.25, 0.50, 0.00) & $e_3$ & 11.06 (0.24, 0.50, 0.03) & \cellcolor{pink} 10.59 (-0.24, 1.35, -1.39) & 11.21 (0.18, 0.11, 0.00) \\
\bottomrule
\end{tabular}
\vspace{-8pt}
\end{table*}

% TODO: 5. invariant learning
In this study, we aim to design a paradigm that distinguishes the causation edges from spuriously-correlated edges and obtains the faithful DAG.
Although learning causal structure from observational data is challenging, we draw inspiration from invariant learning \cite{IRM,REx,GroupDRO} and approximate the task by searching the edges with \textit{invariant} structural equations in multi-environment settings.
Across different environments, only factual causation edges remain invariant, while edges with wrong causal directions hardly remain stable, according to the assumptions in \textit{invariant learning} \cite{IRM}. 
Though reminiscent of past ideas - \eg causal structure learning in multi-domain, from heterogeneous/nonstationary data~\cite{CausalInferenceIP,MultiDomain,CD-NOD,RegressionInvariance,SAEM,FOM,Dep} - these methods are only applicable in a linear system or restricted to conditional independence tests, which suffer from high computation complexity as the number of variables increases.
To achieve effective differentiable causal discovery in both linear and nonlinear scenarios,
we reconsider Table \ref{tab:example} and additionally exhibit the multi-environment information, which impacts the distributions of additive noises.
Applying ERM on these environments separately results in different functions when there exist spurious correlations in the causal graph. This inspires us to exclude such unstable edges towards a robust DAG across environments.

% TODO: 6. causal discovery with invariant learning
Towards this end, we propose \textbf{differentiable invariant causal discovery} (DICD), a novel scheme of DAG structure learning that incorporates the idea of invariant learning --- that is, learning an invariant DAG with the piece of environment information.
Specifically, a DAG generator module learns to generate the environment-aware DAGs from individual environments.
For each DAG, the structure equation model (SEM) \cite{pearl2016causal} describes the functions of learned causation edges.
We then exploit invariant risk minimization (IRM) \cite{IRM} to encourage the SEMs to be optimal across all environments, obtaining DAGs regardless of environment changes.
On synthetic and real-world datasets, extensive experiments demonstrate the effectiveness of DICD to surpass current state-of-the-art CD solutions.

% TODO: 7. contributions
Our main contributions are: 
% \vspace{-3pt}
\begin{itemize}[leftmargin=*]
    \item 
   To the best of our knowledge, we are among the first class to adopt the environment information into the differentiable causal discovery framework. 
    
    \item We propose a novel causal discovery solution, DICD, to incorporate invariant learning in both linear and nonlinear settings. Experimental results on synthetic and real-world datasets demonstrate that DICD could significantly better reveal true correlations and eliminate the spurious ones compared to prevalent methods.
    % \vspace{-5pt}
    
    \item We provide theoretical guarantees for the identifiability of proposed DICD in linear systems under mild conditions, given certain assumptions about the environments.
    % We give the theoretical guarantee about that in linear SEM systems, given certain assumptions about the environments, the causal graph will be identifiable. Besides, DICD will be theoretically guaranteed to find the ground truth graph.
    % 
    % \item We conduct extensive experiments on the synthetic and real-world datasets to demonstrate the excellent performance of ICD.
    % \item Extensive experiments on synthetic and real-world datasets demonstrate that ICD could significantly better reveal true correlations and eliminate the spurious ones compared with earlier methods. 
    %  \vspace{-5pt}
\end{itemize}

%% file: chapters/3_preliminary.tex
% \vspace{-5pt}
\section{Preliminary}
% \vspace{-5pt}
\label{preliminary}

We first make some necessary and reasonable assumptions. Then, we give the task formulation of causal discovery (CD). After that, we introduce the continuous optimization paradigm via empirical risk minimization (ERM), as well as the linear and nonlinear solutions. We list the notations mentioned in our paper in Appendix \ref{notations}
% Table \ref{tab:notations}. 

\vspace{5pt}
\noindent\textbf{Assumptions.} (i) The structural equations are invariant across different environments; (ii) \za{The system is causally sufficient; (iii) Observational data is generated from a structural equation model with independent additive noise;}
(iv) The distribution shift of additive noise appears in different environments.

The assumptions (i) - (iii) are crucial to our method and provide key insights on
how causal structure can be identified from observational data.
Assumption (i) states that the invariance principle is held across environments. 
Assumption (ii) ensures that no hidden confounders exist in the system. 
In other words, there is no systematic bias induced by hidden confounders. 
Assumption (iii) implies that only independent additive noise is considered in our paper, which is commonly used in causal discovery.
Assumption (iv) defines the multiple environments, which is consistent with the definition of environment assumed in \cite{MultiDomain, RegressionInvariance}  with heterogeneous data \cite{CD-NOD} or interventional data \cite{CausalInferenceIP}.
A formal definition of different environments is given in Definition \ref{definition_of_environment}.
Technically speaking, assumption (iv) presents both challenges and opportunities for causal discovery and provides a sufficient condition that the causal structure becomes identifiable.
A more formal Theorem \ref{theorem:identifiability} on this argument will be presented later in Section \ref{theoretical_analysis}.

\vspace{5pt}
\noindent \textbf{Task Formulation of CD.}
Let $\Mat{X}=\Trans{[\Mat{x}_{1}|\cdots|\Mat{x}_{n}]}\in\Space{R}^{n\times d}$ denote the $n$ observational data of $d$ variables, which are generated from a target directed acyclic graph (DAG).
The target DAG is $(\Set{V},\Set{D})$, where $\Set{V}$ represents the set of node variables, denoted as $\{X_1,\cdots, X_d\}$. And $\Set{D}$ is the set of cause-effect edges between variables. 
The observational data is assumed to be generated from the following SEM:
\begin{equation}\label{eq:data_generation}
    X_j = F_j(Pa(X_j)) + z_j, j\in\{1,\cdots,d\},
\end{equation}
    where $X_j$ is the $j$-th node variable, $F_j$ is the causal structure function, $Pa(X_j)$ is the set of the parents of $X_j$, and $z_j$ refers to the additive noise with variance $\sigma_j^2$.
    % where $\textbf{x}_j$ is the observational data corresponding to the $j$-th variable $X_j$, and $F_j$ does not depend on $\textbf{x}_k$ if $X_k\notin Pa(X_j)$, where $Pa(X_j)$ is the set of the parents of $X_j$. $z_j$ refers to the additive noise \za{with variance $\sigma_j^2$.} 
Without loss of generality, we assume that the noises are zero-mean.
In real-life settings, since the dataset may be obtained from various environments, the distribution of additive noises $z_j$ may differ across different environments. The causal structure functions $F_j$, on the other hand, are generally invariant.
The goal of CD is to learn a DAG to reconstruct the observations $\Mat{X}$.
The acyclicity restriction of DAG is the fundamental obstacle, because it frames DAG learning as a NP-hard combinatorial optimization task.

\vspace{5pt}
\noindent \textbf{Common Paradigm of ERM.}
Popular differentiable CD solutions, \eg NOTEARS \cite{zheng2018dags} and its follow-up studies \cite{zheng2020learning,yu2020dags}, convert DAG learning into a continuous optimization process to overcome the obstacle of the combinatorial optimization problem.
The primary idea is to build a scoring function upon the adjacency matrix of variables and discover an equivalent continuous constraint on acyclicity.
To optimize the scoring function, they mostly adopt the paradigm of ERM to minimize the empirical risks between the observational and reconstructed data as: 
% \begin{gather*}
%     \min_{f}\Lapl(f)=\frac{1}{n}\sum_{i=1}^{d}l(\Mat{x}_{i},f_{i}(\Mat{X})),~~\text{s.t.}~~\Set{G}(f)\in\text{DAG},
% \end{gather*}
% where $\Mat{x}_{i}$ is the observations of node variable $X_{i}$, and $l(\cdot,\cdot)$ is the reconstruction loss function, \ie squared loss or negative log-likelihood.
\begin{gather*}
    \min_f \Lapl(f) = \frac{1}{n}\sum_{i=1}^n l(\Mat{x}_i, f(\Mat{x}_i))~~\text{s.t.}~~\Set{G}(f) \in \text{DAG},
\end{gather*}
where $\Mat{x}_i$ is the $i$-th sample in the dataset. $l(\cdot,\cdot)$ is the reconstruction loss function, \ie squared loss or negative log-likelihood. 
$f=(f_{1},\cdots,f_{d})$ formulates the estimated structure function \cite{pearl2016causal} of variables, where $f_{i}:\Space{R}^{ d}\rightarrow\Space{R}$ is the estimated structure function of node variable $X_{i}$, and thus $f: \Space{R}^{d}\rightarrow \Space{R}^d$ is the structure function for all nodes. 
$f_{i}(X_{1},\cdots,X_{d})$ is dependent on $X_{j}$, if $X_{j}\in\text{Pa}(X_{i})$.
The score-based solutions seek to learn $f$ conditioning on the DAG constraint: $\Set{G}(f)\in\text{DAG}$, where $\Set{G}(f)$ refers to the graph corresponding to $f$.

Following prior studies \cite{zheng2018dags,zheng2020learning}, we use the matrix $\Mat{W}\in\Space{R}^{d\times d}$ to encode the graph $\Set{G}(f)$, where each element $[\Mat{W}]_{ij}\neq 0$ indicates the existence of edge $X_{i}\rightarrow X_{j}$.
% The previous work \cite{zheng2020learning} defines the non-parametric acyclicity upon the entries of matrix $\Mat{W}(f)=\Mat{W}(f_{1},\cdots,f_{d})$ as:
Then we define the matrix $\Mat{W}$ as: 
\begin{gather}
    [\Mat{W}(f)]_{ij}:=|\partial_i f_j|, \label{differentiate_of_f}
\end{gather}
where $\partial_i f_j = \frac{\partial f_{j}(X_1,\cdots,X_d)}{\partial X_i}$ is $f_{j}$'s partial derivative \wrt $X_{i}$. 
Thus, $f_{j}$ does not depend on $X_{i}$ if and only if $|\partial_{i}f_{j}|=0$. 
With Equation \eqref{differentiate_of_f}, we can exploit linear and nonlinear functions $f$ to characterize acyclicity in linear and nonlinear SEMs, respectively.

\vspace{5pt}
\noindent \textbf{Linear SEM.}
% In the case of linear SEM \cite{zheng2018dags} where $f_{i}(\Mat{X})=\Trans{\Mat{w}}_{i}\Mat{X}$, we formulate $f(\Mat{X})$ as a linear matrix multiplication: $f(\Mat{X})=\Mat{X}\Mat{A}$, where $\Mat{A}\in\Space{R}^{d\times d}$ denotes the coefficient matrix.
In the case of linear SEM \cite{zheng2018dags}, 
% we have $f_{j}(\textbf{u})=\Trans{\Mat{w}}_{j}\textbf{u}$, with $\textbf{u} \in \mathbb{R}^{d}$ being the sample in $\Mat{X}$ and $\Mat{w}_{j}$ \wy{$\in \mathbb{R}^d$ being the coefficient vector of $f_j$}. Then 
we formulate $f(\Mat{X})$ as a linear matrix multiplication: $f(\Mat{X})=\Mat{X}\Mat{A}$ (with each instance $f(\Mat{x}_i) = \Trans{\Mat{A}}\Mat{x}_i$, $i\in\{1,\cdots. n\}$), where $\Mat{A}\in\Space{R}^{d\times d}$ denotes the coefficient matrix.
This formulation frames the acyclicity in Equation \eqref{differentiate_of_f} as:
\begin{equation}\label{definition_of_W_linear}
    [\Mat{W}(f)]_{ij}= |A_{ij}|.
\end{equation}
\vspace{5pt}
\noindent \textbf{Nonlinear SEM.}
In the case of nonlinear SEM \cite{zheng2020learning}, we define $f_{i}(\Mat{X})$ as a multilayer perceptron (MLP) with $h$ hidden layers and an activation function $\sigma$ as:
\begin{equation}
\label{eq:expansion_of_fi}
    f_{i}(\Mat{X})=\text{MLP}(\Mat{X};\Mat{A}_{i}^{(1)},\cdots,\Mat{A}_{i}^{(h)})=\sigma(\cdots\sigma(\Mat{X}\Mat{A}_{i}^{(1)})\cdots)\Mat{A}_{i}^{(h)},
\end{equation}
where $\Mat{A}_{i}^{(l)}\in\Space{R}^{m_{l-1}\times m_{l}}$ is the learnable weight matrix of the $l$-th hidden layer for the $i$-th node, $m_{l}$ is the number of hidden units in the $l$-th layer, and $m_{0}=d$.
According to \cite{zheng2020learning}, we could define $\Mat{W}_\theta(f)$ as 
% this formulation frames the acyclicity in Equation \eqref{differentiate_of_f} as:
% \vspace{-3pt}
\begin{equation}\label{definition_of_W_nonlinear}
    [\Mat{W}_\theta(f)]_{ij}=\norm{i\text{-th column}(\Mat{A}_{j}^{(1)})}_{2}.
\end{equation} 
Then the acyclicity constraint could be latched on $\Mat{W}_\theta(f)$, which serves as the replacement of the intractable $\Mat{W}(f)$ for nonlinear setting. 
% becomes $h(\Mat{W}_\theta(f)) = 0$ in MLP setting. 

% \yyc{Question: What is $h(\cdot)$ ?}

%% file: chapters/4_methodology.tex
\section{Methodology}
\label{methodology}
% \vspace{-5pt}

In this section, we first present differentiable invariant causal discovery (DICD) to conduct causal structural learning over multiple environments. 
Then, we give detailed formulations in both linear and nonlinear settings. 
% \vspace{-8pt}
\subsection{Differentiable Invariant Causal Discovery (DICD)}
% \vspace{-5pt}

Despite the great success, the ERM paradigm easily captures spurious correlations between variables by over-reconstructing the observations \cite{He0SXLJ21} (See the toy example in Table \ref{tab:example}).
It is essential to distinguish the causation edges from the spuriously-correlated edges.
Towards this end, we get access to the multi-environment information and incorporate the idea of invariant learning, so as to frame the causal discovery task as identifying environment-invariant causation edges and discarding environment-dependent correlations. 

First, we argue that the multi-environment, in various forms like explicit or implicit metadata, is common in many real-world datasets and can partition the data into different groups or domains.
The environments in LFW dataset \cite{huang2008labeled}, for instance, can divide the data into black-and-white and colorful photographs.
In ImageNet dataset \cite{DengDSLL009}, the data can be grouped by different sources and years of images.
Moreover, several benchmarks for the domain-shifts problems have been provided by WILDS \cite{koh2021wilds}.
Among them, Camelyon17 \cite{bandi2018detection} includes the images collected from five hospitals that serve as different environments, while Amazon \cite{ni2019justifying} provides review texts from different reviewers that can work as environments.
Towards the end, the formal definition of the different environments in our paper is given as follow:

\begin{definition} \label{definition_of_environment}
For two datasets generated from the same structure equation model as shown in Equation (\ref{eq:data_generation}). if there exists $i \in \{1,\cdots, d\}$ such that the distribution of $z_i$ is different across these datasets and $X_i$ is not the source node (see Definition \ref{definition_of_source_node}) in the corresponding graph. Then we say these two datasets are drawn from different environments. 
\end{definition}

\begin{definition}\label{definition_of_source_node}
    We define the node in the graph with no parents as the \emph{Source Node}.
\end{definition}

With the environment information, we utilize the invariant learning to conduct multi-environment causal discovery --- the function parameters of SEM from the target DAG should remain optimal across all the environments. 
Guided by this idea, our DICD consists of two modules: (1) Invariant structural model $\Mat{S}$, which presents the structure of DAG, \ie a binary adjacency matrix. By ``invariant'', we mean that $\Mat{S}$ should be consistent across environments; and (2) Optimal causal function $f$, which depicts the causal relation of variables. By ``optimal'', we mean that $f$ should be optimal over all the environments once the structural model $\Mat{S}$ is given.
Considering the example in Table \ref{tab:example} again, $\Mat{S}$ is the DAG structure, while $f$ refers to the coefficients of edges.
For an invariant DAG structure, completely various optimal coefficients will be learned in different environments.
As a result, by utilizing the environment information, incorrect DAGs with smaller reconstruction losses can be excluded based on S and $f$.
% On the basis of $\Mat{S}$ and $f$ across environments, we may exclude the wrong DAGs with lower reconstruction losses than the ground-truth DAG.

Having environment-aware groups of observations, we define the empirical risk within the environment $e \in \Set{E}$ as:
\begin{gather*}
    \Lapl^{e}(\Mat{S}\circ f)=\frac{1}{n_{e}}\sum_{i=1}^{n_e}l(\Mat{X}^{e}_{i},(\Mat{S}\circ f)(\Mat{X}^{e}_i)),
\end{gather*}
where %$(\Mat{S}\circ f)_{i}(\Mat{X}^{e})$ is the vector of the $i$-th column of the reconstructed matrix $(\Mat{S}\circ f)(\Mat{X}^{e})$, and 
$\circ$ refers to the composition of the two functions, and $n_e$ is the number of samples in environment $e$. 
We then build two constraints on $\Mat{S}$ and $f$ across all environments, and establish the model of DICD:
% Towards this end, we propose a new paradigm to minimize the environment-aware empirical risks between the observational and reconstructed data, with the constraints on $\Mat{S}$ and $f$ across environments:
\begin{align}
    \min_{\Mat{S}\circ f}&~\sum_{e\in\Set{E}}\Lapl^{e}(\Mat{S}\circ f),\label{loss_function}\\
    \text{s.t.}&~~\Set{G}(\Mat{S})=\Set{G}(f)\in\text{DAG},\label{equal_graph_constraint}\\
    &~~f=\arg\min_{\bar{f}}\Lapl^{e}(\Mat{S}\circ\bar{f}),\quad\forall e\in\Set{E}.\label{optimal_constraint}
\end{align}
Equation \eqref{equal_graph_constraint} states the DAG constraint, where the DAGs represented by $\Mat{S}$ and $f$ are equivalent. In other words, $f$ latches on $\Mat{S}$'s DAG structure.
Equation \eqref{optimal_constraint} claims that $f$ refers to the optimal causal model fitting Equation \eqref{loss_function} across all environments.
Following NOTEARS \cite{zheng2020learning}, the DAG constraints can rewrite as hard DAG equation constraints:
\begin{gather}\label{hofw}
    \Set{G}(\Mat{S})=\Set{G}(f),\quad h(\Mat{W}(f))=0,
\end{gather}
where $\Mat{W}(f)$ is defined as Equation \eqref{differentiate_of_f} in the linear setting and should be replaced with $\Mat{W}_\theta(f)$ in Equation \eqref{definition_of_W_nonlinear} for the nonlinear setting. Besides $h(\Mat{W})=\text{tr}(e^{\Mat{W}\circ\Mat{W}})-d$. Here, $\circ$ refers to the Hadamard product (\aka the element-wise product), and tr($e^{\Mat{W}\circ\Mat{W}}$) is the trace of $e^{\Mat{W}\circ\Mat{W}}$. 

% \vspace{-8pt}
\subsection{Linear SEM}
% \vspace{-5pt}
When $f$ characterizes acyclicity in the linear SEM, we have $f(\Mat{X})=\Mat{X}\Mat{A}$ and $\partial_{i} f_{j}(\Mat{X}) = A_{ij}$ (\cf Equation \eqref{definition_of_W_linear}).
With the invariant structure model $\Mat{S}$ as the binary matrix, we can convert $\Mat{S}\circ f$ as $\Mat{S}\circ\Mat{A}$, which is the Hadamard product of $\Mat{S}$ and $\Mat{A}$.
However, the bilevel optimization formulation of the DICD model is difficult to solve and is susceptible to failure due to over-parametrization in our setting.
We further simplify the learning of these two matrices as the optimization of a new matrix $\Mat{A}_{\Mat{S}}=\Mat{S}\circ\Mat{A}\in\Space{R}^{d\times d}$.
As such, we restrict the loss function in Equation \eqref{loss_function} with the DAG constraint in Equation \eqref{equal_graph_constraint} as:
\begin{gather}
    \min_{\Mat{A}_{\Mat{S}}}\sum_{e\in\Set{E}}\Lapl^{e}(\Mat{A}_{\Mat{S}}),\quad\text{s.t.}~h(\Mat{A}_{\Mat{S}})=0. \nonumber
\end{gather}
This relaxed version only focuses on optimizing over $\Mat{A}_{\Mat{S}}$ during training. After learning $\Mat{A}_{\Mat{S}}$, we can simply reset coefficient matrix $\Mat{A}=\Mat{A}_{\Mat{S}}$ and let $\Mat{S}$ be the binary indicator on $\Mat{A}_{\Mat{S}}$'s elements.

We then explore the tractable formulation for the optimality constraint across environments in Equation \eqref{optimal_constraint} with the following theorem:
\begin{theorem}\label{prop_of_linear_optimality}
 In the linear setting, we define a matrix variable as $\Mat{B} \in \Space{R}^{d\times d}$. After replacing $\Mat{S} \circ \Mat{A}$ with $\Mat{A}_{\Mat{S}}$, the constraint Equation (\ref{equal_graph_constraint}-\ref{optimal_constraint}) satisfy the following necessary condition:
    \begin{gather}\label{optimal_constraint_final}
        \norm{\frac{\partial\Lapl^{e}(\Mat{A}_{\Mat{S}}\circ\Mat{B})}{\partial\Mat{B}}|_{\Mat{B}=\Mat{1}}}^{2}_{2}=0,\quad\forall e\in\Set{E},
    \end{gather}
    where $\textbf{1}$ is the all-one matrix.
\end{theorem}
\begin{proof}
    The basic optimality constraint across environments for Equation (\ref{equal_graph_constraint}-\ref{optimal_constraint}) can be described as:
    \begin{align}\label{optimal_A}
        \Mat{A}^{*}=\arg\min_{\Mat{A}}\Lapl^{e}(\Mat{S}\circ\Mat{A}), \quad\forall e\in\Set{E}, \nonumber\\
        \text{s.t.}~\Set{G}(\Mat{S})=\Set{G}(\Mat{A})\in\text{DAG}.
    \end{align}
    Then we introduce a new matrix variable $\Mat{B}$ and insert it into Equation \eqref{optimal_A}. The optimal of $\Mat{B}$ can be defined as:
    \begin{gather} \label{replace_B}
        \Mat{B}^{*}=\arg\min_{\Mat{B}}\Lapl^{e}(\Mat{S}\circ\Mat{A}^{*}\circ\Mat{B}), \quad \forall e\in \mathcal{E}.
    \end{gather}
    We assert $\Mat{B}^{*}$ could be the all-one matrix $\Mat{1}$, which indicates:
    \begin{gather}\label{optimal_B_is_E}
        \Lapl^{e}(\Mat{S}\circ\Mat{A}^{*}\circ\Mat{1})\leq\Lapl^{e}(\Mat{S}\circ\Mat{A}^{*}\circ\Mat{B}),\quad\forall\Mat{B}\in\Space{R}^{D\times D}, \forall e\in \mathcal{E}.
    \end{gather}
    Assuming there exists $\Mat{B}'$ that satisfies $\displaystyle \Lapl^{e}(\Mat{S}\circ\Mat{A}^{*}\circ\Mat{B}')<\Lapl^{e}(\Mat{S}\circ\Mat{A}^{*}\circ\Mat{1})$, then replacing $\Mat{A}^{*}\circ\Mat{B}'$ with $\Mat{A}'$, we could have $\displaystyle \Lapl^{e}(\Mat{S}\circ\Mat{A}')<\Lapl^{e}(\Mat{S}\circ\Mat{A}^{*})$. 
    % \begin{gather*}
    This result obviously contradicts with Equation \eqref{optimal_A}. 
    As such, Equation \eqref{optimal_B_is_E} holds. 
    
    We replace $\Mat{S}\circ\Mat{A}$ with $\Mat{A}_{\Mat{S}}$. Now if $\Mat{A}$ is the optimal parameter of Equation \eqref{optimal_A} (\ie $\Mat{A}^*$), then $\Mat{A}_{\Mat{S}}$ becomes $\Mat{A}_{\Mat{S}}^*$, which yield the following equation: 
    \begin{equation*}
        \Mat{1}=\arg\min_{\Mat{B}}\Lapl^{e}(\Mat{A}_{\Mat{S}}^*\circ\Mat{B}), \quad \forall e\in \mathcal{E}. 
    \end{equation*} 
    According to the first-order optimality condition, we have: $\displaystyle \frac{\partial \mathcal{L}^e(\Mat{A}_S^* \circ \Mat{A})}{\partial \Mat{B}}\big|_{\Mat{B} = \Mat{1}} = \Mat{0}$. 
    This concludes the proof.
\end{proof}

Based on Theorem \ref{prop_of_linear_optimality}, we could find that the intractable optimality constraint across environments is now differentiable with the objective Equation \eqref{optimal_constraint_final}. 

We establish the DICD model in linear cases as follows:
% \begin{equation}
%     \min_{\Mat{A}_{\Mat{S}}}\sum_{e\in\Set{E}}\Lapl^{e}(\Mat{A}_{\Mat{S}}) + \lambda\sum_{e\in\Set{E}}\norm{\frac{\partial\Lapl^{e}(\Mat{A}_{\Mat{S}}\circ\Mat{A})}{\partial\Mat{A}}|_{\Mat{A}=\Mat{1}}}^{2}_{2},\quad
%     \text{s.t.}\quad h(\Mat{A}_{\Mat{S}})=0. \label{objective_linear}
% \end{equation}
% \begin{equation}
%  \begin{array}{ll}  \displaystyle \min_{\Mat{A}_{\Mat{S}}} & \sum_{e\in\Set{E}}\Lapl^{e}(\Mat{A}_{\Mat{S}}) + \lambda\sum_{e\in\Set{E}}\norm{\frac{\partial\Lapl^{e}(\Mat{A}_{\Mat{S}}\circ\Mat{B})}{\partial\Mat{B}}|_{\Mat{B}=\Mat{1}}}^{2}_{2},\\
%     \text{s.t.} & h(\Mat{A}_{\Mat{S}})=0. \label{objective_linear}
%     \end{array}
% \end{equation}
\begin{align}
    \min_{\Mat{A}_{\Mat{S}}} & \sum_{e\in\Set{E}}\Lapl^{e}(\Mat{A}_{\Mat{S}}) + \lambda\sum_{e\in\Set{E}}\norm{\frac{\partial\Lapl^{e}(\Mat{A}_{\Mat{S}}\circ\Mat{B})}{\partial\Mat{B}}|_{\Mat{B}=\Mat{1}}}^{2}_{2}, \nonumber \\
    \text{s.t. } & h(\Mat{A}_{\Mat{S}})=0. \label{objective_linear}
\end{align}

This objective function potentially incorporates the invariant structure and the optimal coefficients into the single variable $\Mat{A}_{\Mat{S}}$, thus only adds one more penalty for training compared with NOTEARS.

% \vspace{-8pt}
\subsection{Nonlinear SEM}
% \vspace{-5pt}
In nonlinear settings, We shall continue to concentrate on the scenario with scalar-valued variables (\ie $\textbf{X}\in\mathbb{R}^{n\times d}$) for simplicity. However, the vector-valued variables (\ie $\textbf{X} \in \mathbb{R}^{n\times d\times d_x}$) could be simply incorporated into DICD.
% For simplicity, we will continue to focus on the situation with scalar-valued variables (\ie $\textbf{X}\in\mathbb{R}^{n\times d}$). However, our method could be easily generalized into the vector-valued variables (\ie $\textbf{X} \in \mathbb{R}^{n\times d\times d_x}$).

NOTEARS-MLP~\cite{zheng2020learning} deploys the continuous DAG constraint for nonlinear SEM to the first layer of all the \textrm{MLP}s, as shown in Equation \eqref{definition_of_W_nonlinear}. 
Inspired by this, we establish the relationship between the structure matrix $\Mat{S}$ and the first layer of the \textrm{MLP}s. In other words, ensuring the parameters of \textrm{MLP} latch on the structure $\Mat{S}$ in order to satisfy $\mathcal{G}(\Mat{S}) = \mathcal{G}(f)$ constraint. 
The precise solution to the equation $\mathcal{G}(\Mat{S}) = \mathcal{G}(f)$ is to maintain the following restriction throughout the optimization process:
% The detailed solution for $\mathcal{G}(\Mat{S}) = \mathcal{G}(f)$ is to hold the following constraint during the optimization: 
\begin{equation}\label{eq:nonlinear_constraint}
    \forall i,j, \quad \Mat{S}_{ji}=0 \Longrightarrow ||j\textrm{-th column}(\Mat{A}_i^{(1)})||_2 = 0,
\end{equation}
where $\Mat{A}_i^{(1)}$ is the matrix parameter of the first layer in the $i$-th \textrm{MLP}. The intuition behind the above equation is that setting $j\textrm{-th column}(A_i^{(1)})$ to zero would block the correlation from $j$ to $i$, corresponding to $\Mat{S}_{ji} = 0$. 

Though the first term in constraint Equation \eqref{hofw} could be formulated in Equation \eqref{eq:nonlinear_constraint}. The optimization of Equation (\ref{loss_function}-\ref{optimal_constraint}) is still intractable. We need to construct an differentiable term to replace the constraint Equation \eqref{eq:nonlinear_constraint}. 
% With Equation (\ref{eq:expansion_of_fi}), w
To address this problem, we propose to simplify the optimization of $\Mat{S}$ and $f$ to optimizing a new function $f_{\Mat{S}} =  (f_{\Mat{S}1},\cdots,f_{\Mat{S}d})$, where $f_{\Mat{S}i} =  \Mat{S}\circ f_i$. 
The exact correlation is described in the following equation: 
\begin{align}
    f_{\Mat{S}i} &= \Mat{S}\circ f_i = \textrm{MLP}(\Trans{Re([\Mat{S}]_i, m_1)} \circ \Mat{A}_i^{(1)}, \Mat{A}_i^{(2)}, \cdots, \Mat{A}_i^{(h)})
    \nonumber \\
    &= \textrm{MLP}(\Mat{A}_{\Mat{S}_i}^{(1)}, \Mat{A}_i^{(2)}, \cdots, \Mat{A}_i^{(h)}), \nonumber
\end{align}
where $m_1$ comes from $\Mat{A}_i^{(1)} \in \mathbb{R}^{m_{0}\times m_1}$, and $\Trans{Re([\Mat{S}]_i, m_1)}$ refers to $\Trans{[\underbrace{[\Mat{S}]_i,\cdots,[\Mat{S}]_i}_{m_1\textrm{ times}}]}$, with $[\Mat{S}]_i$ being the $i$-th column of $\Mat{S}$.
We denote $Re([\Mat{S}]_i, m_1) \circ \Mat{A}_i^{(1)}$ as $\Mat{A}_{\Mat{S}i}^{(1)}$. 
Similar to the linear setting, once $f_{\Mat{S}}$ is learned, we can simply set $f = f_{\Mat{S}}$ and $\Mat{S}$ as the adjacency matrix of $\mathcal{G}(f)$.
Then $f$ and $\Mat{S}$ will share the same graph structure. 
Now we propose the following theorem for nonlinear setting:
\begin{theorem}\label{prop_of_nonlinear_optimality}
    In nonlinear system, we denote $\Mat{A}_i^{(l)}, l\in \{1,\cdots, d\}$ as the parameter of the $l$-th layer of \textrm{MLP} $f_i$, and $\Mat{A}_{\Mat{S}i}$ is the parameter of the 1st layer of \textrm{MLP} $f_{\Mat{S}i}$, $i \in \{1, \cdots, d\}$. Then after replacing $Re([\Mat{S}]_i, m_1) \circ \Mat{A}_i^{(1)}$ with $\Mat{A}_{\Mat{S}i}^{(1)}$, we define a matrix variable as $\Mat{B} \in \Space{R}^{m_1\times d}$, the optimal condition Equations (\ref{equal_graph_constraint}-\ref{optimal_constraint}) except for the DAG part satisfy the following necessary condition: 
    \begin{align}
       \sum_{i=1}^d \Big|\Big| &\frac{\partial \mathcal{L}^e(\textrm{MLP}(\Mat{A}_{\Mat{S}i}^{(1)}\circ\Mat{B}, \Mat{A}_i^{(2)}, \cdots, \Mat{A}_i^{(n)}))}{\partial \Mat{B}}\Big|_{\Mat{B}=\textbf{1}}  \Big|\Big|_2^2 = 0, \quad \nonumber \\
       &\forall e \in \mathcal{E}.  \label{optimal_constraint_nonlinear_final}
    \end{align}
\end{theorem}

\begin{proof}
The basic optimal condition described in Equation (\ref{equal_graph_constraint}-\ref{optimal_constraint}) without the DAG constraint could be written as:
% \begin{equation}\label{optimal_A_nonlinear}
% \Mat{A}_{i}^{(1)*} = \textrm{arg}\min_{A_i^{(1)}} \mathcal{L}^e(\textrm{MLP}(Re([\Mat{S}]_i, m_1)\circ\Mat{A}_i^{(1)},\cdots, \Mat{A}_i^{(n)})) \nonumber
% \end{equation}

\begin{align}
    \Mat{A}_{i}^{(1)*} &= \textrm{arg}\min_{A_i^{(1)}} \mathcal{L}^e(\textrm{MLP}(Re([\Mat{S}]_i, m_1)\circ\Mat{A}_i^{(1)},\cdots, \Mat{A}_i^{(n)})), \nonumber \\
    \forall e &\in \mathcal{E}, \quad \textrm{s.t.} \,\, \mathcal{G}(\Mat{S}) = \mathcal{G}(f), \label{optimal_A_nonlinear} 
\end{align}
where $Re([\Mat{S}_i], m_1)$ means $\Trans{[\underbrace{[\Mat{S}]_i,\cdots,[\Mat{S}]_i}_{m_1\textrm{ times}}]}$, with $[\Mat{S}]_i$ being the $i$-th column of $\Mat{S}$. 
Note that the following equations with $e$ in this proof indicate all the environments, we omit the condition $\forall e \in \mathcal{E}$ in the following to simplify the proof.

\noindent Similar to the proof of Proposition \ref{prop_of_linear_optimality}, we define a new matrix $\Mat{B} \in \mathbb{R}^{m_1\times d}$, and insert this term into the above equation, then we can have:
\begin{equation}
\label{eq:nonlinear_optimality_constraint}
\Mat{B}^* = \textrm{arg}\min_{B} \mathcal{L}^e(\textrm{MLP}(Re([\Mat{S}]_i, m_1)\circ\Mat{A}_{i}^{(1)*}\circ\Mat{B}, \cdots, \Mat{A}_i^{(n)})).
\end{equation}

\noindent We argue that $\Mat{B}^*$ could be $\Mat{1}$, a matrix of dimensions $m_1\times d$ filled with ones. This is expressed as:
\begin{equation}
\Mat{1} = \textrm{arg}\min_{B} \mathcal{L}^e(\textrm{MLP}(Re([\Mat{S}]_i, m_1)\circ\Mat{A}_{i}^{(1)*}\circ\Mat{B}, \cdots, \Mat{A}_i^{(n)})). \nonumber
\end{equation}

\noindent Then replacing $Re([\Mat{S}]_i, m_1)\circ\Mat{A}_{i}^{(1)*}$ with $\Mat{A}_{\Mat{S}i}^*$ would yield:
\begin{equation}
    \Mat{1} = \textrm{arg}\min_{B} \mathcal{L}^e(\textrm{MLP}(\Mat{A}_{\Mat{S}i}^*\circ\Mat{B}, \Mat{A}_i^{(2)},\cdots, \Mat{A}_i^{(n)})). \nonumber
\end{equation}
Again, by the first-order optimality condition, we have Equation (\ref{optimal_constraint_nonlinear_final}).
\end{proof}

With Theorem \ref{prop_of_nonlinear_optimality}, we could transform the intractable constraint in Equation \eqref{optimal_constraint} into the differentiable term subject to the first layer of MLP. The detailed formulation in the nonlinear setting can be expressed as:
% \begin{equation}
%     \min_{f_{\Mat{S}}} \sum_{e\in \mathcal{E}} L^e(f_{\Mat{S}}) + \lambda \sum_{e\in\mathcal{E}} \sum_{i=1}^d\Big|\Big|\frac{\partial \mathcal{L}^e(\textrm{MLP}(\Mat{A}_{\Mat{S}i}\circ\Mat{A}, \Mat{A}_i^{(2)}, \cdots, \Mat{A}_i^{(n)}))}{\partial \Mat{A}}\Big|_{\Mat{A}=\textbf{1}}\Big|\Big|_2^2, \quad
%     \textrm{ s.t. } h(\textbf{W}(f_{\Mat{S}})) = 0 \label{objective_nonlinear}
% \end{equation}
\begin{align}
    & \min_{f_{\Mat{S}}} \sum_{e\in \mathcal{E}} L^e(f_{\Mat{S}}) \nonumber \\
    & + \lambda \sum_{e\in\mathcal{E}} \sum_{i=1}^d\Big|\Big|\frac{\partial \mathcal{L}^e(\textrm{MLP}(\Mat{A}_{\Mat{S}i}\circ\Mat{B}, \Mat{A}_i^{(2)}, \cdots, \Mat{A}_i^{(n)}))}{\partial \Mat{B}}\Big|_{\Mat{B}=\textbf{1}}\Big|\Big|_2^2, \nonumber \\
    & \textrm{ s.t. } h(\Mat{W}_\theta(f_{\Mat{S}})) = 0. \label{objective_nonlinear}
\end{align}
This objective function also potentially uses $f_{\Mat{S}}$ to both incorporate $f$ and the invariant structure matrix $\Mat{S}$. 
% Compared to NOTEARS-MLP, only one regularization term is added. 

% ========================================

% ========================================

%% file: chapters/6_theoretical_analysis.tex
% \vspace{-10pt}
\section{Theoretical Analysis}
\label{theoretical_analysis}
% \vspace{-10pt}
In this section, we aim to provide the sufficient conditions for identifiability of DICD in the linear SEM systems. We will leave the discussion about which assumptions can or cannot be
further relaxed in future work.
% \begin{definition}\label{definition_of_source_node}
%     We define the node in the graph with no parents as the \emph{Source Node}.
% \end{definition}
\begin{theorem}\label{theorem:identifiability}
For linear SEMs systems with Gaussian additive noises as in Equation \eqref{eq:data_generation}, if for any $X_i \in \Set{V}$ that is not the source node, there exist two environments $e_1, e_2\in\mathcal{E}$, such that:
\begin{align}
    % & (\sigma_i^{e_1})^2 \neq (\sigma_i^{e_2})^2 \\
    & Var(z_i^{e_1}) \neq Var(z_i^{e_2}), \label{eq:same_noise}\\
    % & \forall X_j \in V\backslash\{X_i\}, (\sigma_j^{e_1})^2 = (\sigma_j^{e_2})^2
    & \forall X_j \in \Set{V}\backslash\{X_i\}, Var(z_j^{e_1}) = Var(z_j^{e_2}), \label{eq:different_noise}
\end{align}
where $Var(\cdot)$ represents the variance, $z_j^{e_1}$ and $z_j^{e_2}$ are the additive noise from $Pa(X_j)$ to $X_j$ in $e_1$ and $e_2$, respectively. Then the causal structure is \textbf{identifiable}. 
\end{theorem}

Here we emphasize that the graph and the coefficients that satisfy the following sufficient conditions and Equation (\ref{loss_function}-\ref{optimal_constraint}) will be exactly the true causal graph and the true coefficients, corresponding to the identifiable graph. 

Theorem \ref{theorem:identifiability} indicates that our DICD is guaranteed to retrieve the true causal graph in linear systems, when the diversity of environments is adequate. 

The assumptions in Theorem \ref{theorem:identifiability} are restrictive mainly because Equation \eqref{eq:different_noise}) requires that the variances of the nodes other than $X_{i}$ to be the same across environments $e_1$ and $e_2$.
In our proof, Equation \eqref{eq:different_noise} is only used for proving Lemma \ref{lemma:stable_graph}. However, we will show that, even if we don't assume Equation \eqref{eq:different_noise}, we may still obtain Equation \ref{eq:contradiction_from_wj0} in the proof of Lemma \ref{lemma:stable_graph} below, which leads to the proof of Lemma \ref{lemma:stable_graph}. 

With Equation \eqref{formulation_of_optimal_w}, we know the sufficient and necessary condition of $\hat{w}_{j_0}^{e_1} = \hat{w}_{j_0}^{e_2}$ is:
\begin{align*}
    & \frac{\mathbb{E}_{e_1}[X_iX_{j_0}] - \sum_{k\in Pa_s(i)\backslash\{j_0\}} \hat{w}^e_k \mathbb{E}_{e_1}[X_k X_{j_0}] }{\mathbb{E}_{e_1}[X_{j_0}^2]} \\
    &= \frac{\mathbb{E}_{e_2}[X_iX_{j_0}] - \sum_{k\in Pa_s(i)\backslash\{j_0\}} \hat{w}^e_k \mathbb{E}_{e_2}[X_k X_{j_0}] }{\mathbb{E}_{e_2}[X_{j_0}^2]}.
\end{align*}
Since we do not have Equation \eqref{eq:different_noise}), every single term in the above equation is not necessarily equivalent. The probability of the combinations of all these non-equivalent terms being equivalent is very small. Thus during the implementation, Lemma \ref{lemma:stable_graph} is easy to be true. Then the other parts of the proof of Theorem \ref{theorem:identifiability} will remain the same.

Next, we provide the proof skeleton of Theorem \ref{theorem:identifiability} and leave the full proof in Appendix \ref{proof_of_identifiability}. 

There are three main steps in our proof. First, we prove that if the causal structure is correct, then the optimal coefficients are the ground truth coefficients (Lemma \ref{lemma: optimal_parameter_is_true_parameter}). Second, we show that the stable graph (see definition \ref{definition:stable_graph}) with the minimal sum of reconstruction loss from all environments is exactly the true graph (Theorem \ref{theorem_of_justification}). In the end, we can prove that our algorithm could yield the ground truth graph and coefficients with adequate environments. 

To begin with, we give the following lemma to show that our DICD could yield the true coefficient in the linear SEM under the true causal structure. 

\begin{lemma}
\label{lemma: optimal_parameter_is_true_parameter}
Given the true graph $G_0$, the corresponding structure $\Mat{S}_0$ is the adjacency matrix of $G_0$. $\forall e \in \mathcal{E}$, we denote the optimal parameters for the coefficients in environment $e$ as:
    \begin{equation}\label{regression_result_for_e}
        \hat{\Mat{W}}^e = arg\min_{\overline{\Mat{W}}} \mathcal{L}^e(\Mat{S}_0\circ \overline{\Mat{W}}), \quad \st \mathcal{G}(\Mat{S}_0) = \mathcal{G}(\overline{\Mat{W}}).
    \end{equation}
    Then we have $\hat{\Mat{W}}^e = \Mat{W}_0, \forall e \in \mathcal{E}$, where $\Mat{W}_0$ is the ground truth coefficients. 
\end{lemma}

The following definitions are necessary for the rest of this section.

\begin{definition}
\label{definition:stable_graph}
For a given graph $G$, if there exist $ \Mat{S}, \Mat{W}$ such that:
    \begin{gather} 
        \mathcal{G}(\Mat{S}) = \mathcal{G}(\Mat{W}) = G, \nonumber \\
        \Mat{W} = arg\min_{\overline{\Mat{W}}} \mathcal{L}^e(\Mat{S}\circ \overline{\Mat{W}}), \quad \forall e\in \mathcal{E}, \label{regression}
    \end{gather}
    then we call $G$ a stable graph.
\end{definition}

\begin{definition}
For two nodes $X_1$ and $X_2$ in a DAG, if $X_2$ is reachable from $X_1$, then $X_1$ is a predecessor of $X_2$. We denote all the predecessors of $X_2$ in graph $G_0$ as $Pre_0(X_2)$. 
\end{definition}

With Definition \ref{definition:stable_graph}, we further denote the parents of the variable $X_i$ in $G_0$ as $Pa_0(X_i)$ and the corresponding indexes as $Pa_0(i)$. 
We propose the following Lemma \ref{lemma:stable_graph} to demonstrate that the causal directions between any two variables cannot violate each other in any stable graph and the true causal graph, which could serve as the pre-conditions in Theorem \ref{theorem_of_justification} to prove it. 

\begin{lemma}\label{lemma:stable_graph}
For any given stable graph $G_s$, if we assume the conditions in Theorem \ref{theorem:identifiability} hold, then $\forall X_i \in \Set{V}$ which is not a source node, we have $X_i\notin Pre_0(X_j)$ for any $X_j \in Pa_s(X_i)$. 
\end{lemma}

\noindent Finally, Theorem \ref{theorem:identifiability} can be implied by the following theorem: 

\begin{theorem}
\label{theorem_of_justification}
For any given stable graph $G_s$ and the ground truth graph $G_0$, we denote their corresponding structures as $\Mat{S}_0$ and $\Mat{S}_s$, respectively. We further denote their corresponding consistent optimal parameters as $\Mat{W}_s$ and $\Mat{W}_0$, respectively. Then we have: 
    \begin{equation}\label{loss_comparison}
        \sum_{e\in\mathcal{E}}\mathcal{L}^e(\Mat{S}_0\circ \Mat{W}_0) \leq \sum_{e\in\mathcal{E}}\mathcal{L}^e(\Mat{S}_s\circ \Mat{W}_s),
    \end{equation}
    and the equation holds only for $\Mat{W}_0 = \Mat{W}_s$.
\end{theorem}

The proof of Theorem \ref{theorem_of_justification} is 
given in Appendix \ref{proof_of_identifiability}. Theorem \ref{theorem_of_justification} imples that, with the conditions in Equation \eqref{equal_graph_constraint} and Equation \eqref{optimal_constraint}, The ground truth graph will be yielded by Equation \eqref{loss_function}.

%% file: chapters/5_experiments.tex
% \vspace{-10pt}
\section{Experiments}
In this section, we study the empirical performance of our proposed method. We aim to answer the following research questions: 
\begin{itemize}[leftmargin=*]
    \item \textbf{RQ1}: How does DICD perform compared to the previous methods in both linear and nonlinear settings? 
    \item \textbf{RQ2}: How do DICD and other baselines perform with various factors (\ie the number of environments, density of graph). 
    \item \textbf{RQ3}: How does DICD perform on real-world datasets compared with other applicable baselines? 
\end{itemize}

% In this section, we study the empirical performance of our proposed method given environments information from three aspects: 
% 1) The overall performances of DICD compared to the previous methods in both linear and nonlinear settings. 2) The effects of various factors (\ie the number of environments, density of graph). 3) The performances of DICD and other applicable baselines on the real-world dataset. 

% \vspace{-8pt}
\subsection{Experimental Settings}
\label{experimental_settings}
% \vspace{-5pt}
We now provide the detailed settings for our experiments. The descriptions for the synthetic datasets and the real-world dataset are in Section \ref{subsub:synthetic_datasets} and Section \ref{subsub:realworld_dataset}, respectively. For all the experiments in this paper, we all generate 10 datasets for each graph setting and report the mean and standard deviation.

% We provide the descriptions of the baselines, hyperparameter settings, as well as evaluation protocols in Appendix \ref{subsub:baselines}, \ref{hyperparameter_settings} and \ref{subsub:evaluation_protocols}, respectively.

\subsubsection{Baselines}
\label{subsub:baselines}
We select four state-of-the-art causal discovery methods for comparison:
\begin{itemize}[leftmargin=*]
\vspace{5pt}
\item CD-NOD~\cite{CD-NOD} is a constrained-based causal discovery method designed for heterogeneous datasets, \ie datasets from different environments. CD-NOD utilizes the independent changes across environments to determine the causal orientations, and proposes constrained-based and kernel-based methods to find the causal structure.
\vspace{5pt}
\item NOTEARS~\cite{zheng2018dags} is specifically designed for linear setting, and is also the backbone of DICD in linear cases. NOTEARS estimates the true causal graph by minimizing the reconstruction loss with the continuous acyclicity constraint. We re-implement NOTEARS with replacing the L-BFGS-B iteration with Adam gradient descent, which could yield compatible performance and more importantly, could be deployed on GPU. 
\vspace{5pt}
\item NOTEARS-MLP~\cite{zheng2020learning} is specifically designed for nonlinear setting, which also serves as the foundation of DICD in nonlinear situations. NOTEARS-MLP approximates the generative SEM model by MLP while only constraining the first layer of the MLP with the continuous acyclicity constraint. 
\vspace{5pt}
\item DAGGNN\cite{yu2019dag} formulates causal discovery with variational autoencoder, where the encoder and decoder are all graph neural networks. Choosing the evidence lower bound as the loss function and slightly modifying the acyclicity constraint, DAGGNN could manage to recover the weighted adjacency matrix. 
\vspace{5pt}
\item NOCURL~\cite{yu2020dags} utilizes a two-step procedure:  first find an initial cyclic solution, then employ Hodge decomposition of graphs and learn an acyclic graph by projecting the cyclic graph to the gradient of a potential function. 
\vspace{5pt}
\item DARING~\cite{He0SXLJ21} imposes explicit residual independence constraint with an adversarial strategy. We choose the backbone as NOTEARS-MLP to conform with the settings above. 
\end{itemize}

\subsubsection{Hyperparameter Settings}
\label{hyperparameter_settings}
For linear settings, there are two hyper-parameters in total: $\lambda_1$ for the $l_1$-norm regularization term; $\lambda_D$ for the DICD penalty term. We tune $\lambda_1$ in $\{0.01, 0.1\}$ for NOTEARS and DICD. Besides, we tune $\lambda_D$ in $\{0.1, 1\}$ for DICD. Then for nonlinear settings, there are three hyper-parameters in total: $\lambda_1, \lambda_2, \lambda_D$, among which $\lambda_1$ and $\lambda_2$ are for the $l_1$-norm and $l_2$-norm regularization terms, respectively. We tune $\lambda_1, \lambda_2$ both in $\{0.01, 0.1\}$ and $\lambda_D$ in $\{0.1, 1\}$. The scheduler for $\lambda$ in Equation (\ref{objective_linear}) and Equation (\ref{objective_nonlinear}) is shown as follows:
\begin{equation*}
    \lambda = \left\{\begin{array}{cc}
        \displaystyle\frac{k}{K/3}\circ\lambda_D & k\leq K/3 \\
        \displaystyle\lambda_D & K/3 \leq k \leq 2K/3 \\
        \displaystyle\frac{K-k}{K/3}\circ\lambda_D & k\geq 2K/3 
    \end{array}\right.,
\end{equation*}
where $K$ is the estimated total step, and $k$ is the current iteration. The intuition behind this scheduler is that we need to let the model fit the data at the beginning, then as the training process goes, we need to enforce our penalty to help find the true causal graph. Then for the last stage, the graph structure has almost been inferred. We need to gradually remove our penalty to let the structural causal function with the given structure fit the data.

\input{chapters/table_linear}

\subsubsection{Evaluation Protocols}
\label{subsub:evaluation_protocols}
We use the three most popular metrics in causal discover: false discovery rate (FDR), true positive rate (TPR) and structural Hamming distance (SHD). Higher TPR stand for better performances, while FDR and SHD should be lower to represent the better strategies. 

% \subsubsection{Hyperparameter Settings}

\subsubsection{Synthetic Datasets}
\label{subsub:synthetic_datasets}
We conduct experiments on two synthetic datasets for linear and nonlinear settings. Besides, we apply our method DICD on Colored MNIST~\cite{mnist} dataset to explore its effectiveness on real-world datasets. As for the synthetic data, the ground truth DAG is generated from two random graph models: Erdos-Renyi (ER) and scale-free (SF), following ~\cite{zheng2020learning}. For the overall experimental comparison, we set the node degree as four. For the linear and nonlinear setting, we construct the environment variable $E$ to simulate the effects of the environments on the additive noises. For the linear setting, after generating the graph, we randomly select $\lfloor 0.3*d \rfloor$ nodes in this graph, and then build $\lfloor 0.3*d \rfloor$ new nodes. We call these new variables as the environment variables and denote them as $E$, and they satisfy the same distributions and are all independent from each other. Then we create $\lfloor 0.3*d \rfloor$ edges from each environment node to the each selected node. In this way, we could simulate different environments with varying the distribution of the environment variables $E$. 

Then given this new graph with $d + \lfloor 0.3*d \rfloor$nodes, we simulate random edge weights to obtain a new matrix $\Mat{W}\in\mathbb{R}^{(d+\lfloor 0.3*d \rfloor)\times(d+\lfloor 0.3*d \rfloor)}$. With $\Mat{W}$, we sample $\Mat{X}=\Mat{W}^T \Mat{X}+z \in \mathbb{R}^{d+\lfloor 0.3*d \rfloor}$ with $z$ from Gaussian noise model to generate 10 random datasets $\Mat{X}_e \in \mathbb{R}^{n\times (d+\lfloor 0.3*d \rfloor)}$. Then we remove the column corresponding to the additional $\lfloor 0.3*d \rfloor$ environment variables $E$ to generate the final datasets $\Mat{X} \in \mathbb{R}^{n\times d}$. 
We change the variances of the noises of the environment variables $E$ to simulate different environments. 
In the nonlinear setting, after generating the graph, we randomly select $\lfloor 0.5*d \rfloor$ nodes in the graph, and then create $\lfloor 0.5*d \rfloor$ environment nodes and also $\lfloor 0.5*d \rfloor$ edges from each environment node to the selected node. 
% we also add a new variable $E$, and randomly build $0.5*d$ edges from $E$ to the other variables. 
Then given this new graph, we simulate the SEM $\Mat{X}_j = F_j(\Mat{X}_{pa(j)}) + z_j$ for all $j\in\{1\cdots, d+\lfloor 0.5*d \rfloor\}$ in topological order. For the environment variables, we vary the distributions of the noises in different environments. Then for the other nodes, we set the distribution of the noises as $\mathcal{N}(0,1)$. We choose $f_j$ to be Additive Noise Models with two-layer MLPs. 
% Similarly, we change the variance of the noise on the environment variable $E$ to create different environments. 
We will also remove the column corresponding to the variable $E$ to generate the datasets $\Mat{X}\in \mathbb{R}^{n\times d}$.

\subsubsection{Real-world Dataset}
\label{subsub:realworld_dataset}
For the real-world dataset, We sample 10000 images in total, and 2000 images for each environment. We classify MNIST digits from 2 classes, where classes 0 and 1
indicate original digits (0,1,2,3,4) and (5,6,7,8,9). Then for each environment, we select the ratio of class 0 being green, which is shown in the "Ratio" column in Table \ref{tab:exp_setting_mnist}. Then the noise variances(i.e. noise scale) for each environment are provided in the ``Noise Scale" column in Table \ref{tab:exp_setting_mnist}.

\begin{table}[t]
    \centering
    \caption{Experimental Settings for Colored MNIST.}
    \label{tab:exp_setting_mnist}
    \begin{tabular}{ccc}
    \toprule
        Environment & Ratio & Noise Scale \\
    \midrule
        $e_1$ & 0.16 & 10/255 \\
        $e_2$ & 0.32 & 20/255 \\
        $e_3$ & 0.48 & 30/255 \\
        $e_4$ & 0.64 & 40/255 \\
        $e_5$ & 0.80 & 50/255 \\
    \bottomrule
    \end{tabular}
    \vspace{-10pt}
\end{table}

\input{chapters/table_nonlinear}

% We provide the experimental results with the node number as $d=100$ and edge number $s_0=400$. The results in linear setting is provided in Table \ref{tab:linear_100node}, and the results for nonlinear setting are in Table \ref{tab:nonlinear_100node}. 

% \vspace{-8pt}
\subsection{Overall Performances (RQ1)}
% \vspace{-5pt}
We present the overall performances of DICD and the baselines for fair comparison. In the baselines, NOTEARS, DAGGNN, NoCurl, DARING are run on the concatenated datasets from all the environments. CD-NOD is run with the environment-id corresponding to each sample.

% \vspace{-8pt}
\subsubsection{Linear Synthetic Data}
% \vspace{-5pt}
In this experiment, we explore the improvements when introducing different groups by comparing the DAG estimations against the ground truth structure. We simulate \{ER4, SF4\} graphs with $d=\{10, 20, 50, 100\}$ nodes. For each environment, we generate $200$ samples. We evaluate our methods with datasets from 5 environments, and the variances of Gaussian noise for the environment variable $E$ in each environment are \{0.2, 0.4, 0.6, 0.8, 1.0\}. Table \ref{tab:linear_exp} and Table \ref{tab:linear_100node} summarizes the results when the number of nodes equals to $\{10, 20, 50, 100\}$. From these tables, we could have the following key observations: (1) DICD has outperformed all other baselines across various settings. More precisely, DICD achieves significant improvements over the strongest baselines by up to 36\% in SHD (10 nodes, ER graph).
(2) Generally, DICD has the lower FDR and higher TPR, which also coincides with the intuition that DICD could eliminate spurious correlations and reveal the true ones. 
% \wy{(3) CD-NOD performs not well in linear cases, the possible reason is that the }

% \vspace{-8pt}
\subsubsection{Nonlinear Synthetic Data}
% \vspace{-5pt}
We also conducte enormous experiments to demonstrate the effectiveness of DICD in the nonlinear setting. Similarly to the linear setting, we simulate \{ER4, SF4\} graphs with $d=\{10, 20, 50, 100\}$ nodes. For each graph, we generate data from two environments, 1000 samples for each environment. The Gaussian noises for the environment variable $E$ in these two groups are $\{0.2, 0.4\}$. The results with $d=\{10, 20, 50\}$ are provided in Table \ref{tab:nonlinear_exp} and results with $d=100$ are shown in Table \ref{tab:nonlinear_100node}. From these tables, we could find: (1) DICD consistently outperforms other baselines in all eight settings upon the most crucial metric SHD. The improvements on SHD over the best baseline are up to 29\% (50 nodes, ER graph). (2) DICD achieves compatible FDR with NOTEARS but far higher TPR. This shows that in the nonlinear setting, DICD is better at revealing the true causal correlations that might have been missed by NOTEARS. (3) The over-reconstruction problem still exists in other methods, while DICD has the potential to mitigate it, which could be the reason for the performance improvements. (4) CD-NOD performs fairly well in nonlinear cases, which means the multi-environment setting might be more helpful when the correlations between variables are more complicated. However, CD-NOD consumes in average more than 9 hours in the simplest setting (10 nodes), and more than 300 hours in the case of 50 nodes, which is far more expensive than our algorithm. As shown in Table \ref{tab:running_time_comparison}, we only record the running time of CD-NOD in nonlinear settings, since its running time is almost unacceptable in these cases. Then we report the averaged running time for different seeds and different graph types (ER or SF). From the table, we can observe that CD-NOD is very expensive for nonlinear cases, while the running time of DICD almost remain constant when the number of nodes get larger. 

\begin{table}[ht]
    \centering
    % \vspace{-15pt}
    \caption{Running time comparison}
    \label{tab:running_time_comparison}
    % \vspace{-5pt}
    \begin{tabular}{cccc}
    \toprule
       &  10 nodes & 20 nodes & 50 nodes \\
       \midrule
      CD-NOD & $>$9h & $>$64h & $>$300h \\
      DICD & 15min & 15min & 15min \\
      \bottomrule
    \end{tabular}
    % \vspace{-10pt}
\end{table}

\input{chapters/table_100node}
\subsection{Study of Various Factors (RQ2)}
% \vspace{-5pt}
In this section, we discuss various factors that may affect the performances of DICD and other methods. Due to limit of space, we discuss the effect of environmental imbalance in Appendix \ref{effects_of_imbalance}. 

% \vspace{-8pt}
\subsubsection{The Effect of Environment Number}
\label{the_effect_of_the_number_of_environments}
% \vspace{-5pt}

To explore how the number of environments affects the performances, we conduct experiments on both linear and nonlinear settings. The total number of the examples is set to 1000. We choose $d=10$ and $s_0=40$ with ER graph for this case study. If the number of the environments is $N$, then the number of samples for each environment is $\lfloor 1000/N\rfloor$. The noises for different environments are sliced from the head of the array \{0.2, 0.4, 0.6, 0.8, 1.0, 0.1, 0.3, 0.5, 0.7, 0.9\}. For example, the noises are \{0.2, 0.4\} when the number of environments is 2. In Figure \ref{fig:effect_of_number_of_environment}, we can observe: 
\begin{figure}[t]
    \centering
    % \vspace{-8pt}
    \subfigure[Linear Setting]{\label{fig:case_1_shd}\includegraphics[width=0.490\linewidth]{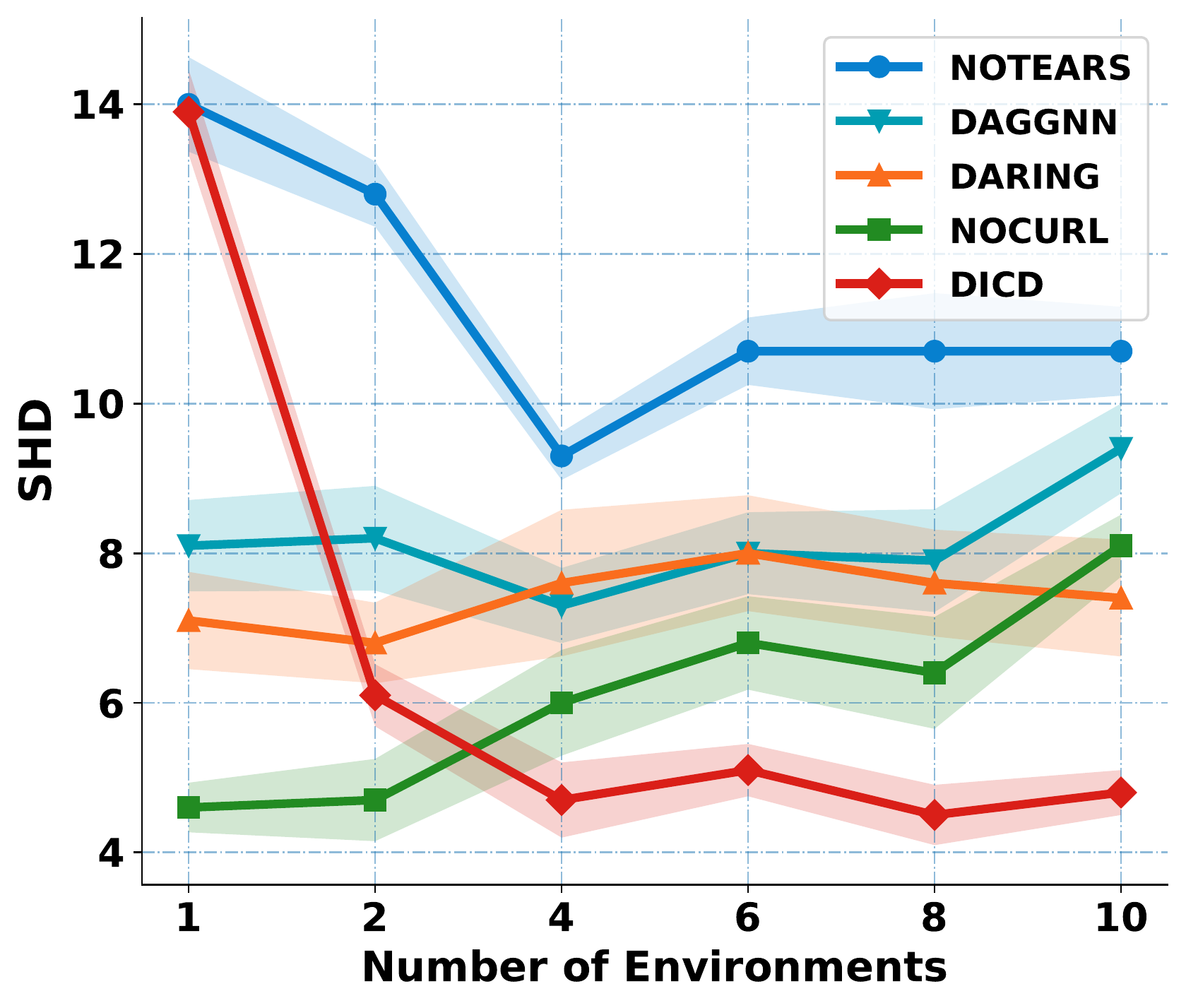}}
\subfigure[Nonlinear Setting]{\label{fig:nonlinear_case_1_shd}\includegraphics[width=0.490\linewidth]{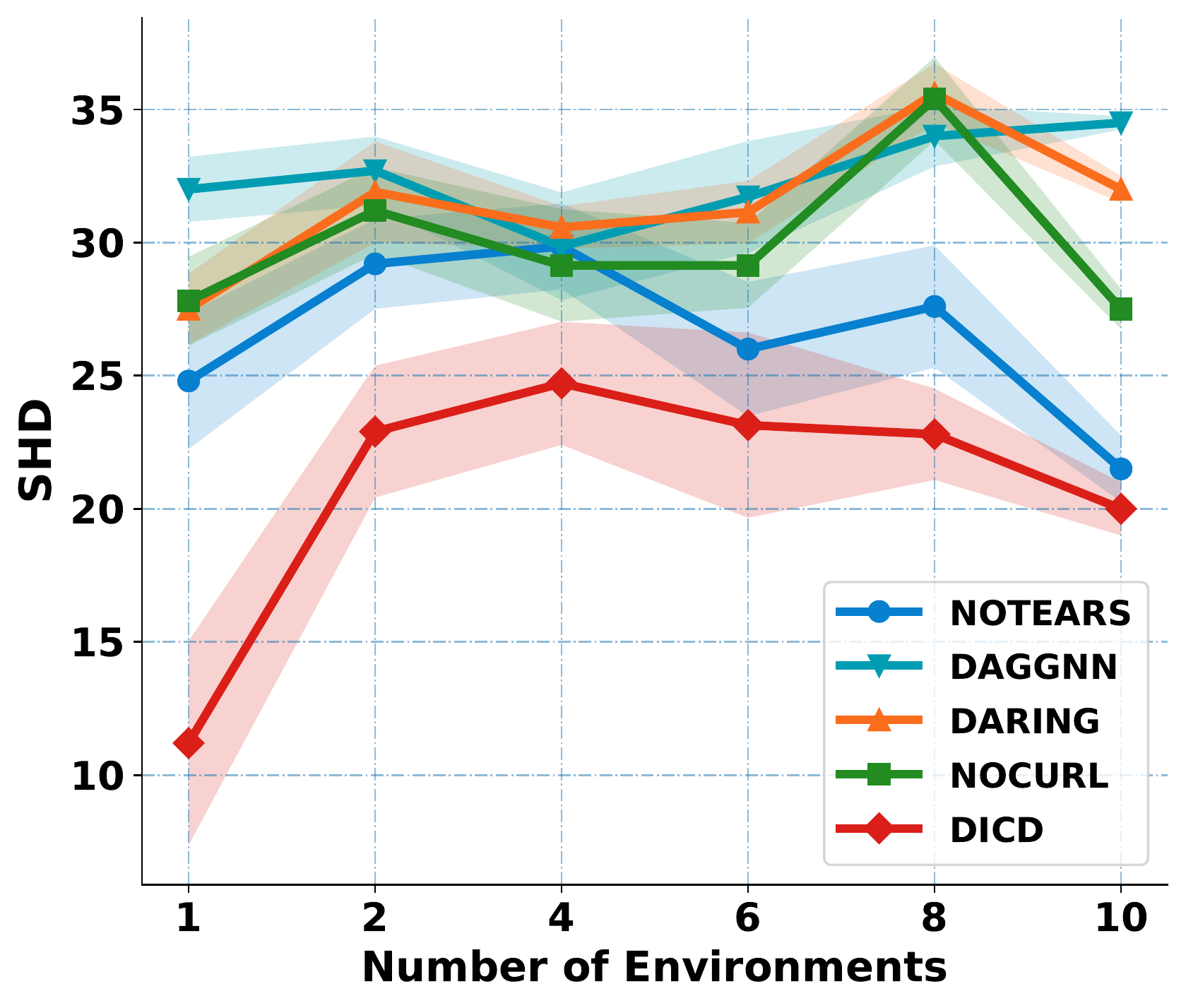}}
\vspace{-8pt}
    \caption{SHD \wrt the number of environments.}
    \label{fig:effect_of_number_of_environment}
    \vspace{-10pt}
\end{figure}
(1) When the number of environments is greater than 4 in the linear setup, more environments will not yield to better results, which means for simple structure equation models, four environments may have already ruled out the possibility of incorrect graphs. 
% In the linear setting, when the number of environments is larger than four, more environments could not lead to stronger performances, which means when the structure equation model is simple, four environments may have already ruled out the potential wrong graphs. 
(2) In the nonlinear setting, it is amazing to find that DICD could achieve better performance even there is only one environment. The reason could be that when there are more parameters in the function, the explicit constraint on the optimality of the parameters could contribute much to learning the best graph, which is more likely to be the ground truth graph. 
(3) More environments in nonlinear settings could degrade the performance of all approaches; the reason for this could be that heterogeneous noises can be particularly unfriendly when the relationships between variables become quite intricate.
% In nonlinear settings, more environments could degenerate the performances for all methods, the reason could be that heterogeneous noises might be very unfriendly when the correlations between variables get very complicated. 

% \vspace{-5pt}
\subsubsection{The Effect of the Imbalance between Different Environments}
\label{effects_of_imbalance}
Since the imbalance of data is the major problem in machine learning, we aim to explore how the imbalance of data size between different environments would affect the performances of our method and other baselines. The total number of samples in this dataset is 1000 for linear settings and 2000 for nonlinear settings. We choose ER graphs with $d=20$ and $s_0=4d=80$ for this case study. The noises for the variable $E$ in two environments are set to be \{0.2, 0.4\}. Then the ratio in Figure \ref{fig:effect_of_imbalance_of_environment} means the percentage of the samples from the first environment. From this figure, we could have the following observations: (1) In most of the settings, DICD outperforms other methods consistently, except when the data is higher imbalanced in linear setting, where there is almost only one environment. (2) In both linear and nonlinear settings, the balanced situation is the best for DICD, which means we have enough information from every environment. (3) Even there is only one group, DICD could make significant improvements against NOTEARS, which coincides with the discovery in Section \ref{the_effect_of_the_number_of_environments}. 
\begin{figure}[t]
    \centering
    % \vspace{-25pt}
    \subfigure[Linear Setting]{\label{fig:case_2_shd}\includegraphics[width=0.490\linewidth]{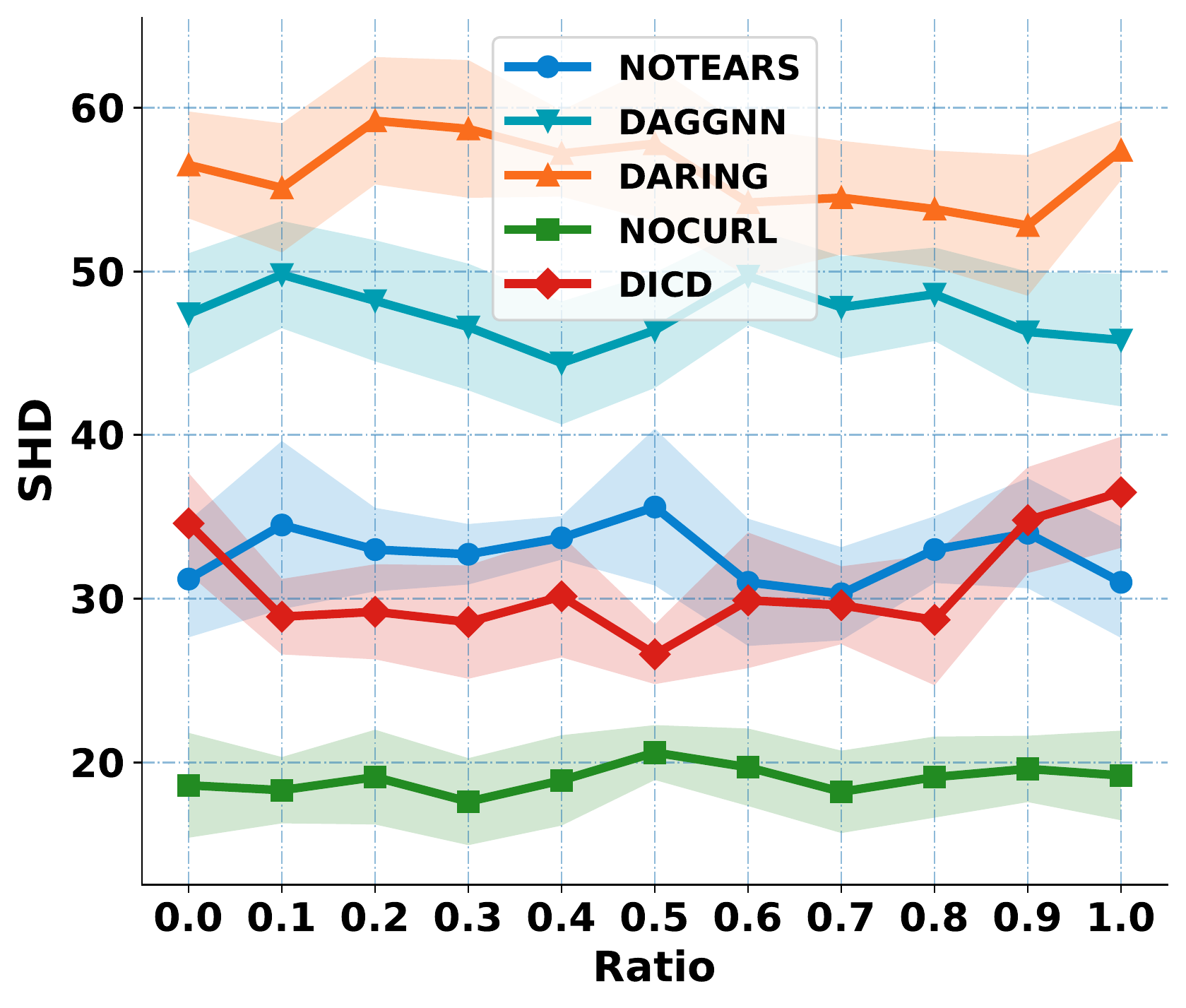}}
\subfigure[Nonlinear Setting]{\label{fig:nonlinear_case_2_shd}\includegraphics[width=0.490\linewidth]{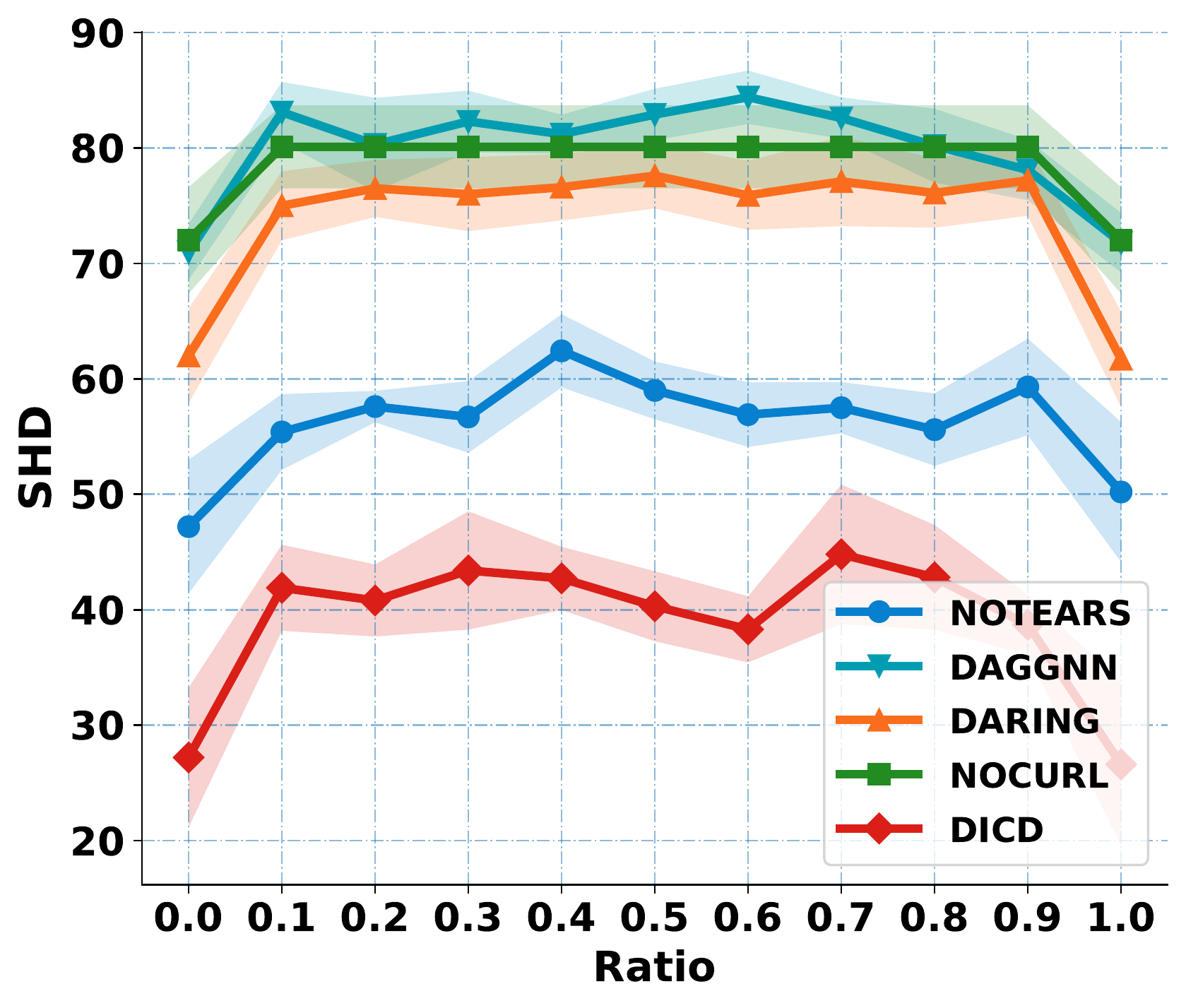}}
% \vspace{-8pt}
    \caption{SHD for different percentage of imbalance between two environments in linear and nonlinear settings.}
    \label{fig:effect_of_imbalance_of_environment}
    \vspace{-5pt}
\end{figure}
\begin{figure*}
    % \begin{minipage}{1.0\linewidth}
    \centering     %%% not \center
    \subfigure[Ground Truth]{\label{fig:GroundTruth_CMNIST}\includegraphics[width=0.18\linewidth]{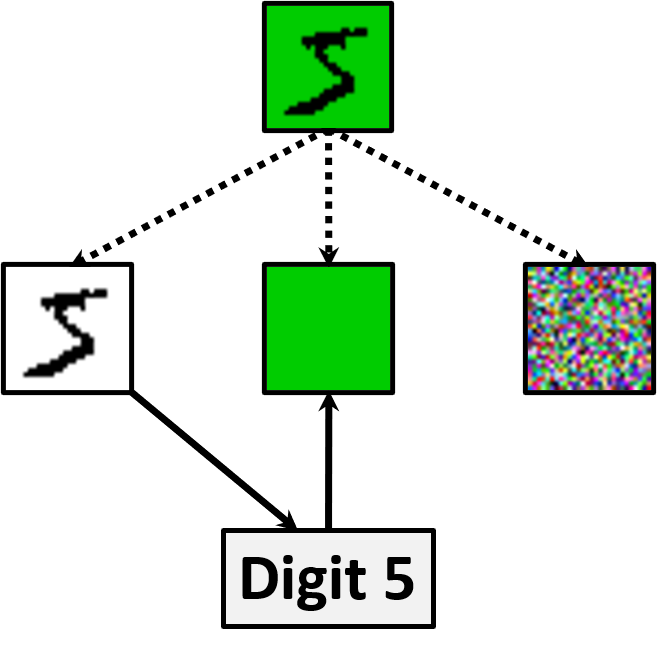}}\hfill
    \subfigure[NOTEARS]{\label{fig:NOTEARS_CMNIST}\includegraphics[width=0.18\linewidth]{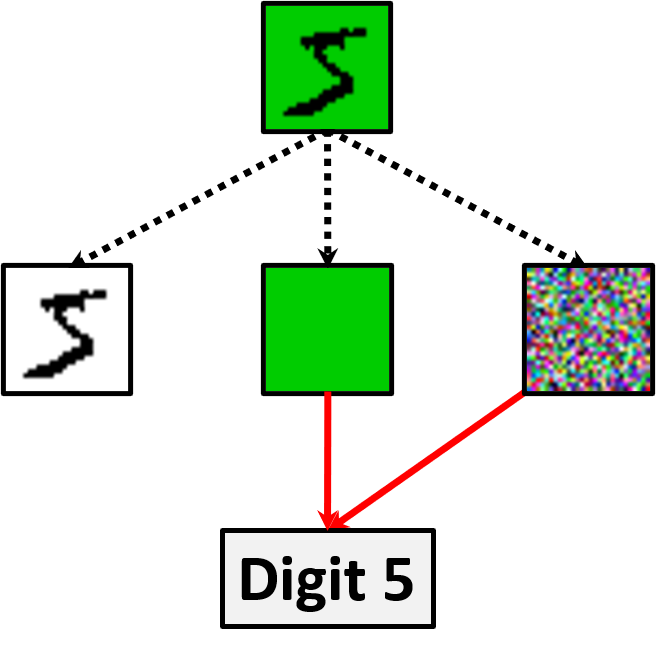}}\hfill
    \subfigure[DAGGNN]{\label{fig:DAGGNN_CMNIST}\includegraphics[width=0.18\linewidth]{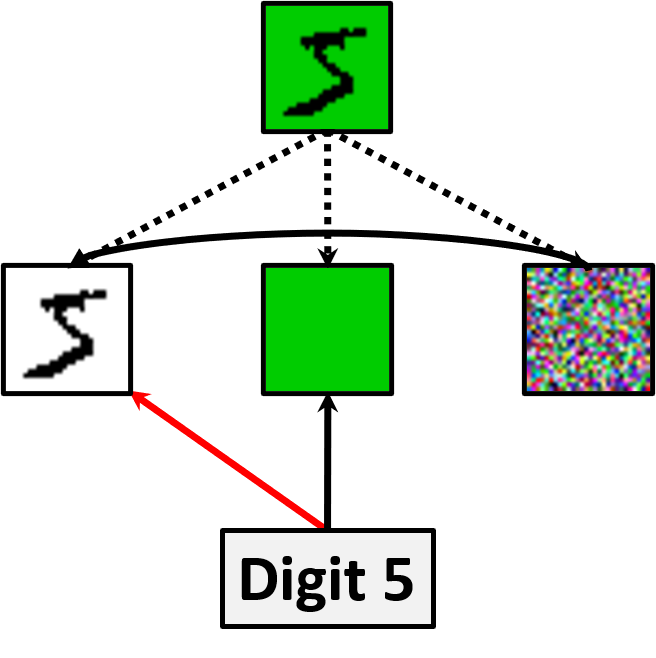}}\hfill
    \subfigure[NOCURL]{\label{fig:NoCurl_CMNIST}\includegraphics[width=0.18\linewidth]{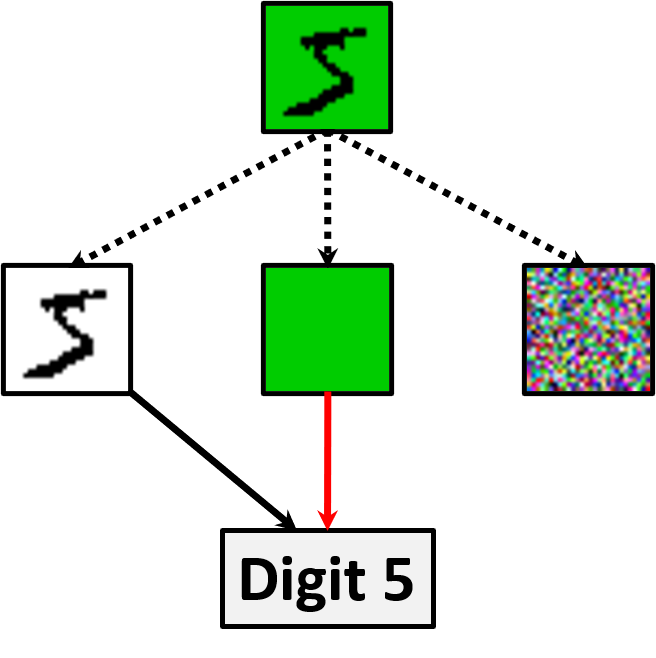}}
    \hfill
    \subfigure[DICD]{\label{fig:DICD_CMNISt}\includegraphics[width=0.18\linewidth]{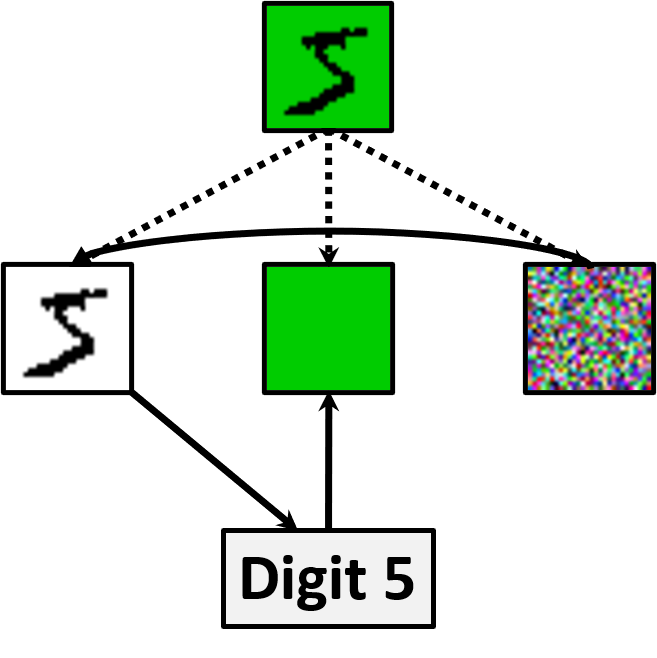}}
    \vspace{-6pt}
    \caption{Revealed graph by different methods on Real World Data}
    \label{fig:real_world_data}
    % \end{minipage}
    \vspace{-10pt}
\end{figure*}

\subsubsection{The Effect of Density of Graph}
\begin{figure}[t]
    % \vspace{-30pt}
    \centering
    \subfigure[Linear Setting]{\label{fig:case_3_shd}\includegraphics[width=0.490\linewidth]{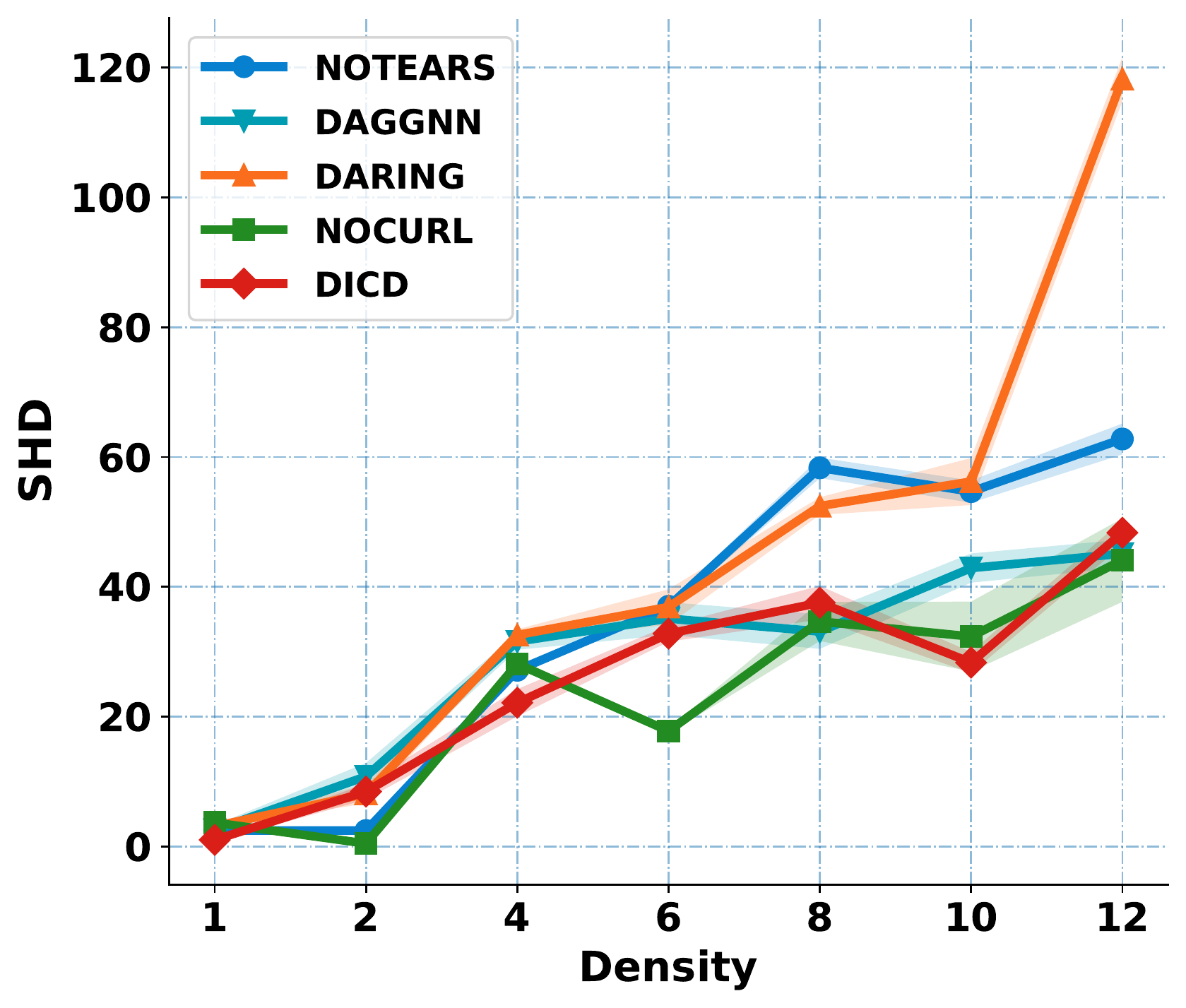}}
\subfigure[Nonlinear Setting]{\label{fig:nonlinear_case_3_shd}\includegraphics[width=0.490\linewidth]{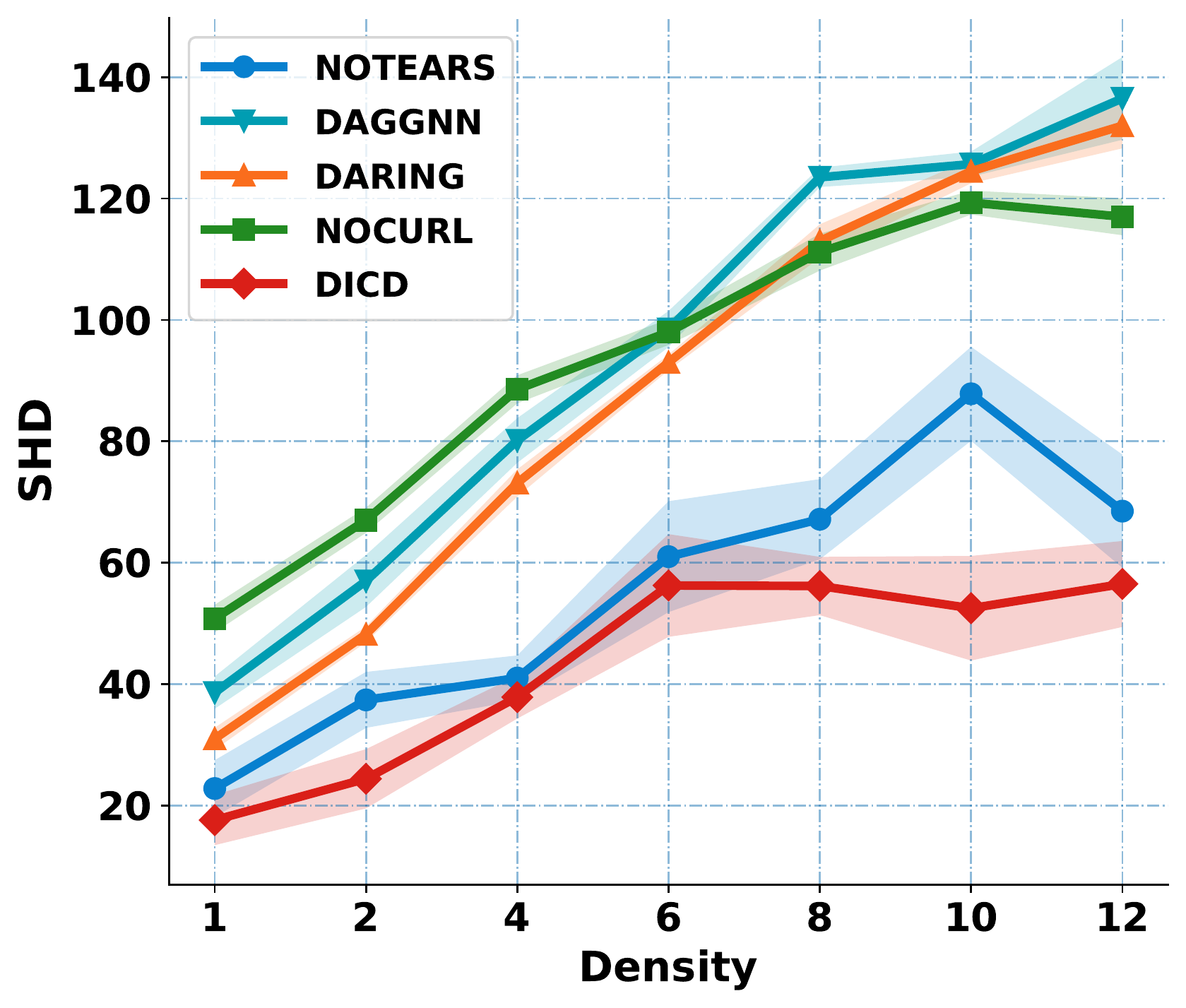}}
    % \vspace{-10pt}
    \caption{SHD for different density conditions.}
    \label{fig:effect_of_the_sparsity_of_the_graph}
    \vspace{-8pt}
\end{figure}
We aim to discover how much the node degree (\ie density) will affect our algorithm compared to the other baselines. We choose SF graph with $d=20$ for this case study. In the linear setting, we generate samples from five environments, 200 samples for each. And for nonlinear settings, we draw samples from two environments, 1000 samples for each. The x-axis represents the mean degree of the nodes in the generated graph. For instance, Node degree = 10 means there are 200 edges in total when generating the SF graph. from Figure \ref{fig:effect_of_the_sparsity_of_the_graph}, we could observe: (1) DICD almost consistently outperforms the other methods, except for few cases. (2) As the node degree (\ie the density) increases, especially in nonlinear settings, the improvements of DICD over baselines get larger, which means DICD could better adapt to the denser settings.

% \vspace{-8pt}
\subsection{Real Data (RQ3)}
% \vspace{-5pt}
Finally, we evaluate our method on the real-world data: ColoredMNIST. Our method, as well as NOTEARS-MLP, could be simply extended to the setting of vector-valued variable. DAG-GNN is already designed to work in this situation. However, DARING do not discuss this situation in the paper. Though it seems that DARING could also be generalized to vector-valued, efforts need to be taken. Thus we omit it as comparison in this section. As stated earlier in Section \ref{subsub:realworld_dataset}, we sample 10000 images from the MNIST dataset and generate the colored version of MNIST. In the pictures, we have kept the digits to be black, while the background has the color red or green. The images are resized to 8*8, and then flattened to yield a 64-dim vector. The label vector and color vector are set to be all zeros or all ones (according to the original label) with dimension as 64. Then during searching (or training) we only consider the possible correlations among the white backgrounds, colors, noises, and digits. Thus we have four variables in the designed task. As shown in Figure \ref{fig:real_world_data}. We could find NOTEARS, DAGGNN and NoCurl all make mistakes, while DICD gives reasonable predictions.

%% file: chapters/table_linear.tex
\begin{table*}[t]
    \centering
    \vspace{5pt}
    \caption{Linear Setting, for ER and SF graphs of 10, 20, 50 nodes}
    % \vspace{-7pt}
    \resizebox{\textwidth}{!}{
    \label{tab:linear_exp}
    \begin{tabular}{c|ccc|ccc|ccc}
\toprule
 & \multicolumn{3}{c|}{10 nodes} &   \multicolumn{3}{c|}{20 nodes}  &  \multicolumn{3}{c}{50 nodes}   \\
\textbf{ER4} & FDR & TPR & SHD & FDR & TPR & SHD & FDR & TPR & SHD \\
 \midrule
CD-NOD & 0.48\scriptsize{$\pm$0.06} &  0.17\scriptsize{$\pm$0.02} & 33.5\scriptsize{$\pm$0.9} & 0.48\scriptsize{$\pm$0.19} &  0.11\scriptsize{$\pm$0.02} & 75.7\scriptsize{$\pm$5.0} & 0.56\scriptsize{$\pm$0.08} &  0.15\scriptsize{$\pm$0.06} & 195.0\scriptsize{$\pm$8.0} \\
NOTEARS  & 0.07\scriptsize{$\pm$0.01} & 0.78\scriptsize{$\pm$0.02} & 10.1\scriptsize{$\pm$0.8} &  
0.20\scriptsize{$\pm$0.03} & 0.71\scriptsize{$\pm$0.08} & 35.9\scriptsize{$\pm$7.4} &  
0.27\scriptsize{$\pm$0.04} & 0.78\scriptsize{$\pm$0.03} & 96.7\scriptsize{$\pm$13.7} \\
DAGGNN  & 0.08\scriptsize{$\pm$0.01} & 0.86\scriptsize{$\pm$0.02} & 7.6\scriptsize{$\pm$0.9} &  
0.34\scriptsize{$\pm$0.04} & 0.79\scriptsize{$\pm$0.05} & 48.3\scriptsize{$\pm$7.1} &  
0.36\scriptsize{$\pm$0.03} & 0.86\scriptsize{$\pm$0.03} & 122.3\scriptsize{$\pm$14.6} \\ 
NOCURL & 0.14\scriptsize{$\pm$0.01} & 0.81\scriptsize{$\pm$0.01} & 9.8\scriptsize{$\pm$0.4} & 0.19\scriptsize{$\pm$0.02} & \textbf{0.92\scriptsize{$\pm$0.01}} & 22.3\scriptsize{$\pm$2.4} & 0.31\scriptsize{$\pm$0.02} & \textbf{0.94\scriptsize{$\pm$0.01}} & 89.5\scriptsize{$\pm$9.5} \\
DARING  & 0.08\scriptsize{$\pm$0.01} & 0.81\scriptsize{$\pm$0.03} & 9.3\scriptsize{$\pm$1.3} &  
0.38\scriptsize{$\pm$0.07} & 0.58\scriptsize{$\pm$0.07} & 60.9\scriptsize{$\pm$9.9} &  
0.48\scriptsize{$\pm$0.01} & 0.60\scriptsize{$\pm$0.06} & 187.6\scriptsize{$\pm$4.9} \\  
DICD  & \textbf{0.03\scriptsize{$\pm$0.01}} & \textbf{0.88\scriptsize{$\pm$0.02}} & \textbf{4.9\scriptsize{$\pm$1.0}} &
\textbf{0.16\scriptsize{$\pm$0.03}} &0.89\scriptsize{$\pm$0.04} & \textbf{19.7\scriptsize{$\pm$5.8}} &  
\textbf{0.26\scriptsize{$\pm$0.02}} & 0.89\scriptsize{$\pm$0.06} & \textbf{82.0\scriptsize{$\pm$5.3}} \\
\bottomrule
\toprule
\textbf{SF4} & FDR & TPR & SHD & FDR & TPR & SHD & FDR & TPR & SHD \\
\midrule
CD-NOD & 0.36\scriptsize{$\pm$0.10} &  0.19\scriptsize{$\pm$0.02} & 25.0\scriptsize{$\pm$1.4} & 0.34\scriptsize{$\pm$0.05} &  0.18\scriptsize{$\pm$0.01} & 59.3\scriptsize{$\pm$0.9} & 0.38\scriptsize{$\pm$0.04} &  0.15\scriptsize{$\pm$0.01} & 168.3\scriptsize{$\pm$1.7} \\
NOTEARS  & \textbf{0.04\scriptsize{$\pm$0.04}} & 0.81\scriptsize{$\pm$0.02} & 6.1\scriptsize{$\pm$1.1} & 
0.19\scriptsize{$\pm$0.01} & 0.77\scriptsize{$\pm$0.02} & 27.1\scriptsize{$\pm$1.5} &  
\textbf{0.17\scriptsize{$\pm$0.01}} & 0.83\scriptsize{$\pm$0.01} & 60.7\scriptsize{$\pm$2.5} \\
DAGGNN  & 0.07\scriptsize{$\pm$0.03} & 0.98\scriptsize{$\pm$0.02} & 2.9\scriptsize{$\pm$1.5} &   
0.27\scriptsize{$\pm$0.03} & 0.84\scriptsize{$\pm$0.02} & 31.6\scriptsize{$\pm$2.6} &  
0.26\scriptsize{$\pm$0.02} & 0.88\scriptsize{$\pm$0.01} & 80.6\scriptsize{$\pm$8.3} \\ 
% GraN-DAG  &  0.82\scriptsize{$\pm$0.04} & 0.16\scriptsize{$\pm$0.03} & 35.3\scriptsize{$\pm$1.6} & 
% 0.89\scriptsize{$\pm$0.02} & 0.15\scriptsize{$\pm$0.03} & 118.7\scriptsize{$\pm$3.1} &  
% 0.94\scriptsize{$\pm$0.01} & 0.13\scriptsize{$\pm$0.02} & 553.3\scriptsize{$\pm$9.2} \\
NOCURL & 0.06\scriptsize{$\pm$0.02} & 0.86\scriptsize{$\pm$0.02} & 4.8\scriptsize{$\pm$1.1} & 0.25\scriptsize{$\pm$0.01} & \textbf{0.86\scriptsize{$\pm$0.01}} & 28.2\scriptsize{$\pm$1.3} & 0.26\scriptsize{$\pm$0.08} & \textbf{0.93\scriptsize{$\pm$0.05}} & 71.8\scriptsize{$\pm$10.6} \\
DARING  & 0.17\scriptsize{$\pm$0.04} & 0.86\scriptsize{$\pm$0.07} & 9.0\scriptsize{$\pm$2.7} &   
0.26\scriptsize{$\pm$0.02} & 0.80\scriptsize{$\pm$0.01} & 32.6\scriptsize{$\pm$1.8} &  
0.28\scriptsize{$\pm$0.02} & 0.87\scriptsize{$\pm$0.01} & 87.3\scriptsize{$\pm$5.9} \\
DICD  & 0.05\scriptsize{$\pm$0.04} & \textbf{0.98\scriptsize{$\pm$0.02}} & \textbf{2.3\scriptsize{$\pm$1.7}} &   
\textbf{0.16\scriptsize{$\pm$0.05}} & 0.81\scriptsize{$\pm$0.09} & \textbf{22.1\scriptsize{$\pm$8.4}} &  
0.18\scriptsize{$\pm$0.03} & 0.91\scriptsize{$\pm$0.01} & \textbf{53.7\scriptsize{$\pm$8.7}}\\ 
\bottomrule
\end{tabular}}
\vspace{-8pt}
\end{table*}

%% file: chapters/table_nonlinear.tex
\begin{table*}[t]
    \centering
\vspace{5pt}
    \caption{Nonlinear Setting, for ER and SF graphs of 10, 20, 50 nodes}
    % \vspace{-7pt}
    \resizebox{\textwidth}{!}{
    \label{tab:nonlinear_exp}
    \begin{tabular}{c|ccc|ccc|ccc}
\toprule
 & \multicolumn{3}{c|}{10 nodes} &   \multicolumn{3}{c|}{20 nodes}  &  \multicolumn{3}{c}{50 nodes}   \\
\textbf{ER4} & FDR & TPR & SHD & FDR & TPR & SHD & FDR & TPR & SHD \\
 \midrule
CD-NOD & 0.39\scriptsize{$\pm$0.06} &  0.50\scriptsize{$\pm$0.07} & 20.0\scriptsize{$\pm$3.1} & 
0.31\scriptsize{$\pm$0.04} &  0.56\scriptsize{$\pm$0.06} & 56.8\scriptsize{$\pm$4.2} & 
0.35\scriptsize{$\pm$0.07} &  0.82\scriptsize{$\pm$0.05} & 
115.7\scriptsize{$\pm$18.3}\\ 
NOTEARS-MLP & 0.24\scriptsize{$\pm$0.10} & 0.44\scriptsize{$\pm$0.12} & 23.8\scriptsize{$\pm$3.5} & \textbf{0.25\scriptsize{$\pm$0.05}} & 0.35\scriptsize{$\pm$0.09} & 59.0\scriptsize{$\pm$5.0} & 0.30\scriptsize{$\pm$0.08} & 0.86\scriptsize{$\pm$0.06} & 102.9\scriptsize{$\pm$25.4}\\
DAGGNN & 0.50\scriptsize{$\pm$0.06} & 0.21\scriptsize{$\pm$0.04} & 32.2\scriptsize{$\pm$1.5} & 0.61\scriptsize{$\pm$0.07} & 0.24\scriptsize{$\pm$0.05} & 82.9\scriptsize{$\pm$4.5} & 0.62\scriptsize{$\pm$0.06} & 0.15\scriptsize{$\pm$0.03} & 217.0\scriptsize{$\pm$16.0} \\
% GraN-DAG & 0.43\scriptsize{$\pm$0.04} & 0.42\scriptsize{$\pm$0.08} & 25.2\scriptsize{$\pm$3.0} & 0.55\scriptsize{$\pm$0.04} & 0.53\scriptsize{$\pm$0.10} & 75.3\scriptsize{$\pm$4.0} & 0.71\scriptsize{$\pm$0.03} & 0.55\scriptsize{$\pm$0.03} & 340.4\scriptsize{$\pm$30.0} \\
NOCURL & 0.38\scriptsize{$\pm$0.02} & 0.38\scriptsize{$\pm$0.05} & 27.4\scriptsize{$\pm$1.5} & 0.56\scriptsize{$\pm$0.07} & 0.34\scriptsize{$\pm$0.07} & 80.2\scriptsize{$\pm$7.6} & 0.69\scriptsize{$\pm$0.07} & 0.28\scriptsize{$\pm$0.07} & 258.2\scriptsize{$\pm$25.2} \\
DARING & 0.44\scriptsize{$\pm$0.04} & 0.28\scriptsize{$\pm$0.04} & 29.8\scriptsize{$\pm$2.7} & 0.55\scriptsize{$\pm$0.09} & 0.23\scriptsize{$\pm$0.07} & 77.6\scriptsize{$\pm$5.7} & 0.58\scriptsize{$\pm$0.09} & 0.23\scriptsize{$\pm$0.05} & 209.4\scriptsize{$\pm$20.0}  \\
DICD & \textbf{0.20\scriptsize{$\pm$0.08}} & \textbf{0.57\scriptsize{$\pm$0.10}} & \textbf{19.2\scriptsize{$\pm$4.3}} & 0.26\scriptsize{$\pm$0.06} & \textbf{0.69\scriptsize{$\pm$0.07}} & \textbf{40.3\scriptsize{$\pm$6.1}} & \textbf{0.26\scriptsize{$\pm$0.04}} & \textbf{0.88\scriptsize{$\pm$0.03}} & \textbf{84.4\scriptsize{$\pm$11.7}} \\ 
\bottomrule
\toprule
\textbf{SF4} & FDR & TPR & SHD & FDR & TPR & SHD & FDR & TPR & SHD \\
\midrule
CD-NOD & 0.42\scriptsize{$\pm$0.03} &  0.68\scriptsize{$\pm$0.04} & 19.5\scriptsize{$\pm$2.5} & 0.31\scriptsize{$\pm$0.09} & 0.64\scriptsize{$\pm$0.07} & 46.3\scriptsize{$\pm$6.8} & 0.32\scriptsize{$\pm$0.05} & 0.75\scriptsize{$\pm$0.09} & 104.3\scriptsize{$\pm$12.1} \\
NOTEARS-MLP & 0.41\scriptsize{$\pm$0.10} & 0.31\scriptsize{$\pm$0.16} & 25.2\scriptsize{$\pm$4.6} & 0.29\scriptsize{$\pm$0.10} & 0.56\scriptsize{$\pm$0.10} & 51.0\scriptsize{$\pm$7.4} & 0.29\scriptsize{$\pm$0.08} & 0.72\scriptsize{$\pm$0.13} & 113.7\scriptsize{$\pm$17.3} \\
DAGGNN & 0.62\scriptsize{$\pm$0.09} & 0.26\scriptsize{$\pm$0.07} & 31.2\scriptsize{$\pm$3.4} & 0.69\scriptsize{$\pm$0.07} & 0.16\scriptsize{$\pm$0.03} & 80.1\scriptsize{$\pm$7.4} & 0.64\scriptsize{$\pm$0.03} & 0.15\scriptsize{$\pm$0.05} & 210.3\scriptsize{$\pm$8.2} \\
% GraN-DAG & 0.49\scriptsize{$\pm$0.17} & 0.43\scriptsize{$\pm$0.13} & 23.7\scriptsize{$\pm$5.0} & 0.73\scriptsize{$\pm$0.06} & 0.42\scriptsize{$\pm$0.09} & 103.7\scriptsize{$\pm$10.8} & 0.87\scriptsize{$\pm$0.01} & 0.45\scriptsize{$\pm$0.04} & 640.3\scriptsize{$\pm$16.9} \\ 
NOCURL & 0.54\scriptsize{$\pm$0.06} & 0.40\scriptsize{$\pm$0.10} & 26.2\scriptsize{$\pm$2.0} & 0.70\scriptsize{$\pm$0.04} & 0.27\scriptsize{$\pm$0.05} & 87.2\scriptsize{$\pm$3.0} & 0.70\scriptsize{$\pm$0.02} & 0.22\scriptsize{$\pm$0.01} & 240.8\scriptsize{$\pm$9.3} \\
DARING & 0.54\scriptsize{$\pm$0.10} & 0.27\scriptsize{$\pm$0.05} & 29.0\scriptsize{$\pm$2.7} & 0.58\scriptsize{$\pm$0.06} & 0.21\scriptsize{$\pm$0.03} & 73.1\scriptsize{$\pm$4.5} & 0.53\scriptsize{$\pm$0.03} & 0.19\scriptsize{$\pm$0.02} & 193.0\scriptsize{$\pm$5.0} \\
DICD & \textbf{0.34\scriptsize{$\pm$0.06}} & \textbf{0.71\scriptsize{$\pm$0.15}} & \textbf{16.2\scriptsize{$\pm$4.3}} & \textbf{0.27\scriptsize{$\pm$0.06}} & \textbf{0.68\scriptsize{$\pm$0.15}} & \textbf{37.9\scriptsize{$\pm$7.1}} & \textbf{0.29\scriptsize{$\pm$0.04}} & \textbf{0.80\scriptsize{$\pm$0.03}} & \textbf{99.3\scriptsize{$\pm$9.0}}\\
\bottomrule
    \end{tabular}}
    \vspace{-8pt}
\end{table*}

%% file: chapters/table_100node.tex
\begin{table}[t]
    \centering
    \caption{Linear Setting, for ER and SF graphs of 100 nodes}
    \label{tab:linear_100node}
    % \vspace{-3pt}
    \resizebox{\linewidth}{!}{
    \begin{tabular}{c|ccc|ccc}
    \toprule
    & \multicolumn{3}{c|}{\textbf{ER4}} & \multicolumn{3}{c}{\textbf{SF4}} \\
    Method & FDR & TPR & SHD & FDR & TPR & SHD \\
    \midrule
NOTEARS & \textbf{0.17\tiny{$\pm$0.03}} & 0.78\scriptsize{$\pm$0.03} & 147.9\scriptsize{$\pm$15.7} & \textbf{0.12\scriptsize{$\pm$0.04}} & 0.91\scriptsize{$\pm$0.02} & 81.4\scriptsize{$\pm$23.0}\\
DAGGNN & 0.23\scriptsize{$\pm$0.03} & 0.79\scriptsize{$\pm$0.04} & 178.3\scriptsize{$\pm$9.6} & 0.31\scriptsize{$\pm$0.02} & 0.92\scriptsize{$\pm$0.01} & 195.3\scriptsize{$\pm$20.4}\\
NOCURL & 0.25\scriptsize{$\pm$0.01} & \textbf{0.94\scriptsize{$\pm$0.00}} & 136.0\scriptsize{$\pm$5.3} & 0.14\scriptsize{$\pm$0.04} & \textbf{0.98\scriptsize{$\pm$0.01}} & 69.8\scriptsize{$\pm$23.2}\\
DARING & 0.29\scriptsize{$\pm$0.02} & 0.69\scriptsize{$\pm$0.04} & 234.0\scriptsize{$\pm$14.3} & 0.29\scriptsize{$\pm$0.01} & 0.90\scriptsize{$\pm$0.01} & 180.9\scriptsize{$\pm$5.4}\\
DICD & 0.20\scriptsize{$\pm$0.03} & 0.87\scriptsize{$\pm$0.01} & \textbf{133.9\scriptsize{$\pm$16.2}} & 0.12\scriptsize{$\pm$0.03} & 0.97\scriptsize{$\pm$0.01} & \textbf{61.9\scriptsize{$\pm$15.6}}\\
 \bottomrule
\end{tabular}}
\vspace{10pt}
\caption{Nonlinear Setting, for ER and SF graphs of 100 nodes}
% \vspace{-3pt}
    \label{tab:nonlinear_100node}
    \resizebox{\linewidth}{!}{
    \begin{tabular}{c|ccc|ccc}
      \toprule
    & \multicolumn{3}{c|}{\textbf{ER4}} & \multicolumn{3}{c}{\textbf{SF4}} \\
    Method & FDR & TPR & SHD & FDR & TPR & SHD \\
    \midrule
NOTEARS & 0.20\scriptsize{$\pm$0.04} & 0.40\scriptsize{$\pm$0.07} & 270.0\scriptsize{$\pm$24.1} & \textbf{0.17\scriptsize{$\pm$0.06}} & 0.42\scriptsize{$\pm$0.10} & 260.1\scriptsize{$\pm$27.3}\\
DAGGNN & 0.52\scriptsize{$\pm$0.05} & 0.09\scriptsize{$\pm$0.01} & 390.6\scriptsize{$\pm$7.7} & 0.50\scriptsize{$\pm$0.08} & 0.09\scriptsize{$\pm$0.02} & 390.3\scriptsize{$\pm$13.4}\\
DARING & 0.41\scriptsize{$\pm$0.05} & 0.15\scriptsize{$\pm$0.02} & 367.0\scriptsize{$\pm$8.8} & 0.42\scriptsize{$\pm$0.04} & 0.15\scriptsize{$\pm$0.02} & 367.3\scriptsize{$\pm$6.9}\\
NOCURL & 0.62\scriptsize{$\pm$0.06} & 0.23\scriptsize{$\pm$0.02} & 447.2\scriptsize{$\pm$33.0} & 0.60\scriptsize{$\pm$0.03} & 0.18\scriptsize{$\pm$0.02} & 418.6\scriptsize{$\pm$10.7}\\
DICD & \textbf{0.18\scriptsize{$\pm$0.05}} & \textbf{0.54\scriptsize{$\pm$0.10}} & \textbf{226.6\scriptsize{$\pm$23.1}} & 0.18\scriptsize{$\pm$0.03} & \textbf{0.53\scriptsize{$\pm$0.06}} & \textbf{228.0\scriptsize{$\pm$16.3}}\\
 \bottomrule
    \end{tabular}}
\vspace{-6pt}
\end{table}

%% file: chapters/2_related_work.tex
\vspace{-8pt}
\section{Related Work}
% \vspace{-8pt}
Causal Discovery has caught enormous attention recently. We will mainly discuss the differentiable score-based algorithms and the works considering multi-environments.

\vspace{5pt}
\noindent \textbf{Differentiable Score-based algorithms}.
Score-based causal discovery methods aim to find the causal structure by optimizing a carefully defined score function via various modelling methods. Though there are conventional methods \cite{DBLP:journals/ml/HeckermanGC95,tsamardinos2006max,DBLP:journals/datamine/GamezMP11,DBLP:conf/nips/NieMCJ14,DBLP:conf/nips/ScanagattaCCZ15} applying various techniques such as hill-climbing\cite{DBLP:journals/jair/YuanM13} and integer programming~\cite{DBLP:journals/corr/abs-1904-10574}, 
the differentiable methods using gradient descent show stronger power. NOTEARS~\cite{zheng2018dags} reformulate 
the causal discovery problem with acyclicity constraint as a continuous program.
DAG-GNN~\cite{yu2019dag} proposes a variant of the acyclicity constraint and solves the generalized linear SEM in a graph autoencoder structure.  
NOTEARS-MLP~\cite{zheng2020learning} and GRAN-DAG~\cite{LachapelleBDL20} extend the continuous acyclicity regularization into a neural network and achieve better results in the nonlinear settings. 
RL-BIC~\cite{zhu2019causal} introduces RL
% Reinforcement Learning (RL) 
to find the DAG with the best BIC score. 
DARING~\cite{He0SXLJ21} proposes to constrain the independence between the residuals and adopts an adversarial training strategy.
DAG-GAN~\cite{gao2021dag} formulates the problem of DAG structure learning from the perspective of distributional optimization.
\cite{GOLEM} shows that applying soft sparsity and DAG constraints would be enough.
There are also other works that apply some alternatives of the continuous acyclicity regularization.  
DAG-NOFEAR~\cite{NoFears} propose the constraint term that only depends on the absolute value of the adjacency matrix $\Mat{W}$, instead of $\Mat{W} \odot \Mat{W}$ in $h(\Mat{W})$ in Equation \eqref{hofw}, which is correlated to $l1$ penalty and sparsity. 
NOCURL~\cite{yu2020dags} aims at eliminating the DAG constraint entirely, by showing that the set of weighted adjacency matrices of DAGs are equivalent to the set of weighted gradients of graph potential functions.
NODAG~\cite{varando2020learning} proposes to solve an $l_1$-penalized optimization, while
ENCO~\cite{lippe2021efficient} provides convergence guarantees without constraining the score function with respect to acyclicity. 
These methods all have their own strategies to control the acyclicity.

\vspace{5pt}
\noindent \textbf{Previous work considering multi-environments}.
Most work with multi-environment settings are built on constraint-based methods~\cite{spirtes2000causation,spirtes2013causal,CD-NOD,zhang2021testing,DBLP:journals/ijon/Shanmugam01}. Although they achieved lots of improvements, there are some major limitations: 1) may have overly strict domain definition~\cite{CausalInferenceIP}; 2) limited to linear cases only~\cite{MultiDomain,SAEM,RegressionInvariance}; 3) designed to solve much simpler problems than CD such as causal direction identification~\cite{SAEM, FOM}; 4) may involve a large number of independence tests and be very time-consuming~\cite{MultiDomain,CD-NOD}, which is also the general problem of constrain-based methods. 
However, We define the environments following \cite{IRM,MultiDomain,RegressionInvariance}, and our method is capable of solving the general causal discovery problems in both linear and nonlinear cases. There is another interesting and more general task called Federated Causal Discovery~\cite{FCD,FCDBN,LocallyOverlappingVariables,CSOVS,CBCDMultiInterventions}, which aims to solve the problem about decentralized datasets. Though our targets are different, the algorithms are somewhat similar. 

%% file: chapters/7_conclusion.tex
% \vspace{-6pt}
\section{Conclusion and Future Work}
\label{conclusion}
% \vspace-8pt}
% In this paper, we propose Invariant Causal Discovery(ICD) to solve the causal discovery problems. We find that previous methods may have the over-reconstruction problem, i.e., they might learn the noises or the spurious relations into the function parameters. Intuitively, true causal relations will remain consistent across different environments, while the spurious relations could not. Thus given the true causal graphs corresponding to the true causal relations, the parameters learnt in different environments given the structure should remain consistent. In other words, the true causal graph should have the parameters corresponding to this graph optimal across different environments. After some theoretical transformation, this hard constraint could be transformed into the regularization term in both linear and nonlinear settings. With this term, we have proposed our method ICD. The experimental results on both linear and nonlinear settings, as well as the real-world dataset all demonstrate that ICD has improved the causal discovery performances significantly.

Despite the great success in causal structure learning, today's differentiable causal discovery methods are still suffering from non-identifiability issue and over-reconstruction problem from observation data. 
Utilizing the inherent heterogeneity when environment partition is provided in advanced is not yet discussed in the differentiable causal discovery task.
In this paper, we proposed a simple yet effective Differentiable Invariant Causal Discovery (DICD) method to tackle the challenge by incorporating the multi-environment information. 
% We note that previous methods may have the over-reconstruction problem. 
% % To deal with this prevalent deficiency, 
% To deal with this, 
% we argue that true causal correlations will remain consistent across different environments, while the spurious relations could not. 
% Thus 
Our main idea is that, given the true causal graph, the parameters of SEM learned from different environments should remain consistent. 
Theoretical guarantees for the identifiability of proposed DICD are
provided under certain assumptions about the environments.
% This leads to our method DICD with 
% reformulating this hard constraint. 
% rewriting this hard constraint with a well-defined regularization term.
The extensive experimental results
% on linear, nonlinear synthetic data and real-world dataset 
demonstrate the effectiveness of DICD. 

One limitation of DICD is the requirement of prior knowledge of the environment. In the future study, we will explore the end-to-end causal discovery technique with environment inference, \ie directly infer partitions of training data without access to environment label. We believe that the idea of this work, \ie invariant learning inspired causal discovery, provides a potential research direction and will inspire more valuable works for learning identifiable DAG from observation data. 

% The major limitation of DICD is strict requirement of prior knowledge of the environment, which limits the application of our method. In the future study, we will explore the end-to-end causal discovery technique with environment inference, \ie directly infer partitions of training data without access to environment label.
% We believe that the idea of this work, \ie invariant learning inspired causal discovery, provides a potential research direction and will inspire more valuable works for learning identifiable DAG from observation data. 

% \vspace{-5pt}

%% file: 8_bio.tex

% \vspace{-50pt}

\begin{IEEEbiography}[{\includegraphics[width=1in,height=1.25in,clip,keepaspectratio]{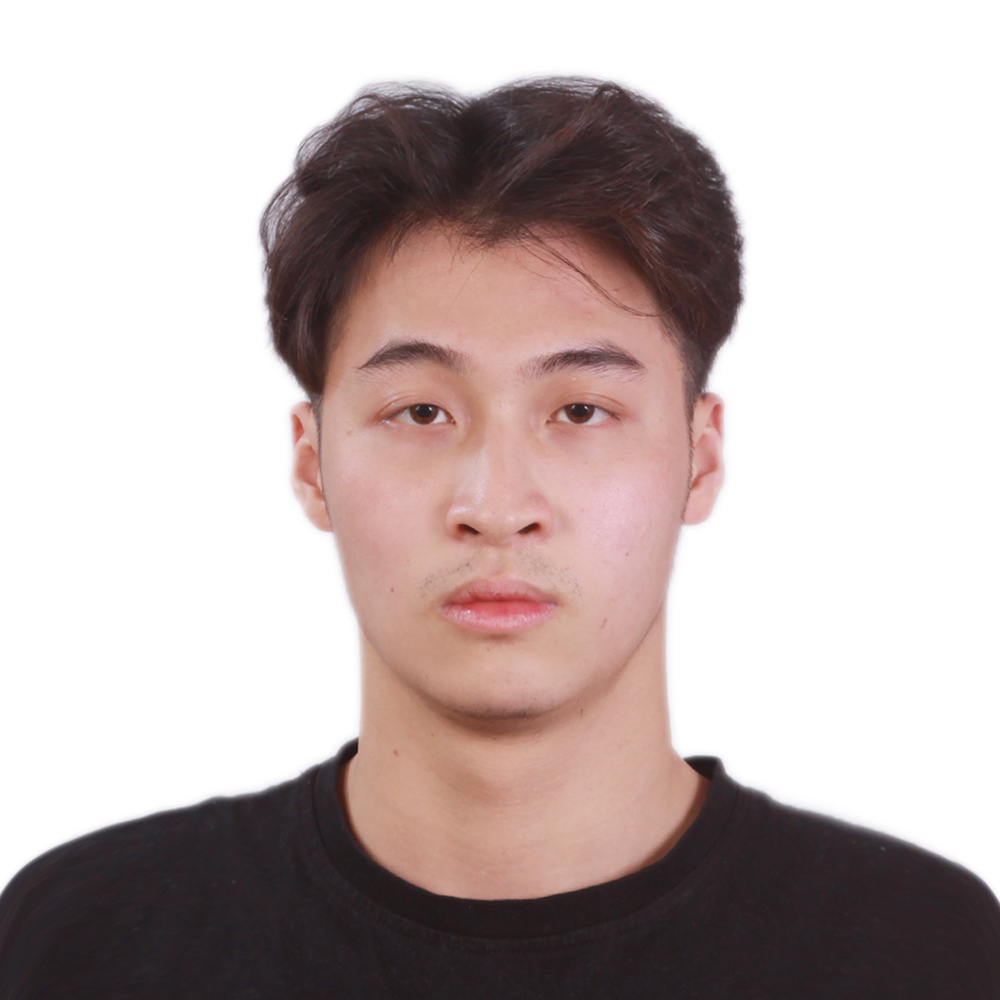}}]{Yu Wang}
    Yu Wang is currently a first-year PhD student at University of California, San Diego. His research interests lie in Machine Learning and Natural Language Process. During his undergraduate stage in University of Science and Technology of Chinan (USTC), he was awarded Baosteel Scholarship, Huawei Scholarship, Excellent Student Scholarship - Gold, etc. 
\end{IEEEbiography}
\vspace{-30pt}

\begin{IEEEbiography}[{\includegraphics[width=1in,height=1.25in,clip,keepaspectratio]{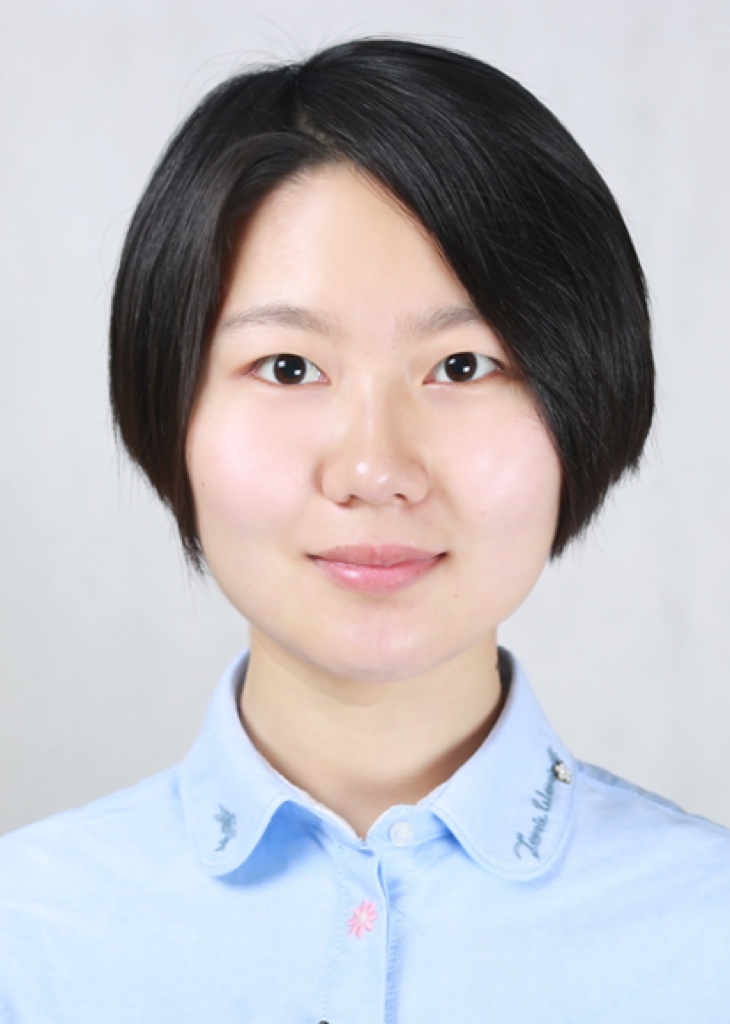}}]{An Zhang}
    is now a research fellow at National University of Singapore.
    She received her Ph.D. degree from the Department of Statistics and Applied Probability, National University of Singapore, in 2021.
    Her areas of interest in research include explainable artificial intelligence, causal representation learning, differentiable causal discovery, and graph neural networks.
    She has publications appeared in several top conferences such as NeurIPS, ICLR, SIGIR.
    % She also participates in the program committees as a number in several top conferences such as NeurIPS, ICLR, and ICML.
\end{IEEEbiography}
\vspace{-30pt}

\begin{IEEEbiography}[{\includegraphics[width=1in,height=1.25in,clip,keepaspectratio]{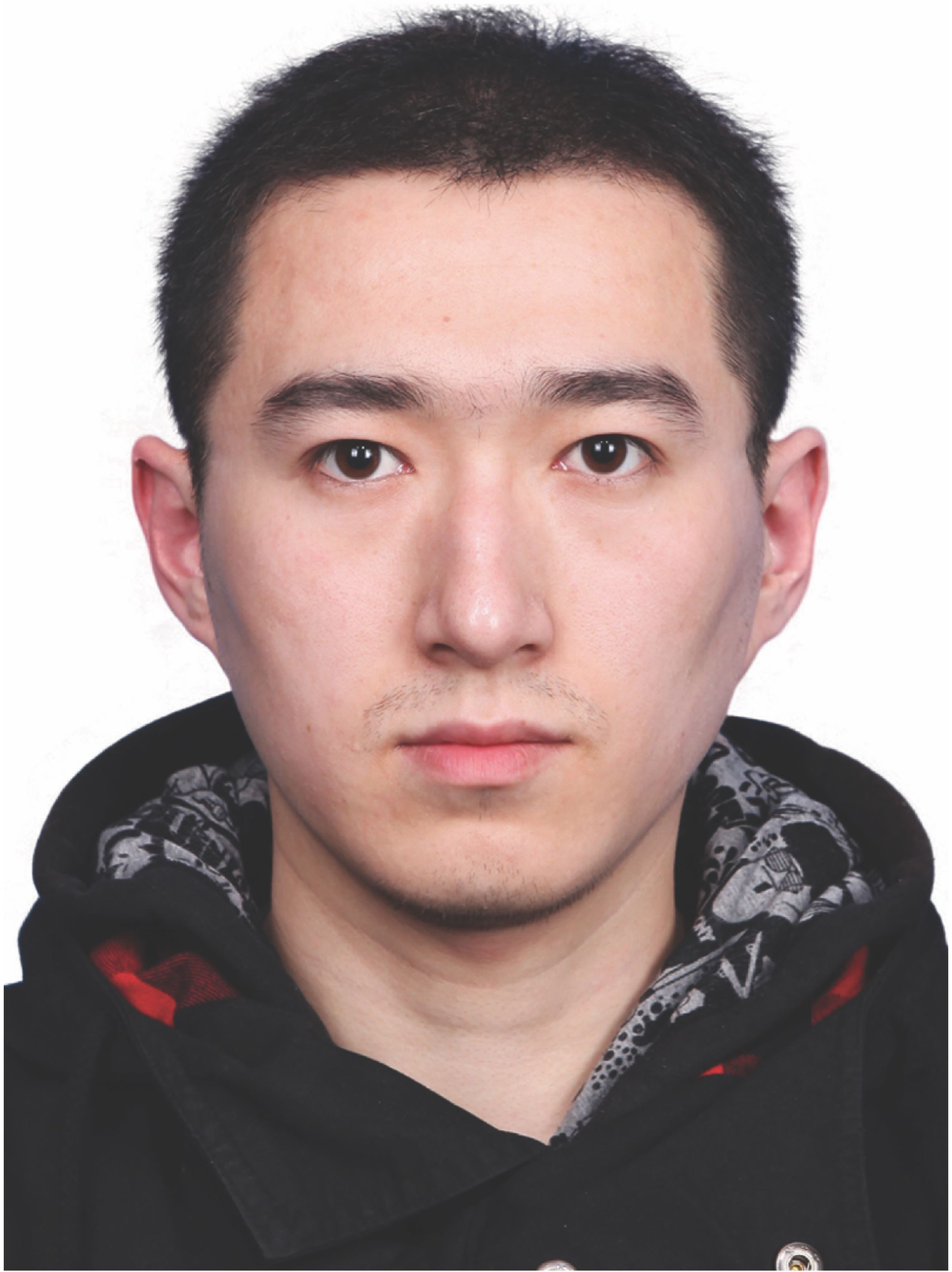}}]{Xiang Wang}
    is now a professor at the University of Science and Technology of China (USTC). He received his Ph.D. degree from National University of Singapore in 2019. His research interests include recommender systems, graph learning, and explainable deep learning techniques. He has published some academic papers on international conferences such as NeurIPS, ICLR, KDD, WWW, SIGIR. He serves as a program committee member for several top conferences such as KDD, SIGIR, WWW, and IJCAI, and invited reviewer for prestigious journals such as TKDE, TOIS, TNNLS.
\end{IEEEbiography}
\vspace{-30pt}

\begin{IEEEbiography}[{\includegraphics[width=1in,height=1.25in,clip,keepaspectratio]{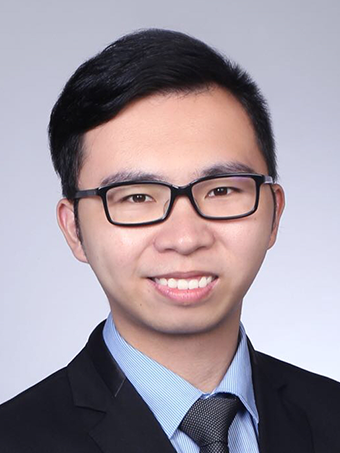}}]{Yancheng Yuan}
is a research assistant Professor of Department of Applied Mathematics, The Hong Kong Polytechnic University. His research focuses on the optimization theory, algorithm design and software development, mathematical foundation of data science, and data-driven applications. He has published papers in prestigious journals and conferences, including SIOPT, JMLR, IJAA, OMS, ICML, WWW.
\end{IEEEbiography}
\vspace{-30pt}

\begin{IEEEbiography}[{\includegraphics[width=1in,height=1.25in,clip,keepaspectratio]{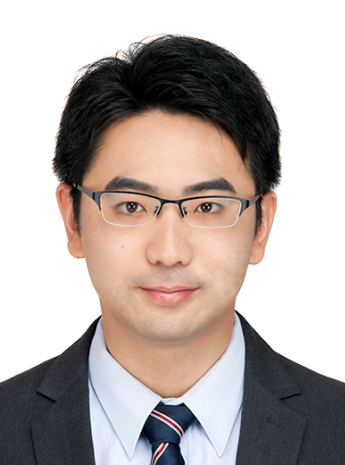}}]{Xiangnan He}
    is a professor at the University of Science and Technology of China (USTC). His research interests span information retrieval, data mining, and multi-media analytics. He has over 90 publications in top conferences such as SIGIR, WWW, and MM, KDD, and journals including TKDE, TOIS, and TMM. His work has received the Best Paper Award Honorable Mention in WWW 2018 and SIGIR 2016. He is in the Editorial Board of the AI Open journal, served as the PC chair of CCIS 2019, the area chair of MM 2019, ECML-PKDD 2020, and the (senior) PC member for top conferences including SIGIR, WWW, KDD, WSDM etc.
\end{IEEEbiography}
\vspace{-30pt}

\begin{IEEEbiography}[{\includegraphics[width=1in,height=1.25in,clip,keepaspectratio]{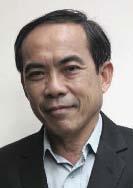}}]{Tat-Seng Chua} is the KITHCT Chair Professor at the School of Computing, National University of Singapore. He was the Acting and Founding Dean of the School during 1998-2000. Dr Chuas main research interest is in multimedia information retrieval and social media analytics. In particular, his research focuses on the extraction, retrieval and question-answering (QA) of text and rich media arising from the Web and multiple social networks. He is the co-Director of NExT, a joint Center between NUS and Tsinghua University to develop technologies for live social media search. Dr Chua is the 2015 winner of the prestigious ACM SIGMM award for Outstanding Technical Contributions to Multimedia Computing, Communications and Applications. He is the Chair of steering committee of ACM International Conference on Multimedia Retrieval (ICMR) and Multimedia Modeling (MMM) conference series. Dr Chua is also the General Co-Chair of ACM Multimedia 2005, ACM ICMR 2005, ACM SIGIR 2008, and ACM Web Science 2015. He serves in the editorial boards of four international journals. Dr. Chua is the co-Founder of two technology startup companies in Singapore. He holds a PhD from the University of Leeds, UK.
\end{IEEEbiography}
    
    
    %\vfill
    

%% file: chapters/8_appendix.tex
% \maketitle
\begin{table*}[ht]
    % \centering
    \caption{The specific confounder case in which NOTEARS\cite{zheng2018dags} might make mistakes. Detailed data generation processes are as follows. Environment 1: $A = z_A(\sim \mathcal{N}(0,4)), B = A/2 + z_B(\sim \mathcal{N}(0,1)), C = A+B/2 + z_C(\sim \mathcal{N}(0,4))$. Environment 2: $A = z_A(\sim \mathcal{N}(0,4)), B = A/2 + z_B(\sim \mathcal{N}(0,4)), C = A+B/2 + z_C(\sim \mathcal{N}(0,1))$. Note that only the variance of $z_B$ is different in two environments.}
    \label{tab:another_toy_example}
    \resizebox{\textwidth}{!}{%
    \begin{tabular}{c|rrrrrr}
    \toprule
      & \multicolumn{1}{c}{\begin{tikzpicture}[scale=1.0, every node/.style={scale=1.0}]
    \node[draw=blue,circle] (y) at (210:0.6) {\color{blue}\textsf{B}};
    \node[draw=blue,circle] (x) at (90:0.6) {\color{blue}\textsf{A}};
    \node[draw=blue,circle] (z) at (-30:0.6) {\color{blue}\textsf{C}};
    \draw[-latex,draw=blue] (x) -- (y);
    \draw[-latex,draw=blue] (x) -- (z);
    \draw[-latex,draw=blue] (y) -- (z);
\end{tikzpicture}}
 &
\multicolumn{1}{c}{\begin{tikzpicture}[scale=1.0, every node/.style={scale=1.0}]
    \node[draw,circle] (y) at (210:0.6) {\textsf{B}};
    \node[draw,circle] (x) at (90:0.6) {\textsf{A}};
    \node[draw,circle] (z) at (-30:0.6) {\textsf{C}};
    \draw[-latex] (x) -- (y);
    \draw[-latex] (x) -- (z);
    \draw[-latex, draw=red] (z) -- (y);
\end{tikzpicture}}  &
\multicolumn{1}{c}{\begin{tikzpicture}[scale=1.0, every node/.style={scale=1.0}]
    \node[draw,circle] (y) at (210:0.6) {\textsf{B}};
    \node[draw,circle] (x) at (90:0.6) {\textsf{A}};
    \node[draw,circle] (z) at (-30:0.6) {\textsf{C}};
    \draw[-latex] (x) -- (y);
    \draw[-latex, draw=red] (z) -- (x);
    \draw[-latex, draw=red] (z) -- (y);
\end{tikzpicture}}  &
\multicolumn{1}{c}{\begin{tikzpicture}[scale=1.0, every node/.style={scale=1.0}]
    \node[draw,circle] (y) at (210:0.6) {\textsf{B}};
    \node[draw,circle] (x) at (90:0.6) {\textsf{A}};
    \node[draw,circle] (z) at (-30:0.6) {\textsf{C}};
    \draw[-latex, draw=red] (y) -- (x);
    \draw[-latex, draw=red] (z) -- (x);
    \draw[-latex] (y) -- (z);
\end{tikzpicture}}  &
\multicolumn{1}{c}{\begin{tikzpicture}[scale=1.0, every node/.style={scale=1.0}]
    \node[draw,circle] (y) at (210:0.6) {\textsf{B}};
    \node[draw,circle] (x) at (90:0.6) {\textsf{A}};
    \node[draw,circle] (z) at (-30:0.6) {\textsf{C}};
    \draw[-latex, draw=red] (y) -- (x);
    \draw[-latex, draw=red] (z) -- (x);
    \draw[-latex, draw=red] (z) -- (y);
\end{tikzpicture}} &
\multicolumn{1}{c}{\begin{tikzpicture}[scale=1.0, every node/.style={scale=1.0}]
    \node[draw,circle] (y) at (210:0.6) {\textsf{B}};
    \node[draw,circle] (x) at (90:0.6) {\textsf{A}};
    \node[draw,circle] (z) at (-30:0.6) {\textsf{C}};
    \draw[-latex, draw=red] (y) -- (x);
    \draw[-latex] (x) -- (z);
    \draw[-latex] (y) -- (z);
\end{tikzpicture}} \\
\midrule
$e_1$ & 6.00(1.00, 0.50, 0.50) & 6.05(0.00, 1.25, 0.40) & 8.97(0.00, 0.67, 0.40) & 5.67(0.00, 0.67, 1.50) & 8.97(0.00, 0.67, 0.40) & 5.00(1.00, 1.00, 0.50)  \\ 
$e_2$ & 9.00(1.00, 0.50, 0.50) & 8.00(-0.75, 1.25, 1.00) & 11.22(-0.75, 0.61, 1.00) & 9.96(-0.29, 0.76, 0.90) & 11.56(-0.29, 0.76, 0.55) & 9.20(0.40, 1.00, 0.50) \\
\bottomrule
\end{tabular}}
\end{table*}

\section*{Basic Structure Example}

In this section, we provide another example shown in Table \ref{tab:another_toy_example} to demonstrate our motivation introduced in Section \ref{introduction}. In this table, the blue graphs denote the ground truth, while graphs with red edges are the wrong edges in the wrong structures. Below the structures, the first number in front of the parenthesis refers to the least reconstruction loss($||\Mat{X}\Mat{W}-\Mat{X}||_2$) given the structure in the same column of Table \ref{tab:another_toy_example}. Given the structure, we can use the structure to constrain which of the elements in $\Mat{W}$ should be non-zero elements and others are not. Then we can learn the optimal parameters for these non-zero elements. As shown in the Table, the numbers in the parenthesis refer to the optimal parameters for the edges between A-B, B-C, A-C, respectively. From this table, we can find that with wrong causal structures which have the potential to obtain lower reconstruction loss (\ie over-reconstruction), the parameters learned given these structures are not consistent across different environments. Thus with the condition given by DICD, we can rule out wrong graphs. 

\section*{Notations}
\label{notations}
We provide the meanings of all the notations used in this paper in Table \ref{tab:notations}. 
\begin{table}[ht]
    \centering
    \caption{Notations and the corresponding Meanings}
    \label{tab:notations}
    \resizebox{0.48\textwidth}{!}{
    \begin{tabular}{cc}
    \toprule
    Notations & Meanings \\
    \midrule
        n & number of the data samples \\
        d & number of the variables (or nodes) \\
        $\Mat{X}$ & $\mathbb{R}^{n\times d}$, observational data \\
        $\Mat{x}_j$ & $\mathbb{R}^{n}$, observational data corresponding to the $j$-th variable \\
        % V & node set \\
        % E & edge set \\
        % G & ground truth graph \\
        $X_j$ & $j$-th node variable \\
        $F$ & ground truth structure function \\
        $F_j$ & ground truth structure function corresponding to $X_j$ \\
        $f$ & estimated structure function \\
        $f_j$ & estimated structure function corresponding to $X_j$ \\
         Pa($X_j$) & parent set of $X_j$ \\
         Pa(j) & the index set of the parents of $X_j$ \\
        $z_j$ & additive noise of $X_j$ \\
        $\sigma_j^2$ & variance of $z_j$ \\
        $\mathcal{G}(\cdot)$ & corresponding graph of $\cdot$ \\
        $\Mat{W}$ & $\mathbb{R}^{d\times d}$, the coefficients for all edges of $\mathcal{G}(f)$ in linear situations\\
        $m_l$ & the number of hidden units in the $l$-th layer in MLP \\
        $\Mat{S}$ & $\{0, 1\}^{d\times d}$ invariant structure model \\
        % $u$ & $\mathbb{R}^d$, a sample in $\Mat{X}$ \\
        \bottomrule
    \end{tabular}}
    \label{tab:my_label}
\end{table}

\section*{Proof of Theorem \ref{theorem:identifiability}}
\label{proof_of_identifiability}
% \begin{proof}
\input{chapters/identifiability_proof}
% \end{proof}

%% file: chapters/identifiability_proof.tex
To begin with, we provide the notations used in this proof. We denote the set of all variables as $\{X_1, \cdots, X_d\}$, and the vector variable $\Trans{[X_1, \cdots, X_d]}$ is denoted as $\Mat{X}$ in this section. The set of parent nodes of $X_i$ is denoted as $Pa(X_i)$, while the index set of $Pa(X_i)$ is denoted as $Pa(i)$. 
% We denote $\Mat{X}_{Pa(i)} \in \mathbb{R}^{|Pa(i)|} $ as the vector variable corresponding to the variables in $Pa(X_i)$. 

First of all, we prove Lemma \ref{lemma: optimal_parameter_is_true_parameter} to show that our DICD could yield the true coefficient in linear systems under the true causal structure.

\begin{proof}
\label{proof_of_optimalparameter_is_trueparameter}
We consider the linear regression coefficient $\hat{\Mat{W}}^e$ for every environment $e$. The regression process in Equation (\ref{regression_result_for_e}) is equivalent to performing the regression for every node $X_m, m \in \{1,\cdots, d\}$ from its parents $Pa(X_m)$. 

We incorporate the environment information into the expectation operator $\mathbb{E}_e$. 
The coefficients of the node $X_m$ is denoted as:
% \begin{equation}\label{regression_result_for_w}
%     \hat{\Mat{w}} = \min_{\overline{\Mat{w}}} \mathbb{E} ||X_m - \overline{\Mat{w}}^T \Mat{X}_{\Mat{S}_0m} || = \mathbb{E}_e[u (\Mat{X}_{\Mat{S}_0m}(\Mat{X}_{\Mat{S}_0m})^T)^{-1} \Mat{X}_{\Mat{S}_0m}]
% \end{equation}
\begin{equation}\label{regression_result_for_w}
    [\hat{\Mat{W}}^e]_m = arg\min_{\overline{\Mat{w}}} \mathbb{E}_e ||X_m - \Trans{(\overline{\Mat{w}} \circ [\Mat{S}_0]_m)} \Mat{X} ||,
\end{equation}
where $[\hat{\Mat{W}}^e]_m$ represents the $m$-th column of $[\hat{\Mat{W}}^e]$ and $[\Mat{S}_0]_m$ is the $m$-th column of $[\Mat{S}_0]$. 
Then we can further express $[\hat{\Mat{W}}^e]_m$ in Equation (\ref{regression_result_for_w}) as:
\begin{equation*}
     [\hat{\Mat{W}}^e]_m = \mathbb{E}_e[X_m \Big( ([\Mat{S}_0]_m\circ  \Mat{X}) \Trans{[(\Mat{S}_0]_m\circ \Mat{X})} \Big)^{-1} [\Mat{S}_0]_m\circ \Mat{X}].
\end{equation*}
For simplicity, we denote $[\Mat{S}_0]_m\circ \Mat{X}$ as $\Mat{X}_{\Mat{S}_0m}$. 

We extract the $m$-th column of $\Mat{W}_0$ as $[\Mat{W}_0]_m$, and the true generation process for $X_m$ could be expressed as:
\begin{equation}\label{generation_process_for_X_m_original}
    X_m = \Trans{[\Mat{W}_0]}_m \Mat{X} + z_m,
\end{equation}
where $z_m \sim \mathcal{N}(0, (\sigma_m^e)^2)$. 
Then as $[\Mat{W}_0]$ is the ground truth coefficients, $[\Mat{W}_0]$ should be consistent with the graph $G_0$, as well as the structure $\Mat{S}_0$. Thus we have: $\Mat{W}_0 = \Mat{S}_0 \circ \Mat{W}_0$, which indicates $\Trans{[\Mat{W}_0]}_m \Mat{X} = \Trans{ [\Mat{S}_0]_m \circ [\Mat{W}_0]}_m \Mat{X} = \Trans{[\Mat{W}_0]}_m \Mat{X}_{\Mat{S}_0m} $. Thus the generation process in Equation (\ref{generation_process_for_X_m_original}) could be rewritten as:
\begin{equation}\label{generation_process_for_X_m}
     X_m = \Trans{[\Mat{W}_0]}_m \Mat{X}_{\Mat{S}_0m} + z_m.
\end{equation}
Then we can substitute Equation (\ref{generation_process_for_X_m}) into Equation (\ref{regression_result_for_w}) to obtain:
\begin{equation*}
    [\hat{\Mat{W}}^e]_m = \mathbb{E}_e[ (\Trans{[\Mat{W}_0]_m} \Mat{X}_{\Mat{S}_0m} + z_m) (\Mat{X}_{\Mat{S}_0m}\Trans{(\Mat{X}_{\Mat{S}_0m})})^{-1} \Mat{X}_{\Mat{S}_0m}].
\end{equation*}
Since $z_m$ is an independent additive noise, we can remove $z_m$ in the above equation and yield:
\begin{align}
    [\hat{\Mat{W}}^e]_m &= \mathbb{E}_e[ (\Trans{[\Mat{w}_0]_m} \Mat{X}_{\Mat{S}_0m}) (\Mat{X}_{\Mat{S}_0m}\Trans{(\Mat{X}_{\Mat{S}_0m})})^{-1} \Mat{X}_{\Mat{S}_0m}] \nonumber \\
        &= \mathbb{E}_e[ (\Mat{X}_{\Mat{S}_0m}\Trans{(\Mat{X}_{\Mat{S}_0m})})^{-1} \Mat{X}_{\Mat{S}_0m}\Trans{\Mat{X}}_{\Mat{S}_0m}]\Trans{[\Mat{W}_0]}_m = \Trans{[\Mat{W}_0]}_m \label{regression_right}.
\end{align}
Note that $\Trans{[\Mat{W}_0]}_m \Mat{X}_{\Mat{S}_0m}$ is a scalar. 
We adopt the rule $(\Mat{a} \cdot \Mat{b}) \Mat{c} = (\Mat{c} \cdot \Mat{b}^T) \cdot \Mat{a} $, where $\Mat{a}, \Mat{b}, \Mat{c}$ are vectors of the same dimension in the second line. Moreover, since $\Trans{[\Mat{W}_0]}_m$ is independent of the environment $e$ and the variable vector $\Mat{X}$, it can be moved out of the expectation operation. 

Based on Equation (\ref{regression_right}), the obtained coefficients in environment $e$, $\Trans{[\hat{\Mat{W}}^e]}_m$, is exactly the true causal coefficients $\Trans{[\Mat{W}_0]}_m$. 
Since $m$ is any number in $\{1,\cdots, d\}$, we have $\hat{\Mat{W}}^e = \Mat{W}_0$. The lemma is proved. 
\end{proof}

% To prove the Theorem \ref{theorem:identifiability}, we further define two important terms - \emph{stable graph} and \emph{predecessor}.

% \definitionStableGraph

% \definitionPredecessor

% \begin{definition}\label{definition:stable_graph}
%     For a given graph $G$, if $\exists ~ \Mat{S}, \Mat{W}$ such that:
%     \begin{gather} 
%         \mathcal{G}(\Mat{S}) = \mathcal{G}(\Mat{W}) = G \nonumber \\
%         \Mat{W} = arg\min_{\overline{\Mat{W}}} \mathcal{L}^e(\Mat{S}\circ \overline{\Mat{W}}), \quad \forall e\in \mathcal{E} \label{regression}
%     \end{gather}
%     Then we say $G$ is a stable graph.
% \end{definition}

% \begin{definition}\label{definition:predecessor}
%     For two nodes $X_1$ and $X_2$ in a DAG, if $X_2$ is reachable from $X_1$, then $X_1$ is a predecessor of $X_2$. We denote all the predecessors of $X_2$ in graph $G_0$ as $Pre_0(X_2)$. 
% \end{definition}

With Definition \ref{definition:stable_graph}, we further denote the parents of the variable $X_i$ in $G_0$ as $Pa_0(X_i)$ and the corresponding indexes as $Pa_0(i)$. 
% we can give another lemma with respect to stable graph: 

Lemma \ref{lemma:stable_graph} demonstrates that the causal directions between any two variables cannot violate each other in any stable graph and the true causal graph. And we give the proof as following: 

\begin{proof}
    To prove the lemma, we only need to show that given any $X_i\in \{X_1,\cdots,X_d\}$ without being the source node, if $\exists~ X_j$ such that:
    \begin{equation}\label{contradiction_condition}
        X_j \in Pa_s(X_i), \quad X_i \in Pre_0(X_j).
    \end{equation}
    Then $G_s$ cannot be a stable graph. 
    
    To achieve this, we need to incorporate the information of the environment $e$. Given $j$ satisfying Equation \eqref{contradiction_condition}, we aim to find $j_0$ such that:
    \begin{equation*}
        X_{j_0}\in Pa_s(X_i), \quad z_{j_0} \perp Pa_s(X_i) \backslash X_{j_0}.
    \end{equation*}
    % We assert that $j_0$ exists as long as $j$ in Equation (\ref{contradiction_condition}) exists. 
    Suppose there exists $j$ satisfying Equation (\ref{contradiction_condition}), we will find $j_0$ with the following process: \\
    (1) Set $j_0$ to be $j$; \\
    (2) If $\exists ~ X_k \in Pa_s(X_i)$ such that there exists the path from $X_{j_0}$ to $X_k$, then we set $j_0$ to be $k$. In this step, the condition $X_{i} \in Pre_0(X_{j_0})$ still holds. If there is no such $X_k \in Pa_s(X_i)$ satisfying the above condition, we will exit the iteration. 
    
    In this iteration, we always have $X_i \in Pre_0(X_{j_0})$. Also, with the above process, we know that $\forall ~ X_k \in Pa_s(X_i)\backslash\{X_{j_0}\}, X_{j_0} \notin Pre_0(X_{k})$, which means:
    \begin{equation}\label{independence_j0}
        z_{j_0}^e \perp X_k| \forall k \in Pa_s(i)\backslash\{j_0\}.
    \end{equation}
    
    % Since $X_i \in Pre_0 (X_j)$ and $X_{j_0} \in Pre_0(X_j)$, we have $X_{j_0} \in Pre_0 (X_j)$, which means $X_i \perp z_{j_0}^e$, where $z_{j_0}^e$ is the additive noise from $Pa_0(X_{j_0})$ to $X_{j_0}$ in environment $e$. Thus we have $z_{j_0}^e \perp  \{X_i\}\cup Pa(X_i) \backslash \{X_{j_0}\} $.
    Now since the result of the regression from $Pa_s(X_i)$ to $X_i$ is exactly the $m$-th column of $\hat{\Mat{W}}^e$, we denote the result of the regression based on Equation \eqref{regression} as:
    \begin{align*}
        [\hat{\Mat{W}}^e]_m &= arg\min_{\Trans{[w_1,\cdots,w_d]}} \mathbb{E}_e \parallel w_{j_0}X_{j_0} \\
        &+ \sum_{k\in Pa_s(m)\backslash\{j\}} w_k X_k - X_i \parallel_2^2, \nonumber \\
        &\st w_j = 0, \quad j \in \{1,\cdots, d\}\backslash Pa_s(m).
    \end{align*}
    Then the optimal value for $w_{j_0}$ is:
    \begin{equation}\label{formulation_of_optimal_w}
        \hat{w}_{j_0}^e = \frac{\mathbb{E}_e[X_iX_{j_0}] - \sum_{k\in Pa_s(i)\backslash\{j_0\}} \hat{w}^e_k \mathbb{E}_e[X_k X_{j_0}] }{\mathbb{E}_e[X_{j_0}^2]}.
    \end{equation}
    
    According to the conditions of theorem \ref{theorem:identifiability}, there exist two distinct environments $e_1, e_2\in \mathcal{E}$ such that $ Var({z_{j_0}^{e_1}}) \neq Var({z_{j_0}^{e_2}})$ and for any other node $X_i$, \ie $i \neq j_0$, we have
    $Var(z_i^{e_1}) = Var(z_i^{e_2})$.
    % $P_{z_{j_0}^{e_1}} = P_{z_{j_0}^{e_2}}$ (Note that $z_{j_0}^{e_1}$ and $z_{j_0}^{e_1}$ all satisfy Gaussian distribution. We assume they are all zero-mean, and their variances are the same according to the conditions in Theorem \ref{theorem:identifiability}). 
    
    Hence in the right term in Equation (\ref{formulation_of_optimal_w}), with $X_i \perp z_{j_0}$ (according to the condition $X_i \in Pre_0(X_{j_0})$), and $\forall k \in Pa_s(i) \backslash \{j_0\}, z_{j_0}^e\perp X_k$ in Equation (\ref{independence_j0}), we could have:
    \begin{align*}
        \mathbb{E}_{e_1}[X_iX_{j_0}] &= \mathbb{E}_{e_2}[X_iX_{j_0}], \\ \quad \forall k \in Pa_s(X_i)\backslash\{j_0\}&, \mathbb{E}_{e_1}[X_k X_{j_0}] = \mathbb{E}_{e_1}[X_k X_{j_0}].
    \end{align*}
    
    If there exists $k\in Pa_s(i)\backslash\{j_0\}$ such that $\hat{w}_k^e$ is different across $e_1$ and $e_2$, then $G_s$ is already not stable, which raises the contradiction. Otherwise, if $\hat{w}_k^{e_1} = \hat{w}_k^{e_2}, \forall k\in Pa_s(i)\backslash\{j_0\}$, then since $E_{e_1}[X_{j_0}^2] \neq  E_{e_2}[X_{j_0}^2]$, we will have 
    % $\hat{w}_{j_0}^{e_1} \neq \hat{w}_{j_0}^{e_2}$, 
    \begin{equation}\label{eq:contradiction_from_wj0}
        \hat{w}_{j_0}^{e_1} \neq \hat{w}_{j_0}^{e_2},
    \end{equation}
    indicating $G_s$ is not a stable graph. The contradiction still exists. 
    
    Thus we can have the conclusion: If $G_s$ is a stable graph, then $\forall X_j \in Pa_s(X_i)$, we have $X_i\notin Pre_0(X_j)$. The lemma is proved. 
\end{proof}

% With Definition \ref{definition:stable_graph}, we can rewrite Theorem 4.1 as: 
% \theoremofjustification
% \begin{theorem}\label{theorem_of_justification}
%     For any stable graph $G_s$, we denote the corresponding structures for the true graph $G_0$ and $G_s$ as $\Mat{S}_0$ and $\Mat{S}_s$, respectively. Also we denote the corresponding consistent optimal parameters as $\Mat{W}_0$ and $\Mat{W}_s$. Then we have: 
%     \begin{equation}\label{loss_comparison}
%         \sum_{e\in\mathcal{E}}\mathcal{L}^e(\Mat{S}_0\circ \Mat{W}_0) \leq \sum_{e\in\mathcal{E}}\mathcal{L}^e(\Mat{S}_s\circ \Mat{W}_s)
%     \end{equation}
%     The equation holds only for $\Mat{W}_0 = \Mat{W}_s$. 
% \end{theorem}
Now we give the proof of Theorem \ref{theorem_of_justification}. Since it's written from Theorem \ref{theorem:identifiability} with Definition \ref{definition:stable_graph}, then Theorem \ref{theorem:identifiability} will also be proved. 
\begin{proof}

Note that in this case $\Mat{W}_0$ is the optimal parameters for $G_0$, and thus it is also the true coefficients for data generation. (as shown in Lemma \ref{lemma: optimal_parameter_is_true_parameter}.) 

To prove Theorem 4.1, we will split the regression loss in Equation \eqref{regression_result_for_e} to the specific loss value in every environment and every node except source nodes. 
In other words, we illustrate that for any environment $e \in \mathcal{E}$, $\mathcal{L}^e(\Mat{S}_0\circ \Mat{W}_0) \leq \mathcal{L}^e(\Mat{S}_s\circ \Mat{W}_s)$, which could directly lead to Equation (\ref{loss_comparison}). 
Hence, in the following part, with a slight abuse of notation, we will omit the subscript $e$. 
We will show that $\forall ~ X_i \in \Set{V}$, where $\Set{V}$ is the node set, the regression loss from $Pa(X_i)$ to $X_i$ in $G_s$ is larger or equal than in $G_0$, with the equaling condition holds when:
\begin{equation}\label{equal_condition}
    \textrm{i-th column($\Mat{W}_0$) = i-th column($\Mat{W}_s$)}. 
\end{equation}

Given any $X_i \in \Set{V}, i \in \{1,\cdots, d\}$, we denote all the parents of $X_i$ in $G_s, G_0$ as $Pa_s(X_i), Pa_0(X_i)$ and the corresponding indexes of its parents as $Pa_s(i), Pa_0(i)$, respectively. Then the true generation process for the variable $X_i$ could be denoted as: 
\begin{equation}\label{generation_process_for_X_i}
    X_i = \sum_{k\in Pa_0(i)}[\Mat{W}_0]_{ki} X_k + z_i; \quad z_i \sim \mathcal{N}(0, \sigma_i^2).
\end{equation}

% Now we give the following lemma: 
% \begin{lemma}\label{perp_condition}
%     $\forall X_j \in Pa_s(X_i)$, $X_j \perp z_i$. 
% \end{lemma}
% \begin{proof}
%    xxxxx
% \end{proof}
When performing regression from $Pa_s(X_i)$ to $X_i$, we need to constrain $\mathcal{G}(\Mat{W}_s) = G_s$, which means
\begin{equation*}
[\Mat{W}_s]_{ki} = 0 | \forall k \in \{1,\cdots, d\}\backslash Pa_s(i).
\end{equation*}
We denote the minimal loss and the optimal elements as
\begin{align}\label{minimal_loss_definition}
    L_i^s = \min_{\{\overline{w}_k| k\in Pa_s(i)\}} & \mathbb{E} \parallel X_i - \sum_{k\in Pa_s(i)} \overline{w}_k X_k \parallel_2^2, \\
    \{w_k^*| k\in Pa_s(i)\} &= arg\min_{\{\overline{w}_k| k\in Pa_s(i)\}} \mathbb{E} \parallel X_i \nonumber \\
    &- \sum_{k\in Pa_s(i)} \overline{w}_k X_k \parallel_2^2, \nonumber
\end{align}
where $\{w_k^*| k\in Pa_s(i)\}$ is a set and its elements are exactly the corresponding elements of $[\Mat{W}_s]_i$, \ie the non-zero elements of $i$-th column of $\Mat{W}_s$. 

Then we define three sets of indexes as follows:
\begin{align*}
    \mathcal{I}_u = Pa_0(i)\cap Pa_s(i), \\
    \quad \mathcal{I}_v  = Pa_0(i)\backslash Pa_s(i), \\
    \quad \mathcal{I}_k  = Pa_s(i)\backslash Pa_0(i).
\end{align*}
Then after substituting Equation (\ref{generation_process_for_X_i}) into Equation (\ref{minimal_loss_definition}), we have: 
\begin{align}
    L_i^s&= \min_{\{\overline{w}_k| k\in Pa_s(i)\}} \mathbb{E} \parallel \sum_{u\in \mathcal{I}_u} ([\Mat{W}_0]_{ui} - \overline{w}_u )X_u \nonumber \\
    &+ \sum_{v\in \Set{I}_v} [\Mat{W}_0]_{vi} X_v + \sum_{k\in \Set{I}_k} \overline{w}_k X_k + z_i \parallel_2^2. \nonumber
\end{align}
According to Lemma \ref{lemma: optimal_parameter_is_true_parameter}, $\forall~ X_j \in Pa_s(X_i)$, we have $ X_i \notin Pre_0(X_j)$, which means $X_j \perp z_i$. Thus we have $z_i\perp Pa_s(X_i)$. By definition, we have $z_i \perp Pa_0(X_i)$, which yield $z_i \perp \{Pa_s(X_i) \cup Pa_0(X_i)\} $. Then we can rewrite the above equation as:
\begin{align}
    L_i^s&= \min_{\{\overline{w}_k| k\in Pa_s(i)\}} \mathbb{E} \parallel \sum_{u\in \Set{I}_u} ([\Mat{W}_0]_{ui} - \overline{w}_u )X_u \nonumber \\
    &+ \sum_{v\in \Set{I}_v} [\Mat{W}_0]_{vi} X_v + \sum_{k\in \Set{I}_k} \overline{w}_k X_k \parallel_2^2 + \sigma_i^2. \nonumber
\end{align}
Since the reconstruction loss for $X_i$ in the true graph $G_0$, denoted as $L_i^0$ is exactly the variance of the additive noise added on $X_i$, i.e., $\sigma_i^2$. Thus here we already have $L_i^s\geq L_i^0 = \sigma_i^2$. In the following part, We will explore the conditions that make the equation hold exactly.

Note that $\forall j\in Pa_0(i)$, we have $[\Mat{W}_0]_{ji} \neq 0$. Then for all the nodes in $\Set{I}_u \cup \Set{I}_v \cup \Set{I}_k$, there must exists a node denoted as $X_{end} \in  \Set{I}_u \cup \Set{I}_v \cup \Set{I}_k$ that has no successors in $\Set{I}_u \cup \Set{I}_v \cup \Set{I}_k$ according to $G_0$. This means the noises corresponding to $X_{end}$ will be independent from all the other nodes in $\Set{I}_u \cup \Set{I}_v \cup \Set{I}_k$. 
Then we discuss three situations for $X_{end}$ to be in $\Set{I}_u$, $\Set{I}_v$, and $\Set{I}_k$, respectively:
\begin{itemize}[leftmargin=*]
    \item If $X_{end} \in \Set{I}_u$, we denote the index for $X_{end}$ as $u_1$. Then we have: 
    \begin{align*}
        L_i^s&- L_i^0 \geq \min_{\{\overline{w}_k| k\in Pa_s(i)\}} \mathbb{E} \parallel \sum_{u\in \Set{I}_u\backslash\{u_1\}} ([\Mat{W}_0]_{ui} - \overline{w}_u )X_u \\
        &+ \sum_{v\in \Set{I}_v} [\Mat{W}_0]_{vi} X_v + \sum_{k\in \Set{I}_k} \overline{w}_k X_k \parallel_2^2 +  ([\Mat{W}_0]_{u_1i} - \overline{w}_{u_1})^2 \sigma_{u_1}^2.
    \end{align*}
    Then the minimum operation $\displaystyle \min_{\{\overline{w}_k| k\in Pa_s(i)\}}$ in Equation \eqref{minimal_loss_definition} will directly set the optimal value of $\overline{w}_{u_1}$, \ie ${w}^*_{u_1}$ to be $[\Mat{W}_0]_{u_1i}$. 
\item If $X_{end} \in \Set{I}_v$, we denote the index for $X_{end}$ as $v_1$. Then we have: \begin{align*}
        L_i^s& - L_i^0 \geq \min_{\{\overline{w}_k| k\in Pa_s(i)\}} \mathbb{E} \parallel \sum_{u\in \Set{I}_u} ([\Mat{W}_0]_{ui} - \overline{w}_u )X_u \\
        &+ \sum_{v\in \Set{I}_v\backslash\{v_1\}} [\Mat{W}_0]_{vi} X_v +  \sum_{k\in \Set{I}_k} \overline{w}_k X_k \parallel_2^2  +  [\Mat{W}_0]_{v_1i}^2 \sigma_{v_1}^2.
    \end{align*}
    In this case, we will have the strict inequality $L_i^s> L_i^0$, which means the reconstruction loss for $G_s$ will be strictly larger than $G_0$. 
    \item If $X_{end} \in \Set{I}_k$, we denote the index for $X_{end}$ as $k_1$. Then we have: \begin{align*}
        L_i^s&- L_i^0 \geq \min_{\{\overline{w}_k| k\in Pa_s(i)\}} \mathbb{E} \parallel \sum_{u\in \Set{I}_u} ([\Mat{W}_0]_{ui} - \overline{w}_u )X_u \\
        &+ \sum_{v\in \Set{I}_v} [\Mat{W}_0]_{vi} X_v + \sum_{k\in \Set{I}_k\backslash\{k_1\}} \overline{w}_k X_k \parallel_2^2 +  \overline{w}_{k_1}^2 \sigma_{k_1}^2.
    \end{align*}
    Then the minimum operation $\displaystyle \min_{\{\overline{w}_k| k\in Pa_s(i)\}}$ in Equation \eqref{minimal_loss_definition} will directly set the optimal value of $\overline{w}_{k_1}$, \ie ${w}^*_{k_1}$ to be 0. 
    % Then the minimum operation $\min_{\overline{\Mat{w}}}$ will set $\overline{w}_{u_1}$ to be 0.  
\end{itemize}

If $X_{end} \in \Set{I}_v$, then we can already have $L_i^s> L_i^0$. For the other two situations, we will remove $X_{end} \in \Set{I}_u $ or $\Set{I}_k$ to form the new set $\Set{I}_u'$ and $\Set{I}_k'$, then find the new $X_{end}$ in $\Set{I}_u'\cup \Set{I}_v \cup \Set{I}_k'$ to perform the above process again. According to this discussion, we could easily give the necessary and sufficient conditions for $L_i^s$ to be equal to $L_i^0$: 
\begin{itemize}
    \item $\Set{I}_v = \emptyset$. If not, then there must be one step of removing the current $X_{end}$ such that $X_{end} \in \Set{I}_v$, which will directly lead to $L_i^s> L_i^0$. 
    \item $\forall u\in \Set{I}_u, {w}^*_u = [\Mat{W}_0]_{ui}$.
    \item $\forall k\in \Set{I}_k, {w}^*_k = 0$. 
\end{itemize}
Now we have the optimal values: $\{w_u| u\in \Set{I}_u\}$ and $\{w_k| k \in \Set{I}_k\}$.
Besides, according to the definition of $\Set{I}_u$ and $\Set{I}_k$, and $\Set{I}_v=\emptyset$, we have:  $\Set{I}_u \cup \Set{I}_k = Pa_s(j)$ and $\Set{I}_u = Pa_0(i)$. Since $\Set{I}_k \cap Pa_0(i) = \emptyset$, we have: $\forall k\in \Set{I}_k, [\Mat{W}_0]_{ki} = 0$. 

Since the structure of $\Mat{W}_s$ and $\Mat{W}_0$ need to be constrained to be the same as $G_s$ and $G_0$, respectively, we must have:
\begin{equation*}
    [\Mat{W}_0]_{ji} = [\Mat{W}_s]_{ji} = 0, \quad \forall j\in \{1,\cdots,d\} \backslash (\Set{I}_u \cup \Set{I}_k).
\end{equation*}
Then according to the above conditions, we have:
% Denote the obtained $\Mat{W}_s$ after performing the regression according to $G_s$. Then since $H\cup U = Pa_s(j)$, the $i$-th column of $\Mat{W}_s$ will be totally defined by $\overline{\Mat{w}}$ in Equation (\ref{minimal_loss_definition}). Then because that $U \cap Pa_0(i) = \emptyset$, we have: $\forall u\in U, [\Mat{W}_t]_{ui} = 0$. Thus we could finally rewrite the conditions listed above as: 
\begin{equation*}
    \textrm{i-th column($\Mat{W}_0$) = i-th column($\Mat{W}_s$)}. 
\end{equation*}
which is Equation (\ref{equal_condition}).
Since $X_i$ is any variable except source node in $\Set{V}$, we could have $\Mat{W}_0 = \Mat{W}_s$. (Note that for the columns corresponding to the source node, all the elements in these columns should be zero such that $\mathcal{G}(\Mat{W}_s) = G_s$ and $\mathcal{G}(\Mat{W}_0) = G_0$). 

We conclude the proof of Theorem \ref{theorem_of_justification}.

\end{proof}

%% file: 0_main.bbl
\begin{thebibliography}{10}
\providecommand{\url}[1]{#1}
\csname url@samestyle\endcsname
\providecommand{\newblock}{\relax}
\providecommand{\bibinfo}[2]{#2}
\providecommand{\BIBentrySTDinterwordspacing}{\spaceskip=0pt\relax}
\providecommand{\BIBentryALTinterwordstretchfactor}{4}
\providecommand{\BIBentryALTinterwordspacing}{\spaceskip=\fontdimen2\font plus
\BIBentryALTinterwordstretchfactor\fontdimen3\font minus
  \fontdimen4\font\relax}
\providecommand{\BIBforeignlanguage}[2]{{%
\expandafter\ifx\csname l@#1\endcsname\relax
\typeout{** WARNING: IEEEtran.bst: No hyphenation pattern has been}%
\typeout{** loaded for the language `#1'. Using the pattern for}%
\typeout{** the default language instead.}%
\else
\language=\csname l@#1\endcsname
\fi
#2}}
\providecommand{\BIBdecl}{\relax}
\BIBdecl

\bibitem{zheng2018dags}
X.~Zheng, B.~Aragam, P.~Ravikumar, and E.~P. Xing, ``Dags with {NO} {TEARS:}
  continuous optimization for structure learning,'' in \emph{NeurIPS}, 2018,
  pp. 9492--9503.

\bibitem{CXPlain}
P.~Schwab and W.~Karlen, ``Cxplain: Causal explanations for model
  interpretation under uncertainty,'' in \emph{NeurIPS}, 2019, pp.
  10\,220--10\,230.

\bibitem{PGM-Explainer}
M.~N. Vu and M.~T. Thai, ``Pgm-explainer: Probabilistic graphical model
  explanations for graph neural networks,'' in \emph{NeurIPS}, 2020.

\bibitem{sachs2005causal}
K.~Sachs, O.~Perez, D.~Pe'er, D.~A. Lauffenburger, and G.~P. Nolan, ``Causal
  protein-signaling networks derived from multiparameter single-cell data,''
  \emph{Science}, vol. 308, no. 5721, pp. 523--529, 2005.

\bibitem{opgen2007correlation}
R.~Opgen-Rhein and K.~Strimmer, ``From correlation to causation networks: a
  simple approximate learning algorithm and its application to high-dimensional
  plant gene expression data,'' \emph{BMC systems biology}, vol.~1, no.~1, pp.
  1--10, 2007.

\bibitem{sanford2012bayesian}
A.~D. Sanford and I.~A. Moosa, ``A bayesian network structure for operational
  risk modelling in structured finance operations,'' \emph{JORS}, vol.~63,
  no.~4, pp. 431--444, 2012.

\bibitem{yu2019dag}
Y.~Yu, J.~Chen, T.~Gao, and M.~Yu, ``{DAG-GNN:} {DAG} structure learning with
  graph neural networks,'' in \emph{{ICML}}, vol.~97, 2019, pp. 7154--7163.

\bibitem{zhu2019causal}
S.~Zhu, I.~Ng, and Z.~Chen, ``Causal discovery with reinforcement learning,''
  in \emph{{ICLR}}, 2020.

\bibitem{LachapelleBDL20}
S.~Lachapelle, P.~Brouillard, T.~Deleu, and S.~Lacoste{-}Julien,
  ``Gradient-based neural {DAG} learning,'' in \emph{{ICLR}}.\hskip 1em plus
  0.5em minus 0.4em\relax OpenReview.net, 2020.

\bibitem{zheng2020learning}
X.~Zheng, C.~Dan, B.~Aragam, P.~Ravikumar, and E.~P. Xing, ``Learning sparse
  nonparametric dags,'' in \emph{{AISTATS}}, vol. 108, 2020, pp. 3414--3425.

\bibitem{2021differentiable}
R.~Bhattacharya, T.~Nagarajan, D.~Malinsky, and I.~Shpitser, ``Differentiable
  causal discovery under unmeasured confounding,'' in \emph{{AISTATS}}.\hskip
  1em plus 0.5em minus 0.4em\relax PMLR, 2021, pp. 2314--2322.

\bibitem{yu2020dags}
Y.~Yu and T.~Gao, ``Dags with no curl: Efficient dag structure learning,'' in
  \emph{NeurIPS 2020}, 2020.

\bibitem{ERM}
V.~Vapnik, ``Principles of risk minimization for learning theory,'' in
  \emph{{NeurIPS}}, 1991, pp. 831--838.

\bibitem{IRM}
M.~Arjovsky, L.~Bottou, I.~Gulrajani, and D.~Lopez{-}Paz, ``Invariant risk
  minimization,'' \emph{CoRR}, 2019.

\bibitem{GroupDRO}
S.~Sagawa, P.~W. Koh, T.~B. Hashimoto, and P.~Liang, ``Distributionally robust
  neural networks for group shifts: On the importance of regularization for
  worst-case generalization,'' 2019.

\bibitem{Rubi}
R.~Cad{\`{e}}ne, C.~Dancette, H.~Ben{-}younes, M.~Cord, and D.~Parikh, ``Rubi:
  Reducing unimodal biases in visual question answering,'' \emph{CoRR}, vol.
  abs/1906.10169, 2019.

\bibitem{REx}
D.~Krueger, E.~Caballero, J.~Jacobsen, A.~Zhang, J.~Binas, D.~Zhang, R.~L.
  Priol, and A.~C. Courville, ``Out-of-distribution generalization via risk
  extrapolation (rex),'' in \emph{{ICML}}, 2021, pp. 5815--5826.

\bibitem{He0SXLJ21}
Y.~He, P.~Cui, Z.~Shen, R.~Xu, F.~Liu, and Y.~Jiang, ``{DARING:} differentiable
  causal discovery with residual independence,'' in \emph{{KDD}}, 2021, pp.
  596--605.

\bibitem{CausalInferenceIP}
J.~Peters, P.~B{\"u}hlmann, and N.~Meinshausen, ``Causal inference by using
  invariant prediction: identification and confidence intervals,''
  \emph{Journal of the Royal Statistical Society: Series B (Statistical
  Methodology)}, vol.~78, no.~5, pp. 947--1012, 2016.

\bibitem{MultiDomain}
A.~Ghassami, N.~Kiyavash, B.~Huang, and K.~Zhang, ``Multi-domain causal
  structure learning in linear systems,'' in \emph{NeurIPS}, 2018, pp.
  6269--6279.

\bibitem{CD-NOD}
B.~Huang, K.~Zhang, J.~Zhang, J.~D. Ramsey, R.~Sanchez{-}Romero, C.~Glymour,
  and B.~Sch{\"{o}}lkopf, ``Causal discovery from heterogeneous/nonstationary
  data,'' \emph{JMLR}, vol.~21, pp. 89:1--89:53, 2020.

\bibitem{RegressionInvariance}
A.~Ghassami, S.~Salehkaleybar, N.~Kiyavash, and K.~Zhang, ``Learning causal
  structures using regression invariance,'' in \emph{{NIPS}}, 2017, pp.
  3011--3021.

\bibitem{SAEM}
B.~Huang, K.~Zhang, M.~Gong, and C.~Glymour, ``Causal discovery and forecasting
  in nonstationary environments with state-space models,'' in \emph{{ICML}},
  2019, pp. 2901--2910.

\bibitem{FOM}
R.~Cai, J.~Ye, J.~Qiao, H.~Fu, and Z.~Hao, ``{FOM:} fourth-order moment based
  causal direction identification on the heteroscedastic data,'' \emph{Neural
  Networks}, vol. 124, pp. 193--201, 2020.

\bibitem{Dep}
J.~Qiao, Y.~Bai, R.~Cai, and Z.~Hao, ``Causal discovery from multi-domain data
  using the independence of modularities,'' \emph{Neural Computing and
  Applicationis}, vol.~34, no.~3, pp. 1939--1949, 2022.

\bibitem{pearl2016causal}
J.~Pearl, M.~Glymour, and N.~P. Jewell, \emph{Causal inference in statistics: A
  primer}.\hskip 1em plus 0.5em minus 0.4em\relax John Wiley \& Sons, 2016.

\bibitem{huang2008labeled}
G.~B. Huang, M.~Mattar, T.~Berg, and E.~Learned-Miller, ``Labeled faces in the
  wild: A database forstudying face recognition in unconstrained
  environments,'' in \emph{Workshop on faces in'Real-Life'Images: detection,
  alignment, and recognition}, 2008.

\bibitem{DengDSLL009}
J.~Deng, W.~Dong, R.~Socher, L.~Li, K.~Li, and L.~Fei{-}Fei, ``Imagenet: {A}
  large-scale hierarchical image database,'' in \emph{{CVPR}}, 2009, pp.
  248--255.

\bibitem{koh2021wilds}
P.~W. Koh, S.~Sagawa, S.~M. Xie, M.~Zhang, A.~Balsubramani, W.~Hu, M.~Yasunaga,
  R.~L. Phillips, I.~Gao, T.~Lee \emph{et~al.}, ``Wilds: A benchmark of
  in-the-wild distribution shifts,'' in \emph{ICML}, 2021, pp. 5637--5664.

\bibitem{bandi2018detection}
P.~Bandi, O.~Geessink, Q.~Manson, M.~Van~Dijk, M.~Balkenhol, M.~Hermsen, B.~E.
  Bejnordi, B.~Lee, K.~Paeng, A.~Zhong \emph{et~al.}, ``From detection of
  individual metastases to classification of lymph node status at the patient
  level: the camelyon17 challenge,'' \emph{TMI}, 2018.

\bibitem{ni2019justifying}
J.~Ni, J.~Li, and J.~McAuley, ``Justifying recommendations using
  distantly-labeled reviews and fine-grained aspects,'' in \emph{EMNLP-IJCNLP},
  2019.

\bibitem{mnist}
L.~Deng, ``The mnist database of handwritten digit images for machine learning
  research,'' \emph{SPM}, vol.~29, no.~6, pp. 141--142, 2012.

\bibitem{DBLP:journals/ml/HeckermanGC95}
D.~Heckerman, D.~Geiger, and D.~M. Chickering, ``Learning bayesian networks:
  The combination of knowledge and statistical data,'' \emph{Maching Learning},
  vol.~20, no.~3, pp. 197--243, 1995.

\bibitem{tsamardinos2006max}
I.~Tsamardinos, L.~E. Brown, and C.~F. Aliferis, ``The max-min hill-climbing
  bayesian network structure learning algorithm,'' \emph{Machine learning},
  vol.~65, no.~1, pp. 31--78, 2006.

\bibitem{DBLP:journals/datamine/GamezMP11}
J.~A. G{\'{a}}mez, J.~L. Mateo, and J.~M. Puerta, ``Learning bayesian networks
  by hill climbing: efficient methods based on progressive restriction of the
  neighborhood,'' \emph{Data Mining and Knowledge Discovery}, vol.~22, no. 1-2,
  pp. 106--148, 2011.

\bibitem{DBLP:conf/nips/NieMCJ14}
S.~Nie, D.~D. Mau{\'{a}}, C.~P. de~Campos, and Q.~Ji, ``Advances in learning
  bayesian networks of bounded treewidth,'' in \emph{{NeurIPS}}, 2014, pp.
  2285--2293.

\bibitem{DBLP:conf/nips/ScanagattaCCZ15}
M.~Scanagatta, C.~P. de~Campos, G.~Corani, and M.~Zaffalon, ``Learning bayesian
  networks with thousands of variables,'' in \emph{{NeurIPS}}, 2015, pp.
  1864--1872.

\bibitem{DBLP:journals/jair/YuanM13}
C.~Yuan and B.~M. Malone, ``Learning optimal bayesian networks: {A} shortest
  path perspective,'' \emph{JAIR}, vol.~48, pp. 23--65, 2013.

\bibitem{DBLP:journals/corr/abs-1904-10574}
H.~Manzour, S.~K{\"{u}}{\c{c}}{\"{u}}kyavuz, and A.~Shojaie, ``Integer
  programming for learning directed acyclic graphs from continuous data,''
  \emph{CoRR}, vol. abs/1904.10574, 2019.

\bibitem{gao2021dag}
Y.~Gao, L.~Shen, and S.-T. Xia, ``Dag-gan: Causal structure learning with
  generative adversarial nets,'' in \emph{ICASSP}.\hskip 1em plus 0.5em minus
  0.4em\relax IEEE, 2021, pp. 3320--3324.

\bibitem{GOLEM}
I.~Ng, A.~Ghassami, and K.~Zhang, ``On the role of sparsity and {DAG}
  constraints for learning linear dags,'' in \emph{NeurIPS}, 2020.

\bibitem{NoFears}
D.~Wei, T.~Gao, and Y.~Yu, ``Dags with no fears: {A} closer look at continuous
  optimization for learning bayesian networks,'' in \emph{NeurIPS}, 2020.

\bibitem{varando2020learning}
G.~Varando, ``Learning dags without imposing acyclicity,'' \emph{CoRR}, vol.
  abs/2006.03005, 2020.

\bibitem{lippe2021efficient}
P.~Lippe, T.~Cohen, and E.~Gavves, ``Efficient neural causal discovery without
  acyclicity constraints,'' \emph{CoRR}, vol. abs/2107.10483, 2021.

\bibitem{spirtes2000causation}
P.~Spirtes, C.~N. Glymour, R.~Scheines, and D.~Heckerman, \emph{Causation,
  prediction, and search}.\hskip 1em plus 0.5em minus 0.4em\relax MIT press,
  2000.

\bibitem{spirtes2013causal}
P.~Spirtes, C.~Meek, and T.~S. Richardson, ``Causal inference in the presence
  of latent variables and selection bias,'' in \emph{{UAI}}, 1995, pp.
  499--506.

\bibitem{zhang2021testing}
H.~Zhang, K.~Zhang, S.~Zhou, J.~Guan, and J.~Zhang, ``Testing independence
  between linear combinations for causal discovery,'' in \emph{AAAI}, vol.~35,
  no.~7, 2021, pp. 6538--6546.

\bibitem{DBLP:journals/ijon/Shanmugam01}
R.~Shanmugam, ``Causality: Models, reasoning, and inference : Judea pearl;
  cambridge university press, cambridge, uk, 2000, pp 384, {ISBN}
  0-521-77362-8,'' \emph{Neurocomputing}, vol.~41, no. 1-4, pp. 189--190, 2001.

\bibitem{FCD}
E.~Gao, J.~Chen, L.~Shen, T.~Liu, M.~Gong, and H.~Bondell, ``Federated causal
  discovery,'' \emph{CoRR}, vol. abs/2112.03555, 2021.

\bibitem{FCDBN}
I.~Ng and K.~Zhang, ``Towards federated bayesian network structure learning
  with continuous optimization,'' \emph{CoRR}, vol. abs/2110.09356, 2021.

\bibitem{LocallyOverlappingVariables}
R.~E. Tillman, D.~Danks, and C.~Glymour, ``Integrating locally learned causal
  structures with overlapping variables,'' in \emph{{NIPS}}.\hskip 1em plus
  0.5em minus 0.4em\relax Curran Associates, Inc., 2008, pp. 1665--1672.

\bibitem{CSOVS}
S.~Triantafilou, I.~Tsamardinos, and I.~G. Tollis, ``Learning causal structure
  from overlapping variable sets,'' in \emph{{AISTATS}}, ser. {JMLR}
  Proceedings, vol.~9.\hskip 1em plus 0.5em minus 0.4em\relax JMLR.org, 2010,
  pp. 860--867.

\bibitem{CBCDMultiInterventions}
S.~Triantafillou and I.~Tsamardinos, ``Constraint-based causal discovery from
  multiple interventions over overlapping variable sets,'' \emph{J. Mach.
  Learn. Res.}, vol.~16, pp. 2147--2205, 2015.

\end{thebibliography}
